\documentclass[11pt]{article}
\newif\ifarxiv
\arxivfalse
\usepackage[preprint]{acl}
\usepackage{times}
\usepackage{latexsym}

\usepackage[T1]{fontenc}

\usepackage[utf8]{inputenc}

\usepackage{microtype}
\usepackage{inconsolata}
\usepackage{amsthm}

\usepackage{xcolor}
\usepackage{graphicx}
\usepackage{booktabs}
\usepackage{multirow}
\usepackage{enumitem}
\usepackage{amsmath}
\usepackage{amsfonts}
\usepackage{nicefrac}
\usepackage{thm-restate}
\makeatletter
\@ifundefined{newcounteralias}{}{%
  \renewcommand\thmt@autorefsetup{%
    \@xa\def\csname\thmt@envname autorefname\@xa\endcsname\@xa{\thmt@thmname}%
  }%
}
\makeatother
\usepackage{pifont}
\usepackage[ruled,vlined,linesnumbered]{algorithm2e}
\usepackage{placeins}
\usepackage[skins,breakable]{tcolorbox}
\usepackage{listings}
\usepackage{upquote}
\SetKwIF{If}{ElseIf}{Else}{if}{:}{else if}{else:}{end if}
\SetKwFor{For}{for}{:}{end for}
\SetArgSty{textnormal}
\usepackage{tikz}
\usetikzlibrary{calc,tikzmark,decorations.pathreplacing,calligraphy}

\lstdefinestyle{prompt}{
  basicstyle=\ttfamily\normalsize,
  columns=fullflexible,
  keepspaces=true,
  breaklines=true,
  breakatwhitespace=true,
  breakindent=0pt,
  breakautoindent=false,
  showstringspaces=false,
  aboveskip=0pt,
  belowskip=0pt
}
\newcommand{\AlgoHighlightStart}[1]{%
  \tikzmark{#1-start}%
  \begin{tikzpicture}[remember picture,overlay]
    \coordinate (highlight-east) at ([xshift=\hsize]pic cs:#1-start);
    \coordinate (highlight-south) at ([yshift=-0.55ex]pic cs:#1-end);
    \fill[yellow!50]
      ($(pic cs:#1-start)+(-0.15em,1.55ex)$)
      rectangle (highlight-east |- highlight-south);
  \end{tikzpicture}%
}
\newcommand{\AlgoHighlightEnd}[1]{\tikzmark{#1-end}}

\newcommand{\framework}{\texttt{NFA-LM}}

\newcommand{\median}{\operatorname{median}}

\newtheorem{lemma}{Lemma}
\newtheorem{proposition}{Proposition}
\newtheorem{example}{Example}

\theoremstyle{definition}
\newtheorem{definition}{Definition}
\theoremstyle{plain}
\newtheorem{fact}{Fact}

\newenvironment{paperalgorithm}[1][t]{%
  \ifarxiv\begin{algorithm*}[t]\else\begin{algorithm}[#1]\fi
}{%
  \ifarxiv\end{algorithm*}\else\end{algorithm}\fi
}

\title{Provably Tractable NFA-Constrained Language Generation via HMMs}

\author{Jialiang Sun \and Kuldeep Meel \\
  University of Toronto \\
  \texttt{\{sjl,meel\}@cs.toronto.edu}}

\begin{document}
\maketitle

\begin{abstract}
Constrained generation aims to sample from language models (LMs) conditioned on hard constraints. Existing constrained-generation techniques for nondeterministic finite automaton (NFA) constraints either distort the distribution or sacrifice efficiency. Theoretically, this task reduces to counting the length-$n$ sequences accepted by an NFA (\#NFA), and the exact \#NFA problem is \#P-complete. Recent work has shown that
\#NFA admits a fully polynomial randomized approximation scheme (FPRAS). Inspired by this result, we propose \framework{}, a polynomial-time engine for NFA-constrained generation with theoretical guarantees under mild assumptions. Experiments show that \framework{} efficiently generates high-quality outputs with theoretically bounded approximation error.

\end{abstract}

\section{Introduction}
As language models (LMs) become ubiquitous across a wide range of tasks, the demand for \emph{constrained generation}, which requires outputs to satisfy specific constraints, continues to grow. Constrained-generation techniques benefit downstream applications such as text detoxification~\citep{realtoxicityprompts,trace}, agentic function calling~\citep{gorilla,bfcl,openai_function_calling_2024}, and structured SQL generation~\citep{spider}. Regular expressions (Regex)~\citep{kleene1956representation} are a powerful tool for modeling such constraints and can be represented efficiently by NFAs with a linear number of states~\citep{thompson_nfa,glushkov_nfa}.

Given an LM $P_{\text{lm}}$ and a constraint $\alpha$, constrained generation seeks to sample outputs from $P_{\text{lm}}(x_{1:n}\mid \alpha)$. At step $\ell$, this is equivalent to sampling $x_\ell$ from a distribution proportional to $P_{\text{lm}}(x_{\ell}\mid x_{1:\ell-1})\cdot P_{\text{lm}}(\alpha\mid x_{1:\ell-1}\cdot x_{\ell})$. The main difficulty lies in computing the constraint-conditioned probability $P_{\text{lm}}(\alpha\mid x_{1:\ell})$, which requires marginalizing over future continuations: $\underset{x_{\ell+1:n}: x_{1:n}\text{ satisfies }\alpha}{\sum}P_{\text{lm}}(x_{\ell+1:n}\mid x_{1:\ell})$. The unweighted special case of computing
the number of suffixes $x_{\ell+1:n}$ such that $x_{1:n}$ satisfies $\alpha$, given a prefix $x_{1:\ell}$, is a \#NFA problem, which is known to be \#P-complete~\citep{alvarez1993very}. Existing constrained-generation techniques fall into two categories based on the tradeoff between efficiency and distribution preservation. (i) \textit{Distribution-agnostic} techniques, such as PICARD~\citep{picard}, Synchromesh~\citep{synchromesh}, and XGrammar~\citep{xgrammar,xgrammar2}, mask out tokens that violate the constraint using explicit decoding-time checkers. These approaches efficiently determine whether $P_{\text{lm}}(\alpha\mid x_{1:\ell})$ is zero, but they distort the true conditional distribution~\citep{gad,approxaligned}. (ii) \textit{Distribution-aware} techniques seek to preserve $P_{\text{lm}}(\alpha\mid x_{1:\ell})$. Existing approaches, such as Langevin-based methods~\citep{mucola,qin2022cold}, MCMC samplers~\citep{amini2023structured,mcmc_cfg}, and SMC-based methods~\citep{lew2023sequential,zhao2024probabilistic,loula2025syntactic,lipkin2025fast}, generally lack polynomial-time guarantees for approximating the constrained LM distribution to a prescribed accuracy. Recently, a line of research distills Hidden Markov Model (HMM) $P_{\text{hmm}}(\alpha\mid x_{1:\ell})\approx P_{\text{lm}}(\alpha\mid x_{1:\ell})$ as a tractable proposal: GeLaTo~\citep{tractablecontrol} and Ctrl-G~\citep{ctrlg} develop efficient algorithms that compute $P_{\text{hmm}}(\alpha\mid x_{1:\ell})$ exactly for DFAs and handle basic constraints such as keyword inclusion and length requirements, but general Regex still requires up to exponentially many DFA states~\citep{moore1971bounds}.

Consequently, LMs still struggle with constrained generation~\citep{jsonschemabench,sun2023evaluatinglargelanguagemodels,promptconstraints}. This raises a research question: \textit{Can we achieve efficient distribution-aware constrained generation for NFAs with theoretically bounded error?}
We answer this question affirmatively with \framework{}. Our framework is theoretically inspired by efficient FPRAS for \#NFA~\citep{arenas2021nfa,meel2025towards}, and generalizes Ctrl-G~\citep{ctrlg} to NFA case. 
\framework{} first distills an HMM from the base LM and uses it to assign tractable probability weights to future continuations. It then generalizes the \#NFA FPRAS to estimate the HMM completion weight $P_{\text{hmm}}(\alpha\mid x_{1:\ell})$ to reweight the LM's next-token distribution.

We organize the rest of the paper as follows. We introduce the preliminaries in Section~\ref{sec:preliminaries}, followed by the key concepts of canonical runs in Section~\ref{sec:canonical-runs}. We present \framework{} in Section~\ref{sec:methods} and its main theoretical guarantees in Section~\ref{sec:theoretical-analysis}. We empirically evaluate \framework{} in Section~\ref{sec:evaluation} and conclude in Section~\ref{sec:conclusion}. Complete engineering details and proofs appear in Appendices~\ref{sec:cache},~\ref{sec:whole_proof}, and~\ref{sec:time_complexity}.

\section{Preliminaries and Background}
\label{sec:preliminaries}
We introduce the preliminaries as below.
\paragraph{LM and HMM}
A token in the vocabulary $\Sigma$ is a basic element of a language model. We denote by $\Sigma^n$ and $\Sigma^*$ the sets of length-$n$ and arbitrary-length sequences, respectively. For each sequence $x=x_{1:n}\in\Sigma^n$, we represent the language model as the autoregressive distribution $P_{\text{lm}}(x_{1:n})=\prod_{\ell=1}^n P_{\text{lm}}(x_\ell\mid x_{1:\ell-1})$. An HMM models a joint distribution over a token sequence $x_{1:n}$ and a hidden-state sequence $y_{1:n}\in [h]^n$: $P_{\text{hmm}}(x_{1:n},y_{1:n})=P_{\text{hmm}}(y_1)P_{\text{hmm}}(x_1\mid y_1) \prod_{\ell=2}^n P_{\text{hmm}}(y_\ell\mid y_{\ell-1})P_{\text{hmm}}(x_\ell\mid y_{\ell})$. It is specified by an
initial distribution vector for $P_{\text{hmm}}(y_1=i)$, an emission matrix for $P_{\text{hmm}}(x_\ell=k\mid y_{\ell}=i)$, and a transition matrix for $P_{\text{hmm}}(y_{\ell}=j\mid y_{\ell-1}=i)$.

\paragraph{Mathematical Notation}
Throughout the paper, the tolerance and failure-probability parameters satisfy $0<\varepsilon,\delta<1$.
We use $\mathbf{1}_{A}$ to denote the indicator function that equals $1$ if $A$ is true.
Given $n\in \mathbb{N}^+$, let $[n]$ denote $\{1,\ldots,n\}$. Given $a,b\in \mathbb{R}^{\geq 0}$, let $a\in (1\pm \varepsilon)b$ denote $(1-\varepsilon)b\leq a\leq (1+\varepsilon)b$. Similarly, $a\in \frac{b}{(1\pm \varepsilon)}$ denotes $\frac{b}{1+\varepsilon}\leq a\leq \frac{b}{1-\varepsilon}$. In convention, $0^{-1}=+\infty$
We use $\lambda$ to denote the length-zero empty token or hidden state. For $x,x'\in \Sigma^*$, let $x\cdot x'$ be their concatenation. For $S,S'\subseteq \Sigma^*$, define $S\cdot x = \{x'\cdot x\mid x'\in S\}$ and $S\cdot S'=\{x\cdot x' \mid x\in S, x'\in S'\}$. Concatenation of product sequences of tokens and hidden states is defined similarly.

\paragraph{NFA and Regex}
An NFA is a 5-tuple $A=(Q,\Sigma,T,q_0,F)$ comprising a set of states $Q$, a set of transitions $T\subseteq Q\times \Sigma\times Q$, an initial state $q_0\in Q$, and a set of final states $F\subseteq Q$. A DFA is a special case of an NFA such that, for any $q\in Q$ and $a\in \Sigma$, there is at most one $q'$ with $(q,a,q')\in T$. A run $r$ of $x_{1:k}\in \Sigma^k$ from $q_0$ to $q_k$ on $A$ is denoted by $q_0\overset{x_1}{\rightarrow}q_1 \overset{x_2}{\rightarrow}\cdots\overset{x_k}{\rightarrow} q_k$, where $(q_{\ell-1},x_\ell,q_\ell)\in T$ for every $\ell\in [k]$. We say that $x$ is accepted by $q'$ from $q$ if there exists a run $r$ from $q$ to $q'$ for $x$, and we denote the set of all such $x$ by $L(q,q')$. In particular, $L(q)=L(q_0,q)$, and the language of $A$ is $L(A)=\bigcup_{q\in F}L(q)$. For $n\in \mathbb N^+$, the $n$-th slice of the language of $A$, denoted by $L_n(A)$, is the set of length-$n$ sequences in $L(A)$.
For $q,q'\in Q$, define the set of symbols labeling transitions from $q$ to $q'$ as $\Sigma(q,q')=\{a\in \Sigma\mid (q,a,q')\in T\}$, the successors of $q$ as $\mathsf{succ}(q)=\{q'\in Q\mid \Sigma(q,q')\neq \emptyset\}$, and the $a$-successors of $q$ as $\mathsf{succ}(q,a)=\{q'\in Q\mid (q,a,q')\in T\}$. Given a regular expression (Regex) $\alpha$ consisting of $k$ symbol occurrences, one can construct an NFA $A$ with $k+1$ states using Glushkov's construction~\citep{glushkov_nfa}. We abuse notation and use $P_{\text{lm}}(\alpha\mid x_{1:\ell})$ to denote the probability of the LM constraint-satisfaction event $x_{1:n}\in L_n(A)$ given $x_{1:\ell}$.
Given tolerance $\varepsilon$ and confidence $\delta$, the fully polynomial randomized approximation scheme (FPRAS) for \#NFA returns an estimate $\hat{N}$ such that $\mathbb{P}[\hat{N}\in (1\pm\varepsilon)|L_n(A)|]\geq 1-\delta$. Its running time is polynomial in $m$ (the number of states), $n$, $\varepsilon^{-1}$, and $\log(\delta^{-1})$.

\paragraph{Constrained Generation}
The main goal is to compute $P_{\text{lm}}(x_\ell \mid x_{1:\ell-1})P_{\text{lm}}(\alpha\mid x_{1:\ell})$ at each step $\ell$. We distill an HMM so that $P_{\text{hmm}}(\alpha\mid x_{1:\ell})$ approximates $P_{\text{lm}}(\alpha\mid x_{1:\ell})$, and we seek to approximate
the term $P_{\text{hmm}}(\alpha\mid x_{1:\ell})$ by our $\hat{P}_{\text{hmm}}(\alpha\mid x_{1:\ell})$. Given a prefix $x_{1:\ell-1}$, define \(P(x_\ell\mid x_{1:\ell-1},\alpha)=
{P_{\text{lm}}(x_\ell\mid x_{1:\ell-1})}P_{\text{hmm}}(\alpha\mid x_{1:\ell})/\Gamma(x_{1:\ell-1})\) and 
\(\hat{P}(x_\ell\mid x_{1:\ell-1},\alpha)=
{P_{\text{lm}}(x_\ell\mid x_{1:\ell-1})}\hat{P}_{\text{hmm}}(\alpha\mid x_{1:\ell})/\hat{\Gamma}(x_{1:\ell-1})\), where $\Gamma(x_{1:\ell-1})$ and $\hat{\Gamma}(x_{1:\ell-1})$ are the respective normalizing constants. They induce $P(x_{1:n}\mid\alpha)=\prod_{\ell=1}^nP(x_\ell\mid x_{1:\ell-1},\alpha)$ and $\hat{P}(x_{1:n}\mid\alpha)=\prod_{\ell=1}^n\hat{P}(x_\ell\mid x_{1:\ell-1},\alpha)$. We will show in Theorem~\ref{thm:distribution-preservation} that $P$ is provably close to $\hat{P}$.

\section{Canonical Runs on an NFA}
\label{sec:canonical-runs}
In this section, we summarize the unrolling procedure and canonical runs that are key to the algorithm description and theoretical analysis.
\paragraph{Unrolling} Given $n\in \mathbb{N}^+$, we can unroll the NFA $A$ into $A^u=(Q^{u},\Sigma,T^{u},q_0^0,q_F^n)$ with the property $L(A^{u})=L_n(A)$. The structure of $A^u$ resembles a directed acyclic graph (DAG), with layers of state copies $Q^\ell$ for $\ell=0,\ldots,n$. For each $q\in Q$ that is reachable from the initial state $q_0$ by a length-$\ell$ sequence, the corresponding state $q^\ell\in Q^\ell$ is reachable by that sequence in $A^u$. Formally, $A^u$ is defined as follows: (i) $q_0^0\in Q^u$; (ii) for every $\ell$ with $0<\ell<n$, if $(q_1,a,q_2)\in T$ and $q_1^{\ell-1}\in Q^u$, then $q_2^\ell\in Q^u$ and $(q_1^{\ell-1},a,q_2^\ell)\in T^u$; and (iii) $q_F^n\in Q^u$, and if $q_1^{n-1}\in Q^u$ and $(q_1,a,q_2)\in T$ for some $q_2\in F$, then $(q_1^{n-1},a,q_F^n)\in T^u$. The fixed-length case can be generalized to length ranges by padding: for every $q\in F$ reachable by a length-$j$ sequence for $n'\leq j\leq n$, add a padding chain $q^j \xrightarrow{\#} q_{\#}^{j+1}
\xrightarrow{\#}\cdots\xrightarrow{\#}
q_{\#}^{n-1}\xrightarrow{\#}q_F^n$ so that $|L(A^{u})|=\sum_{\ell=n'}^{n}|L_\ell(A)|$. For conciseness, we assume the fixed-length case throughout the paper. Given a prefix \(x_{1:\ell}\), let
\(R(x_{1:\ell})=\{q\in Q^\ell\mid x_{1:\ell}\in L(q_0^0,q)\}\) be its set of reachable states. Its set of accepting-completion suffixes is
\(L(x_{1:\ell},q_F^n)=
\bigcup_{q\in R(x_{1:\ell})}L(q,q_F^n)\).

\paragraph{Canonical Runs}
An accepted suffix may have several runs in an NFA. To count every suffix
once, fix a total order \(\prec\) on the states. For
\(q\in Q^\ell\) and \(a\cdot x'\in L(q,q_F^n)\), define its canonical
successor $\mathsf{first}(q,a,x')$ as follows:
\[
\underset{\prec}{\min}\{q'\in Q^{\ell+1}\mid
(q,a,q')\in T^u,\ x'\in L(q',q_F^n)\}
\]

\begin{definition}[Canonical run]
\label{def:canonical-run}
For \(q\in Q^\ell\) and \(x\in L(q,q_F^n)\), the canonical run of \(x\)
from \(q\), denoted by \(\mathsf{run}(x,q)\), is defined inductively as
follows:
\begin{itemize}
\item If \(q=q_F^n\), then \(x=\lambda\) and
\(\mathsf{run}(\lambda,q_F^n)=q_F^n\).
\item If \(q\neq q_F^n\), write \(x=a\cdot x'\), let
\(q'=\mathsf{first}(q,a,x')\), and set
\(\mathsf{run}(x,q)=
q\overset{a}{\rightarrow}\mathsf{run}(x',q')\).
We call \((x',q')\) the canonical child of \((x,q)\).
\end{itemize}
\end{definition}
For a reachable-state set \(R\subseteq Q^\ell\), we similarly write
\(\mathsf{first}(R,x)=
\underset{\prec}{\min}\{q\in R\mid x\in L(q,q_F^n)\}\).
This assigns every suffix in \(\bigcup_{q\in R}L(q,q_F^n)\) to one state
in \(R\).

\begin{definition}[Convergence state]
Let $r,r'$ be two runs in $A^u$ that both end at $q_F^n$. Let $r_k$ (respectively, $r_k'$) be the suffix of $r$ (respectively, $r'$) beginning at layer $k$. We have $r_n=r_n'=q_F^n$. The longest common suffix of $r$ and $r'$ is the suffix $r_k$ for the smallest $k$ such that $r_k=r_k'$. The convergence state of $r$ and $r'$ is the state reached by the longest common suffix of $r$ and $r'$. For $x,x'\in L(q,q_F^n)$, we denote by $q^{x,x'}$ the convergence state of their canonical runs $\mathsf{run}(x,q)$ and $\mathsf{run}(x',q)$.
\end{definition}

\definecolor{runCoral}{RGB}{244,182,181}
\definecolor{runPeriwinkle}{RGB}{179,179,250}
\definecolor{runGreen}{RGB}{166,214,166}
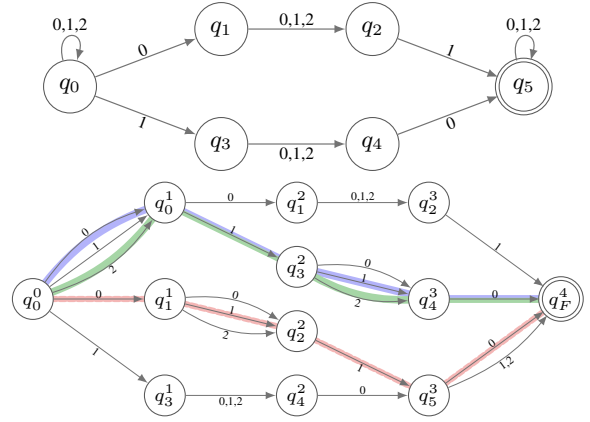
\begin{figure}[t]
\centering
\begin{tikzpicture}[
  x=1.00cm,
  y=0.95cm,
  nfastate/.style={
    circle,
    draw=black!65,
    fill=white,
    minimum size=7mm,
    inner sep=0pt,
    font=\small
  },
  finalstate/.style={
    nfastate,
    double,
    double distance=1pt
  },
  trans/.style={
    draw=black!55,
    -latex,
    line width=0.45pt
  },
  edgelabel/.style={
    font=\fontsize{6.5}{7.0}\selectfont,
    inner sep=0.25mm,
    outer sep=0pt
  }
]
  \node[nfastate] (rawq0) at (0,0) {$q_0$};
  \node[nfastate] (rawq1) at (2,0.8) {$q_1$};
  \node[nfastate] (rawq3) at (2,-0.8) {$q_3$};
  \node[nfastate] (rawq2) at (4,0.8) {$q_2$};
  \node[nfastate] (rawq4) at (4,-0.8) {$q_4$};
  \node[finalstate] (rawq5) at (6,0) {$q_5$};

  \draw[trans] (rawq0) to[loop above,looseness=6]
    node[edgelabel,above,yshift=0.35mm] {0,1,2} (rawq0);
  \draw[trans] (rawq0) --
    node[edgelabel,above,sloped] {0} (rawq1);
  \draw[trans] (rawq0) --
    node[edgelabel,below,sloped] {1} (rawq3);
  \draw[trans] (rawq1) --
    node[edgelabel,above] {0,1,2} (rawq2);
  \draw[trans] (rawq3) --
    node[edgelabel,below,yshift=-0.2mm] {0,1,2} (rawq4);
  \draw[trans] (rawq2) --
    node[edgelabel,above,sloped] {1} (rawq5);
  \draw[trans] (rawq4) --
    node[edgelabel,below,sloped] {0} (rawq5);
  \draw[trans] (rawq5) to[loop above,looseness=6]
    node[edgelabel,above,yshift=0.35mm] {0,1,2} (rawq5);
\end{tikzpicture}
\par\smallskip
\begin{tikzpicture}[
  x=0.97cm,
  y=0.85cm,
  nfastate/.style={
    circle,
    draw=black!65,
    fill=white,
    minimum size=5.4mm,
    inner sep=0pt,
    font=\scriptsize
  },
  finalstate/.style={
    nfastate,
    double,
    double distance=1pt
  },
  trans/.style={
    draw=black!55,
    -latex,
    line width=0.35pt
  },
  edgelabel/.style={
    font=\fontsize{4.4}{4.8}\selectfont,
    inner sep=0.18mm,
    outer sep=0pt
  },
  runhighlight/.style={
    line width=1.00mm,
    line cap=round,
    line join=round
  },
  noncanonical/.style={
    draw=runCoral,
    line width=0.82mm,
    line cap=round,
    line join=round,
    dash pattern=on 2pt off 1.4pt
  }
]
  \coordinate (c00) at (0,0);

  \coordinate (c01) at (1.8,1.5);
  \coordinate (c11) at (1.8,0);
  \coordinate (c31) at (1.8,-1.5);

  \coordinate (c12) at (3.6,1.5);
  \coordinate (c32) at (3.6,0.5);
  \coordinate (c22) at (3.6,-0.5);
  \coordinate (c42) at (3.6,-1.5);

  \coordinate (c23) at (5.4,1.5);
  \coordinate (c43) at (5.4,0);
  \coordinate (c53) at (5.4,-1.5);

  \coordinate (cf4) at (7.2,0);

  \draw[runhighlight,draw=runPeriwinkle]
    (c00) to[bend left=18] (c01);
  \draw[runhighlight,draw=runPeriwinkle] (c32) -- (c43);

  \draw[noncanonical] (c00) -- (c11);
  \draw[noncanonical] (c11) -- (c22);
  \draw[noncanonical] (c22) -- (c53);
  \draw[noncanonical] (c53) -- (cf4);

  \draw[runhighlight,draw=runGreen]
    (c00) to[bend right=18] (c01);
  \draw[runhighlight,draw=runGreen]
    (c32) to[bend right=18] (c43);

  \draw[draw=runPeriwinkle,line width=0.68mm,line cap=round]
    ([yshift=0.18mm]c01) -- ([yshift=0.18mm]c32);
  \draw[draw=runGreen,line width=0.68mm,line cap=round]
    ([yshift=-0.18mm]c01) -- ([yshift=-0.18mm]c32);
  \draw[draw=runPeriwinkle,line width=0.68mm,line cap=round]
    ([yshift=0.18mm]c43) -- ([yshift=0.18mm]cf4);
  \draw[draw=runGreen,line width=0.68mm,line cap=round]
    ([yshift=-0.18mm]c43) -- ([yshift=-0.18mm]cf4);

  \node[nfastate] (q00) at (c00) {$q_0^0$};

  \node[nfastate] (q01) at (c01) {$q_0^1$};
  \node[nfastate] (q11) at (c11) {$q_1^1$};
  \node[nfastate] (q31) at (c31) {$q_3^1$};

  \node[nfastate] (q12) at (c12) {$q_1^2$};
  \node[nfastate] (q32) at (c32) {$q_3^2$};
  \node[nfastate] (q22) at (c22) {$q_2^2$};
  \node[nfastate] (q42) at (c42) {$q_4^2$};

  \node[nfastate] (q23) at (c23) {$q_2^3$};
  \node[nfastate] (q43) at (c43) {$q_4^3$};
  \node[nfastate] (q53) at (c53) {$q_5^3$};

  \node[finalstate] (qf4) at (cf4) {$q_F^4$};

  \draw[trans] (q00) to[bend left=18]
    node[edgelabel,above,sloped,pos=0.5] {0} (q01);
  \draw[trans] (q00) --
    node[edgelabel,above,sloped,pos=0.5] {1} (q01);
  \draw[trans] (q00) to[bend right=18]
    node[edgelabel,below,sloped,pos=0.5] {2} (q01);
  \draw[trans] (q00) --
    node[edgelabel,above] {0} (q11);
  \draw[trans] (q00) --
    node[edgelabel,below,sloped] {1} (q31);

  \draw[trans] (q01) --
    node[edgelabel,above] {0} (q12);
  \draw[trans] (q01) --
    node[edgelabel,above,sloped] {1} (q32);
  \draw[trans] (q11) to[bend left=18]
    node[edgelabel,above,sloped,pos=0.5] {0} (q22);
  \draw[trans] (q11) --
    node[edgelabel,above,sloped,pos=0.5] {1} (q22);
  \draw[trans] (q11) to[bend right=18]
    node[edgelabel,below,sloped,pos=0.5] {2} (q22);
  \draw[trans] (q31) --
    node[edgelabel,below] {0,1,2} (q42);

  \draw[trans] (q12) --
    node[edgelabel,above] {0,1,2} (q23);
  \draw[trans] (q22) --
    node[edgelabel,below,sloped] {1} (q53);
  \draw[trans] (q32) to[bend left=18]
    node[edgelabel,above,sloped,pos=0.5] {0} (q43);
  \draw[trans] (q32) --
    node[edgelabel,above,sloped,pos=0.5] {1} (q43);
  \draw[trans] (q32) to[bend right=18]
    node[edgelabel,below,sloped,pos=0.5] {2} (q43);
  \draw[trans] (q42) --
    node[edgelabel,above] {0} (q53);

  \draw[trans] (q23) --
    node[edgelabel,above,sloped] {1} (qf4);
  \draw[trans] (q43) --
    node[edgelabel,above] {0} (qf4);
  \draw[trans] (q53) --
    node[edgelabel,above,sloped] {0} (qf4);
  \draw[trans] (q53) to[bend right=14]
    node[edgelabel,below,sloped] {1,2} (qf4);
\end{tikzpicture}
\caption{\textbf{Top}: An NFA \(A\) over $\Sigma=\{0,1,2\}$ with ordered states $q_0\prec\cdots\prec q_5$ and final state $q_5$. \textbf{Bottom}: The unrolled NFA \(A^u\) of $A$ for length \(n=4\).}
\label{fig:unrolled-nfa}
\end{figure}

\newcommand{\runpaint}[2]{%
  \tikz[baseline=(run.base)]{%
    \node[
      fill=#1,
      text=black,
      rounded corners=0.8pt,
      inner xsep=1.2pt,
      inner ysep=0.45pt,
      outer sep=0pt
    ] (run) {$\displaystyle\mathstrut #2$};%
  }%
}

\begin{example}
\label{ex:canonical-runs}
\normalfont
Figure~\ref{fig:unrolled-nfa} illustrates the constraint ``Alice and Bob
occur two positions apart,'' where \(\text{Alice}\rightarrow0\), \(\text{Bob}\rightarrow1\),
and \(\text{others}\rightarrow2\). The sequence ``\text{Introduce Bob to Alice}'' corresponds to $2120$, whose canonical run is \(\runpaint{runGreen}{
q_0^0\overset{2}{\rightarrow}q_0^1
\overset{1}{\rightarrow}q_3^2
\overset{2}{\rightarrow}q_4^3
\overset{0}{\rightarrow}q_F^4}\). Another sequence $0110$ has two accepting runs: since \(\mathsf{first}(q_0^0,0,110)=q_0^1\) under the ordering $\prec$, the canonical run is \(\runpaint{runPeriwinkle}{
q_0^0\overset{0}{\rightarrow}q_0^1
\overset{1}{\rightarrow}q_3^2
\overset{1}{\rightarrow}q_4^3
\overset{0}{\rightarrow}q_F^4}\), whereas $\runpaint{runCoral}{
q_0^0\overset{0}{\rightarrow}q_1^1
\overset{1}{\rightarrow}q_2^2
\overset{1}{\rightarrow}q_5^3
\overset{0}{\rightarrow}q_F^4}$ is noncanonical. The longest common suffix of the canonical runs for $0110$ and $2120$ is therefore \(q_4^3\overset{0}{\rightarrow}q_F^4\), with convergence state \(q_4^3\).

\end{example}

\paragraph{NFA with HMM Weights}
Let \(H_0=\{\lambda\}\) for the initial-distribution step, and let \(H_\ell=[h]\) for \(1\leq\ell\leq n\).
A product state is a pair \(v=(q,b)\), where \(q\in Q^\ell\) and
\(b\in H_\ell\). An NFA transition \((q,a,q')\in T^u\) and a hidden-state
transition from \(b\) to \(b'\in H_{\ell+1}\) form a product edge
\((q,b)\xrightarrow{(a,b')}(q',b')\) with weight defined as follows:
\[
\psi(a,b,b')
=
\begin{cases}
P_{\text{hmm}}(b')P_{\text{hmm}}(a\mid b')&b=\lambda\\
P_{\text{hmm}}(b'\mid b)P_{\text{hmm}}(a\mid b')&b\neq\lambda
\end{cases}
\]
For $q\in Q^\ell$, define $Z(q)=L(q,q_F^n)\times[h]^{n-\ell}$. Given any suffix atom \(z=(x_{\ell+1:n},y_{\ell+1:n})\), define its weight as $w(z,b)=\prod_{i=\ell+1}^n\psi(x_i,y_{i-1},y_i)$ and the sum of weights at $(q,b)$ as $W(q,b)=\sum_{z\in Z(q)}w(z,b)$. Let \(Z_+(q,b)=\{z\in Z(q)\mid w(z,b)>0\}\). We call \(\mathcal E_\ell(q,b)\) the set of {productive edges} of \((q,b)\) which is defined inductively: at $\ell=n-1$, define $\mathcal E_{n-1}(q,b)=\{(a,q_F^n,b')\mid (q,a,q_F^n)\in T^u,\ b'\in H_n,\psi(a,b,b')>0\}$, and for $0\leq\ell<n-1$, define \(\mathcal E_\ell(q,b)
=\{(a,q',b')\mid (q,a,q')\in T^u,\ b'\in H_{\ell+1},\psi(a,b,b')>0,\
\mathcal E_{\ell+1}(q',b')\neq\emptyset\}\). Observe that $w(z,b)=P_{\text{hmm}}(x_{\ell+1:n},y_{\ell+1:n}\mid y_\ell=b)$ and $W(q,b)=P_{\text{hmm}}(x_{\ell+1:n}\in L(q,q_F^n)\mid y_\ell=b)$. For each prefix \(x_{1:\ell}\) with \(P_{\text{hmm}}(x_{1:\ell})>0\), let \(R=R(x_{1:\ell})\) be the set of reachable NFA states and let \(Z_R=\bigcup_{q\in R}Z(q)\) be the set of suffix atoms. If we write \(W_R(b)=\sum_{z\in Z_R}w(z,b)\), then the target probability is given by the following identity:
\[P_{\text{hmm}}(\alpha\mid x_{1:\ell})
=
\sum_{b\in H_\ell}
P_{\text{hmm}}(y_\ell=b\mid x_{1:\ell})W_R(b)\]
It therefore suffices to estimate \(W_R(b)\), as described in Section~\ref{sec:methods}.

The canonical NFA run lifts to a canonical product run.
If \(z=(a,b')\cdot z'\), where \(z'=(x',y')\), and
\(q'=\mathsf{first}(q,a,x')\), define analogously:
\[
\mathsf{run}(z,(q,b))
=
(q,b)\xrightarrow{(a,b')}
\mathsf{run}(z',(q',b'))
\]
The recursion terminates at \(q_F^n\), where all terminal product states are
identified with \(q_F^n\). We call \((z',(q',b'))\) the canonical child of
\((z,(q,b))\). These canonical product runs will control the dependence
among reused samples in Section~\ref{sec:theoretical-analysis}.

\section{Method}
\label{sec:methods}
Overall, \framework{} consists of three stages: (i) \textit{Distillation:} following~\citet{tractablecontrol,ctrlg}, we distill an HMM \(P_{\text{hmm}}\) using responses sampled from the base LM \(P_{\text{lm}}\); (ii) \textit{Precomputation:} we precompute HMM-weighted suffix samples that provide reusable information for constrained generation; and (iii) \textit{Generation:} we use the precomputed information to generate outputs. Algorithm~\ref{alg:overview} summarizes \framework{}.

\begin{paperalgorithm}[ht!]
\DontPrintSemicolon
\SetAlgoNlRelativeSize{-1}
\mbox{Construct NFA \(A\) for \(\alpha\) and unroll to \(A^u\)}

Run \(\mathsf{precompute}(A,n,\varepsilon,\delta)\)

Initialize \(x_{1:0}=\lambda\)

\For{\(\ell=1,\ldots,n\)}{
    \For{\(a\in \Sigma\)}{\mbox{\(\hat{P}_{\text{hmm}}(\alpha\mid x_{1:\ell-1}\cdot a)\leftarrow \mathsf{compute}(x_{1:\ell-1}\cdot a)\)}
    }

    \mbox{Sample \(x_\ell\propto P_{\text{lm}}(x_\ell\mid x_{1:\ell-1})\hat{P}_{\text{hmm}}(\alpha\mid x_{1:\ell})\)}
    \nllabel{line:generation-token-sampling}
}
\Return{\(x_{1:n}\)}
\caption{\framework(\(\alpha,n,\varepsilon,\delta\))}
\label{alg:overview}
\end{paperalgorithm}

\begin{paperalgorithm}[ht!]
\DontPrintSemicolon

\SetAlgoNlRelativeSize{-1}
\For{\(b\in H_\ell\)}{
    \mbox{\lIf{\nllabel{line:precompute-empty-edges}\(\mathcal{E}_{\ell}(q,b)=\emptyset\)}{\textbf{continue} to next \(b\)}}
    \(\rho^j(q,b)\leftarrow\min_{(a,q_i,b')\in\mathcal E_{\ell}(q,b)}
    \frac{p^j(q_i,b')}
    {\psi(a,b,b')}\)\nllabel{line:precompute-common-rate}

    \For{\(r\in[n_s n_t]\)}{
        \For{\((a,q_i,b')\in\mathcal E_\ell(q,b)\)}{
            \(
            \ifarxiv
\textstyle\bar S^{r,j}(a,q,q_i,b')\leftarrow\mathsf{reduce}\biggl(
            \textstyle(a,b')\cdot S^{r,j}(q_i,b'),\frac{\rho^j(q,b)\psi(a,b,b')}{p^j(q_i,b')}\biggr)
\else
\begin{aligned}
            &\textstyle\bar S^{r,j}(a,q,q_i,b')\leftarrow\mathsf{reduce}\biggl(\\
            &\textstyle(a,b')\cdot S^{r,j}(q_i,b'),\frac{\rho^j(q,b)\psi(a,b,b')}{p^j(q_i,b')}\biggr)
            \end{aligned}
\fi
            \)\nllabel{line:precompute-expand-reduce}
        }

        \(
        \ifarxiv
\textstyle\hat{S}^{r,j}(q,b)\leftarrow
        \underset{(a,q_i,b')\in \mathcal{E}_\ell(q,b)}{\bigcup}\{(a,b')\cdot (x',y')
        \textstyle\in \bar{S}^{r,j}(a,q,q_i,b'):\mathsf{first}(q,a,x')=q_i\}
\else
\begin{aligned}
        &\textstyle\hat{S}^{r,j}(q,b)\leftarrow
        \underset{(a,q_i,b')\in \mathcal{E}_\ell(q,b)}{\bigcup}\{(a,b')\cdot (x',y')\\
        &\textstyle\in \bar{S}^{r,j}(a,q,q_i,b'):\mathsf{first}(q,a,x')=q_i\}
        \end{aligned}
\fi
        \)\nllabel{line:precompute-canonical-union}
    }

    \For{\(t\in [n_t]\)}{
        \mbox{\(
        M^{t,j}(q,b)
        \leftarrow
        \frac{\overset{n_s t}{\underset{r=n_s(t-1)+1}{\sum}}
        \left|\hat S^{r,j}(q,b)\right|}{n_s\rho^j(q,b)}
        \)}\nllabel{line:precompute-block-mean}
    }
    \(\hat W^j(q,b)
    \leftarrow \median_{t\in[n_t]} M^{t,j}(q,b)\)
    \nllabel{line:precompute-median}

    \mbox{\(p^j(q,b)\leftarrow\min\left(
    \rho^j(q,b),\ \hat W^j(q,b)^{-1}
    \right)\)}\nllabel{line:precompute-rate-update}

    \For{\(r\in [n_s n_t]\)}{
        \mbox{\(
        S^{r,j}(q,b)
        \leftarrow
        \mathsf{reduce}
        \left(
        \hat S^{r,j}(q,b),
        \frac{p^j(q,b)}{\rho^j(q,b)}
        \right)
        \)}\nllabel{line:precompute-final-reduce}
    }
}
\caption{\(\mathsf{estimateAndSample}(q,j,\ell)\)}
\label{alg:weighted-estimate-sample}
\end{paperalgorithm}

\begin{paperalgorithm}[ht!]
\DontPrintSemicolon
\SetAlgoNlRelativeSize{-1}
Unroll \(A\) to \(A^u=(Q^u,\Sigma,T^u,q_0^0,q_F^n)\)

\mbox{\(\kappa,n_t,n_u\leftarrow \frac{\varepsilon}{6+\varepsilon},\left\lceil 8\log(16h|Q^u|)\right\rceil,\left\lceil 8\log(\delta^{-1})\right\rceil\)}

\mbox{
\(n_s,\theta
\leftarrow\left\lceil\frac{16(n+1)}{\kappa^2(1-\kappa)}\right\rceil,
\left\lceil 16(1+\kappa)n_s n_t h|Q^u|\right\rceil\)}

\For{\(j\in [n_u]\)}{
    \(\mathsf{fail}^j\leftarrow0\)
    \nllabel{line:precompute-core-start}

    {\advance\rightskip by -2\algoskipindent
    \For{\mbox{\(\ell=0,\ldots,n, q\in Q^\ell, b\in H_\ell, r\in [n_s n_t]\)}}{
        \mbox{\(\hat{W}^j(q,b),p^j(q,b),S^{r,j}(q,b)\leftarrow0,1,\emptyset\)}
    }}

    \For{\(b\in [h], r\in [n_s n_t]\)}
    {\nllabel{line:precompute-terminal-init}\(\hat{W}^j(q_F^n,b),p^j(q_F^n, b),S^{r,j}(q_F^n, b)\leftarrow1,1,\{(\lambda,\lambda)\}\)}
    \For{\nllabel{line:precompute-backward-order}\(\ell=n-1,\ldots,0\)}{
        \For{\(q\in Q^\ell\)}{
            \(\mathsf{estimateAndSample}(q,j,\ell)\)

            {\advance\rightskip by -2\algoskipindent\If{\mbox{
            \(
            \overset{n_s n_t}{\underset{r=1}{\sum}}
            \overset{n}{\underset{k=0}\sum}
            \underset{q\in Q^k}{\sum}
            {\underset{b\in H_k}{\sum}}
            |S^{r,j}(q,b)|\geq \theta
            \)\nllabel{line:precompute-budget-test}
            }}{\(\mathsf{fail}^j\leftarrow1\), go to repetition \(j+1\)\nllabel{line:precompute-budget}}}
        }
    }
}

\Return{\(\hat W^j(q,b), p^j(q,b),S^{r,j}(q,b)\)}
\caption{\(\mathsf{precompute}(A,n,\varepsilon,\delta)\)}
\label{alg:hmm-weighted-count-nfa}
\SetNlSty{textbf}{c}{}
\end{paperalgorithm}

\begin{paperalgorithm}[ht!]
\DontPrintSemicolon
\SetAlgoNlRelativeSize{-1}
\mbox{Let \(R=R(x_{1:\ell})\) be the set of reachable states.}

\mbox{\lIf{\(\ell=n\)}{\Return{\(\mathbf{1}_{q_F^n\in R}\)}}}

\lIf{\(P_{\text{hmm}}(x_{1:\ell})=0\) \textbf{or} \(R=\emptyset\)}{\Return{\(0\)}}

\For{\(j\in [n_u]\)}{
    \mbox{\lIf{\(\mathsf{fail}^j=1\)}{
        \(\hat{P}_{\text{hmm}}^{\,j}(\alpha\mid x_{1:\ell})=0\); \textbf{continue}
    }}

    \For{\(b\in H_\ell\)}{
        \(R_b^+\leftarrow\{q\in R\mid \mathcal{E}_{\ell}(q,b)\neq \emptyset\}\)

        \mbox{\lIf{\(R_b^+=\emptyset\)}{
            \(\hat W_R^j(b)\gets0\); \textbf{continue}
        }}

        \(\rho_R^j(b)\gets
        \min_{q\in R_b^+} p^j(q,b)\)
        \nllabel{line:generation-query-rate}

        \For{\(r\in [n_s n_t]\)}{
            \For{\(q\in R_b^+\)}{
                \mbox{\(\bar S_R^{r,j}(q,b)\leftarrow
                \mathsf{reduce}\left(
                    S^{r,j}(q,b),
                    \frac{\rho_R^j(b)}{p^j(q,b)}
                \right)\)}
            }

            \(\widetilde S_R^{r,j}(b)=
            \underset{q\in R_b^+}{\bigcup}
            \{(x,y)\in \bar S_R^{r,j}(q,b)
            \mid \mathsf{first}(R,x)=q\}\)
        }

        \For{\(t\in [n_t]\)}{\(M_R^{t,j}(b)\gets
            \frac{\overset{n_s t}{\underset{r=n_s(t-1)+1}{\sum}}
            |\widetilde S_R^{r,j}(b)|}{n_s\rho_R^j(b)}
            \)
        }

        \(\hat W_R^j(b)\gets
        \median_{t\in[n_t]} M_R^{t,j}(b)\)
    }

    \mbox{\(\hat{P}_{\text{hmm}}^{\,j}(\alpha\mid x_{1:\ell})\gets
    \underset{b\in H_\ell}{\sum}
    P_{\text{hmm}}(y_\ell=b\mid x_{1:\ell})\hat W_R^j(b)\)}
    \nllabel{line:generation-posterior-aggregation}
}

\(\hat{P}_{\text{hmm}}(\alpha\mid x_{1:\ell})=
\underset{j\in [n_u]}{\text{median }}\text{ }
\hat{P}_{\text{hmm}}^{\,j}(\alpha\mid x_{1:\ell})\)
\nllabel{line:generation-outer-median}

\Return{\(\hat{P}_{\text{hmm}}(\alpha\mid x_{1:\ell})\)}
\caption{\(\mathsf{compute}(x_{1:\ell})\)}
\label{alg:hmm-weighted-prefix-estimate}
\end{paperalgorithm}
\subsection{Precomputation}
\label{subsec:precomputation}
In this stage, we precompute suffix sample sets \(S^{r,j}(q,b)\) at rates \(p^j(q,b)\) such that, for every atom \(z\in Z_+(q,b)\), the invariant \(\mathbb{P}[z\in S^{r,j}(q,b)]=p^j(q,b)w(z,b)\) holds.\footnote{This expression is only for intuitive understanding. A rigorous formalization of the invariant is given in Appendix~\ref{sec:weighted-probabilistic-environment}.} Define \(\mathsf{reduce}(S,p)\) to initialize \(S'=\emptyset\), add each \(s\in S\) to \(S'\) with probability \(p\), and return \(S'\). The precomputation core corresponds to the body of one repetition \(j\) in Algorithm~\ref{alg:hmm-weighted-count-nfa}, spanning Lines~\ref{line:precompute-core-start}--\ref{line:precompute-budget}. For the terminal product states, we set every rate and weight to its exact value, \(1\), and every sample set to \(\{(\lambda,\lambda)\}\) (Line~\ref{line:precompute-terminal-init}). The core then processes the product states backward, from layer \(n-1\) to layer \(0\) (Line~\ref{line:precompute-backward-order}). Thus, when \(\mathsf{estimateAndSample}(q,j,\ell)\) is called, all rates \(p^j(q_i,b')\) and sample sets \(S^{r,j}(q_i,b')\) associated with the productive edges of \((q,b)\) have already been computed.

We now analyze \(\mathsf{estimateAndSample}\). Fix \(r\in[n_s n_t]\), a productive edge \((a,q_i,b')\), and an atom
\(z=(a,b')\cdot z'\), where \(z'=(x',y')\in Z_+(q_i,b')\).
Suppose that \(z'\) occurs in \(S^{r,j}(q_i,b')\) with weighted sampling rate \(p^j(q_i,b')w(z',b')\).
We first compute
\(\rho^j(q,b)=
\min_{(a,q_i,b')\in\mathcal E_\ell(q,b)}
\frac{p^j(q_i,b')}{\psi(a,b,b')}\) (Line~\ref{line:precompute-common-rate}). We expand \(z'\) by prepending \((a,b')\), and then reduce the expanded set with probability
\(\rho^j(q,b)\psi(a,b,b')/p^j(q_i,b')\) (Line~\ref{line:precompute-expand-reduce}). The resulting occurrence rate of \(z\) is \(\rho^j(q,b)w(z,b)\). Thus, all expanded atoms have the same rate relative to their weights.

The same atom may nevertheless be generated through several NFA successors, causing overcounting. Thus, Line~\ref{line:precompute-canonical-union} uses a canonical union to retain the atom \((a,b')\cdot(x',y')\) if and only if \(q_i=\mathsf{first}(q,a,x')\). This ensures that \(|\hat S^{r,j}(q,b)|/\rho^j(q,b)\) is an unbiased estimator of \(W(q,b)\). To concentrate this estimator, we set \(\hat W^j(q,b)\) to the median of \(n_t\) block means in Line~\ref{line:precompute-median}. The new sampling rate is
\(p^j(q,b)=\min\left(\rho^j(q,b),\hat W^j(q,b)^{-1}\right)\), which guarantees \(p^j(q,b)/\rho^j(q,b)\leq1\).

To keep the precomputation bounded, a core is marked as failed once the total number of stored samples reaches the threshold \(\theta\).
Each core is accurate with constant probability, and the \(n_u\) core repetitions amplify the success probability to \(1-\delta\).

\subsection{Constrained Generation}
Given a prefix \(x_{1:\ell}\), Algorithm~\ref{alg:hmm-weighted-prefix-estimate} estimates \(W_R(b)\) by combining the precomputed results associated with the reachable states \(R=R(x_{1:\ell})\). For every hidden state \(b\), let \(R_b^+=\{q\in R\mid\mathcal E_\ell(q,b)\neq\emptyset\}\). Within each non-failed repetition \(j\), we apply the same rate alignment, canonical union, and median-of-means construction as in precomputation: the sketches are aligned at \(\rho_R^j(b)=\min_{q\in R_b^+}p^j(q,b)\) (Line~\ref{line:generation-query-rate}), reduced accordingly, and deduplicated using \(\mathsf{first}(R,x)\). This yields \(\hat W_R^j(b)\), and the HMM-posterior identity gives
\[
\hat P_{\text{hmm}}^{\,j}(\alpha\mid x_{1:\ell})
=\sum_{b\in H_\ell}
P_{\text{hmm}}(y_\ell=b\mid x_{1:\ell})\hat W_R^j(b)
\]
(Line~\ref{line:generation-posterior-aggregation}). The algorithm defines \(\hat{P}_{\text{hmm}}(\alpha\mid x_{1:\ell})\) as the median of these estimates across the \(n_u\) repetitions (Line~\ref{line:generation-outer-median}). At every generation step, Algorithm~\ref{alg:overview} evaluates this estimate for every candidate \(a\in\Sigma\) and samples \(a\) with probability proportional to \(P_{\text{lm}}(a\mid x_{1:\ell-1})\hat P_{\text{hmm}}(\alpha\mid x_{1:\ell-1}\cdot a)\).

\section{Theoretical Guarantees}
\label{sec:theoretical-analysis}
We summarize the main theoretical guarantees of \framework{} for approximation error and time complexity below. Let
\(\mathcal V
=\{(q,b):q\in Q^k,\ b\in H_k,\ 0\leq k<n,\ Z_+(q,b)\neq\emptyset\}\) be the set of product states with at least one positive-weight accepting completion.
In precomputation, for \(j\in[n_u]\), let \(\mathcal N^j\) denote the \(j\)-th core
repetition of \(\mathsf{precompute}\) without the test against \(\theta\), so
that \(\mathcal N^j\) always completes. In \(\mathcal N^j\), define the following bad events \(\mathcal A^j,\mathcal B^j\) and the good event \(\mathcal C^j\):
\begin{align*}
\mathcal A^j
&=\bigcup_{(q,b)\in\mathcal V}
\left\{p^j(q,b)\notin(1\pm\kappa)W(q,b)^{-1}\right\}\\
\mathcal B^j
&=\left\{
\sum_{r=1}^{n_sn_t}\sum_{k=0}^{n}
\sum_{q\in Q^k}\sum_{b\in H_k}|S^{r,j}(q,b)|\geq\theta
\right\}\\
\mathcal C^j
&=(\mathcal A^j)^c\cap(\mathcal B^j)^c
\end{align*}

Intuitively, for the \(j\)-th core repetition, \(\mathcal A^j\) denotes the event that the approximation error exceeds the tolerance, and \(\mathcal B^j\) denotes the overflow event in which the total number of stored samples becomes too large. Without Line~\ref{line:precompute-budget}, this overflow would violate the time-complexity guarantee. In fact, each bad event occurs with a small constant probability, and the outer median across independent core repetitions makes the overall failure probability arbitrarily small:

\begin{restatable}{lemma}{stateestimates}
\label{lem:state-estimates}
For \(j\in[n_u]\), \(\mathbb{P}[\mathcal A^j]\leq1/16\).
\end{restatable}

\begin{restatable}{lemma}{sampleoverflow}
\label{lem:sample-overflow}
For \(j\in[n_u]\), \(\mathbb{P}[\mathcal B^j\cap(\mathcal A^j)^c]\leq1/16\).

\end{restatable}

After the precomputation, for \(\ell\in[n]\), \(j\in[n_u]\), and every prefix satisfying
\(P(x_{1:\ell}\mid\alpha)>0\), define
\begin{align*}
&\xi^j(x_{1:\ell})
=
\frac{\hat P_{\text{hmm}}^{\,j}(\alpha\mid x_{1:\ell})}
{P_{\text{hmm}}(\alpha\mid x_{1:\ell})}
-1 \\
&\Delta^j
=
\sum_{\ell=1}^{n}
\sum_{\substack{x_{1:\ell}\in\Sigma^\ell\\
P(x_{1:\ell}\mid\alpha)>0}}
P(x_{1:\ell}\mid\alpha)
\left(\xi^j(x_{1:\ell})\right)^2\\
&\mathcal I^j
=
\mathcal C^j
\cap
\left\{\Delta^j\leq2n\kappa^2\right\}
\end{align*}

Intuitively, \(\xi^j(x_{1:\ell})\) is the relative error of repetition \(j\) in the HMM completion score. \(\Delta^j\) accumulates the weighted squared error across generation steps, and \(\mathcal I^j\) is the event that repetition \(j\) remains within its error budget, assuming that the good precomputation event \(\mathcal C^j\) occurs. In fact, we have:

\begin{restatable}{lemma}{integratedonecore}
\label{lem:integrated-one-core}
For \(j\in[n_u]\),
\(\mathbb{P}[(\mathcal I^j)^c]\leq3/16\).
\end{restatable}

Together, these key lemmas lead to the following guarantees for approximation error and time complexity:

\begin{restatable}{theorem}{prefixaccuracy}
\label{thm:prefix-accuracy}
For any constraint \(\alpha\), \(n\in\mathbb N^+\), parameters \(\varepsilon,\delta\), and \(x_{1:n}\in\Sigma^n\)
satisfying \(P_{\mathrm{hmm}}(x_{1:n})>0\), with probability at least
\(1-\delta\), the following holds simultaneously for all \(\ell\in[n]\):
\[
\begin{aligned}
\hat P_{\mathrm{hmm}}(\alpha\mid x_{1:\ell})
&\in (1\pm\varepsilon) P_{\mathrm{hmm}}(\alpha\mid x_{1:\ell})
\end{aligned}
\]
\end{restatable}

\begin{restatable}{theorem}{distributionpreservation}
\label{thm:distribution-preservation}
For any constraint \(\alpha\), \(n\in\mathbb N^+\), and parameters \(\varepsilon,\delta\), assume that for every \(\ell\in[n]\), \(\Gamma(x_{1:\ell-1})>0\) whenever \(P(x_{1:\ell-1}\mid\alpha)>0\). With probability at least
\(1-\delta\), the Total Variation (TV) distance is bounded as below:
\begin{equation*}
D_{\mathrm{TV}}\!\left(
\hat P(\cdot\mid\alpha),P(\cdot\mid\alpha)
\right)
\leq n\varepsilon
\label{eq:distribution-preservation}
\end{equation*}
\end{restatable}

\begin{restatable}{theorem}{timecomplexity}
\label{thm:timecomplexity}
Given any constraint \(\alpha\) represented by an NFA with \(m\) states, a maximum sequence length \(n\), an HMM with \(h\) hidden states, an alphabet of size \(|\Sigma|\), and parameters \((\varepsilon,\delta)\),
Algorithm~\ref{alg:overview} has time complexity \(O\!\left(
|\Sigma|n^2m^3h^2\varepsilon^{-2}
\log(nmh)\log(\delta^{-1})
\right)\).
\end{restatable}
We provide complete proofs in
Appendices~\ref{sec:whole_proof} and~\ref{sec:time_complexity}.

\section{Evaluation}
\label{sec:evaluation}

We design the evaluation to answer the following research questions (\textbf{RQs}): 

\begin{enumerate}[label=\textbf{RQ\arabic*.},leftmargin=*]
    \item \textbf{Scalability.} Can \framework{} efficiently scale to large Regex constraints $\alpha$?
    
    \item \textbf{Approximation Accuracy.}
    Can \framework{} accurately compute constrained probabilities?
    
    \item \textbf{Generation Quality.} Can \framework{} generate high-quality outputs?
\end{enumerate}

\paragraph{Dataset and Constraints}

\begin{table*}[ht!]
\centering
\small
\setlength{\tabcolsep}{3pt}
\begin{tabular}{lllr}
\toprule
Family & Constraint & NFA states & DFA states \\
\midrule
\texttt{kth\_last}
& The $k_1$-th token from the end is one of the tracked keywords.
& $\Theta(k_1)$ & $\Theta(2^{k_1})$ \\
\texttt{repeat\_after\_k}
& Some tracked keyword repeats $k_1$ tokens apart.
& $\Theta(k_1k_2)$ & $\Theta(k_2^{k_1})$ \\
\texttt{at\_least\_k}
& Some tracked keyword appears at least $k_1$ times.
& $\Theta(k_1k_2)$ & $\Theta(k_1^{k_2})$ \\
\texttt{exactly\_once}
& Some tracked keyword occurs exactly once.
& $\Theta(k_2)$ & $\Theta(3^{k_2})$ \\
\texttt{left\_not\_right}
& Some tracked keyword occurs on the left of the separator but not on the right.
& $\Theta(k_2)$ & $\Theta(2^{k_2})$ \\
\texttt{right\_not\_left}
& Some tracked keyword occurs on the right of the separator but not on the left.
& $\Theta(k_2)$ & $\Theta(2^{k_2})$ \\
\texttt{sets\_differ}
& Some tracked keyword occurs on exactly one side of the separator.
& $\Theta(k_2)$ & $\Theta(3^{k_2})$ \\
\texttt{both\_sides}
& Some tracked keyword occurs on both sides of the separator.
& $\Theta(k_2)$ & $\Theta(2^{k_2})$ \\
\texttt{even\_count}
& Some tracked keyword occurs a positive even number of times.
& $\Theta(k_2)$ & $\Theta(3^{k_2})$ \\
\texttt{odd\_count}
& Some tracked keyword occurs an odd number of times.
& $\Theta(k_2)$ & $\Theta(2^{k_2})$ \\
\bottomrule
\end{tabular}
\caption{The ten constraint families and their minimum numbers of NFA and DFA states.}
\label{tab:constraint-families}
\end{table*}

Following the experimental setups of prior constrained-generation work~\citep{tractablecontrol,ctrlg,dang2026mitigating}, we evaluate \framework{} on
CommonGen~\citep{commongen}, in which each instance includes three to five key concepts as input and the goal is to generate a sequence that incorporates these concepts. Table~\ref{tab:constraint-families} presents 10 Regex constraint families that are difficult for DFAs, with difficulty controlled by the parameters $k_1$ and $k_2$ (where $k_2$ is the number of tracked keywords). To increase the number of tracked keywords $k_2$, we augment CommonGen to include up to 20 tracked keywords by merging distinct raw instances. We randomly sample 500 instances from the constraint families, difficulty parameters, and tracked keywords, with the number of NFA states limited to 50. We also require each output sequence to contain between 1 and 64 tokens.

\paragraph{Models and Baselines.}
We evaluate the proprietary model GPT-5.6 Luna via the OpenAI API, as well as the open-source models Qwen3.5-2B and Gemma-4-E2B~\citep{qwen3.5,gemma4}, as LM baselines. We distill a 128-state HMM for each open-source model.
We compare our method against the distribution-agnostic framework XGrammar~\citep{xgrammar} and the distribution-aware method Ctrl-G~\citep{ctrlg}. For \framework{}, we set the theoretical parameters to $(\varepsilon,\delta)=(0.1,0.1)$. For Ctrl-G, we convert the Regex constraints to optimal-size DFAs. We conduct the experiments on one NVIDIA L40S (48 GB) GPU and 10 Intel Xeon Gold 6448Y CPU cores, with a timeout of 256 seconds.

\paragraph{Metrics.}
We use the following evaluation metrics. For \textbf{RQ1}, we report the average elapsed time and the number of successful constraint-satisfying generations completed without timing out or encountering errors. For \textbf{RQ2}, for each of the 500 instances, we attempt to compute the exact value of $P_{\text{hmm}}(\alpha\mid x_{1:\ell})$ at every prefix $1\leq \ell\leq n$ using a brute-force method. We then report the maximum relative error $\max_{1\leq \ell \leq n}\left|\frac{\hat{P}_{\text{hmm}}(\alpha\mid x_{1:\ell})}{P_{\text{hmm}}(\alpha\mid x_{1:\ell})}-1\right|$ to empirically validate the bound in Theorem~\ref{thm:prefix-accuracy}. For \textbf{RQ3}, we use GPT-5.6 Luna as a judge to grade the quality of the generated text on a scale from 1.0 to 4.0, following the prompt template and grading rubric in Appendix~\ref{sec:prompt-templates}. The overall results are shown in Table~\ref{tab:overall-results}.
\begin{table}[t!]
\centering
\small
\begin{tabular*}{\columnwidth}{@{\extracolsep{\fill}}lccc@{}}
\toprule
\textbf{Method} &
\textbf{Success} &
\textbf{\shortstack{Time (s)}} &
\textbf{Quality} \\
\midrule
\textbf{LLM-only} & & & \\
\enspace GPT-5.6 Luna & 306 (61.2\%) & 1.9 & 2.93 \\
\enspace Qwen3.5-2B & 190 (38.0\%) & 1.0 & 1.89 \\
\enspace Gemma-4-E2B & 227 (45.4\%) & 1.0 & 2.70 \\
\addlinespace
\textbf{XGrammar} & & & \\
\enspace Qwen3.5-2B & 500 (100.0\%) & 2.8 & 1.75 \\
\enspace Gemma-4-E2B & 500 (100.0\%) & 2.5 & 1.84 \\
\addlinespace
\textbf{Ctrl-G} & & & \\
\enspace Qwen3.5-2B & 226 (45.2\%) & 152.6 & 2.58 \\
\enspace Gemma-4-E2B & 217 (43.4\%) & 155.6 & 2.71 \\
\addlinespace
\framework{} \textbf{(Ours)} & & & \\
\enspace Qwen3.5-2B & 500 (100.0\%) & 27.7 & 2.44 \\
\enspace Gemma-4-E2B & 500 (100.0\%) & 28.5 & 2.48 \\
\bottomrule
\end{tabular*}
\caption{Overall results of successful constraint-satisfaction, average elapsed time, and quality. The quality score is computed over instances with constraint-satisfying outputs.}
\label{tab:overall-results}
\end{table}

\begin{figure*}[t!]
    \centering
    \includegraphics[width=0.9\textwidth]{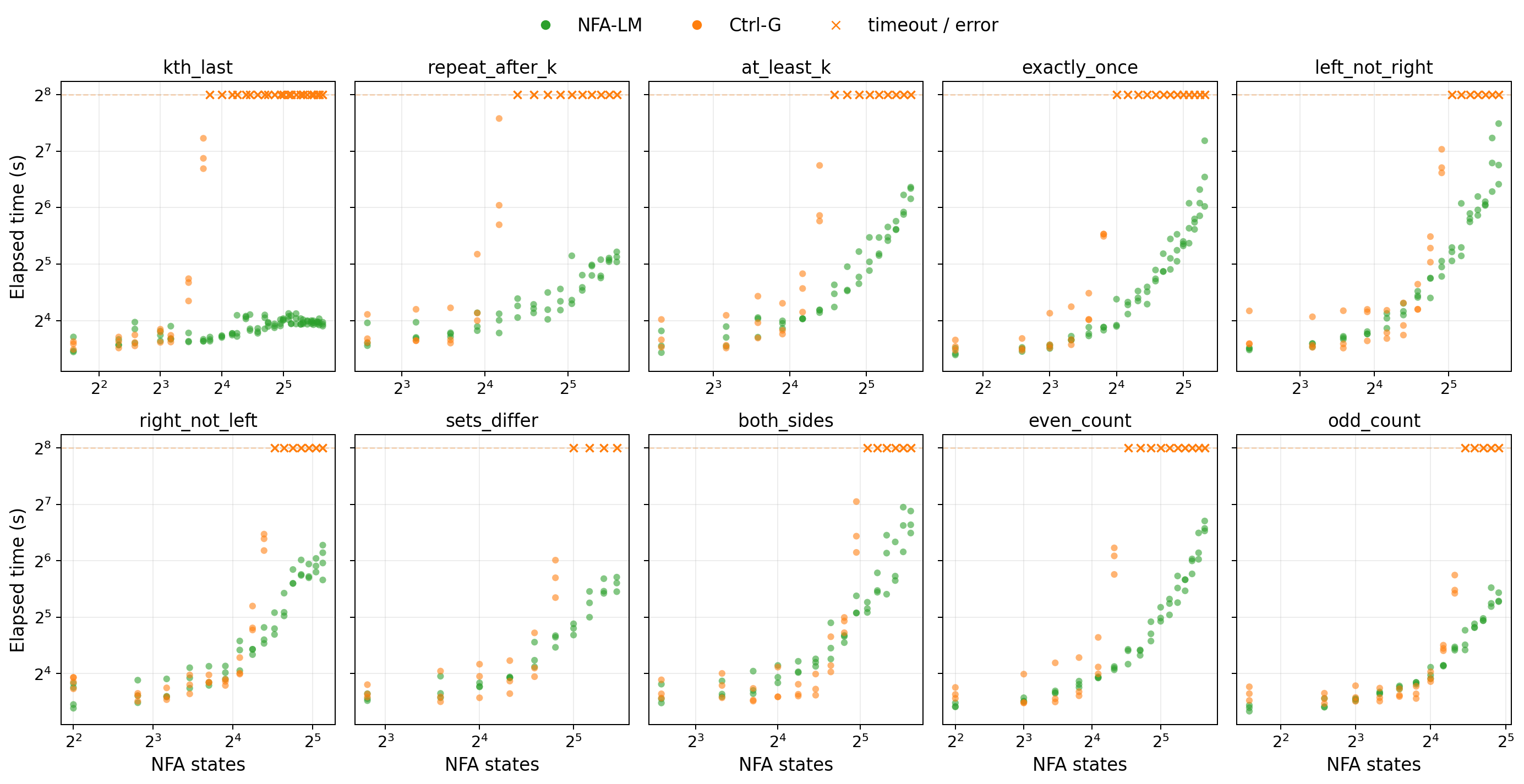}
    \caption{Elapsed time (s) of \framework{} and Ctrl-G versus the number of NFA states for each constraint family. Crosses denote instances that timed out or encountered out-of-memory errors.}
    \label{fig:precomputation-scaling}
\end{figure*}

\subsection{RQ1: Scalability}
\label{sec:rq-efficiency}
Overall, \framework{} efficiently scales to large Regex instances. 
As shown in Table~\ref{tab:overall-results}, \framework{} satisfies the constraints on all 500 instances, with an average runtime of 28.5 seconds for Gemma-4-E2B, whereas Ctrl-G succeeds on fewer than half of the instances under the timeout limit.
These compact NFAs can require exponentially larger DFAs (Table~\ref{tab:constraint-families}), and Figure~\ref{fig:precomputation-scaling} shows that Ctrl-G's failures are concentrated among larger constraints, while \framework{} remains tractable. The LLM-only baselines are the fastest but satisfy only 38.0--61.2\% of the constraints, highlighting the difficulty of the Regex constraints.

\subsection{RQ2: Approximation Accuracy}
\label{sec:rq-prefix-accuracy}
\begin{figure}[ht!]
    \centering
    \includegraphics[width=0.9\linewidth]{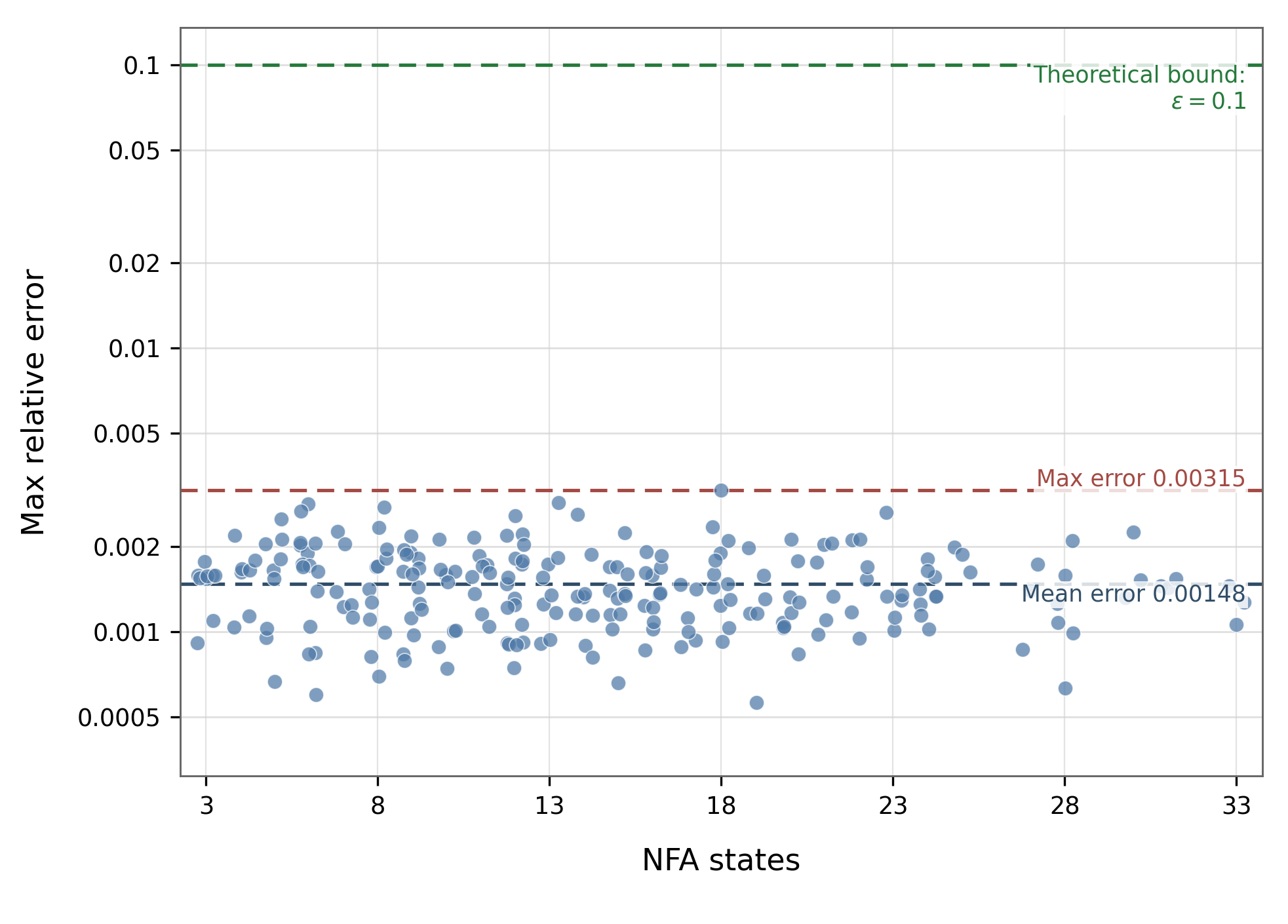}
    \caption{Maximum relative error of \framework{} against exact HMM constraint probabilities for Gemma-4-E2B.}
    \label{fig:prefix-accuracy-results}
\end{figure}

For Gemma-4-E2B, exact brute-force computation completed for 232/500 instances within the 256-second timeout, providing $P_{\text{hmm}}(\alpha\mid x_{1:\ell})$ at every prefix $\ell\in[n]$ for each of these instances. Figure~\ref{fig:prefix-accuracy-results} reports the maximum relative error $e_i=\max_{1\leq \ell\leq n}\left|\frac{\hat{P}_{\text{hmm}}(\alpha\mid x_{1:\ell})}{P_{\text{hmm}}(\alpha\mid x_{1:\ell})}
-1\right|$ for each instance $i$. The maximum relative errors for all of these instances are within the theoretical upper bound of $\varepsilon=0.1$. In fact, the worst observed error is approximately 0.00315, which is 3.15\% of the target tolerance. Overall, \framework{} computes accurate constrained-probability estimates within the theoretical error tolerance.

\subsection{RQ3: Generation Quality}
\label{sec:rq-quality}
We average quality scores over each method's constraint-satisfying outputs. Overall, Gemma-4-E2B yields higher average quality than Qwen3.5-2B for every method. On their respective successful subsets, Ctrl-G scores slightly higher than \framework{}. This is because Ctrl-G computes the exact value of ${P}_{\text{hmm}}(\alpha\mid\cdot)$, while \framework{} computes an approximate value. Both \framework{} and XGrammar succeed on all instances, but \framework{} achieves substantially higher average quality than XGrammar. The following CommonGen example uses the concepts (tracked keywords) \texttt{[hair, use, uses, iron, curl, actress]} with the \texttt{both\_sides} constraint. The word ``and'' is selected as the separator. XGrammar generates ``The actress uses an iron to curl her hair for a dramatic effect. end end ... and uses'' (with 46 repetitions of ``end'' omitted), while \framework{} generates ``The actress uses an iron to curl her hair for her role and uses the style in the film.''
Both outputs satisfy the constraint because ``uses'' appears on both sides of the separator ``and''. However, XGrammar repeats ``end'' 46 times and delays satisfying the constraint until near the length limit, earning a judge grade of 1.5. In contrast, \framework{} produces a shorter, more natural sentence with a judge grade of 3.0. Overall, \framework{} generates high-quality sequences and outperforms distribution-agnostic methods.

\section{Conclusion}
\label{sec:conclusion}
We introduced \framework{}, a polynomial-time approach to distribution-aware constrained generation for NFA constraints. It adapts the \#NFA FPRAS to estimate HMM-weighted accepting completions and uses these estimates to guide LM generation with theoretical guarantees. \framework{} efficiently generates high-quality outputs with provably bounded approximation error for NFA constraints at scale.

\section*{Limitations}

First, \framework{} focuses on Regex constraints represented by NFAs, motivated by the efficient FPRAS for \#NFA~\citep{meel2025towards}. Richer constraints such as context-free languages and nonregular Regex features such as backreferences still remain as challenges. Although an FPRAS for \#CFG was recently established~\citep{meel2026cfg}, its best-known running time is still $O(g^{14} n^{57}\varepsilon^{-4}\log(\delta^{-1}))$ for a Chomsky-normal-form grammar of size $g$, a sequence of length $n$, and error parameters $(\varepsilon,\delta)$, which remains impractical for constrained generation.

Second, \framework{} relies on a distilled HMM as a proxy for the LM's completion probabilities:
$P_{\text{hmm}}(\alpha\mid x_{1:\ell})\approx P_{\text{lm}}(\alpha\mid x_{1:\ell})$. Our guarantees bound the weighted-\#NFA approximation error relative to exact HMM-guided decoding, but they do not bound this HMM-to-LM approximation. Future work includes extending the guarantees to context-free constraints and
reducing or eliminating our reliance on the HMM.

\FloatBarrier
\bibliography{references}

@article{xgrammar,
  title={Xgrammar: Flexible and efficient structured generation engine for large language models},
  author={Dong, Yixin and Ruan, Charlie F and Cai, Yaxing and Xu, Ziyi and Zhao, Yilong and Lai, Ruihang and Chen, Tianqi},
  journal={Proceedings of Machine Learning and Systems},
  volume={7},
  year={2025}
}

@article{gorilla,
  title={Gorilla: Large language model connected with massive apis},
  author={Patil, Shishir G and Zhang, Tianjun and Wang, Xin and Gonzalez, Joseph E},
  journal={Advances in Neural Information Processing Systems},
  volume={37},
  pages={126544--126565},
  year={2024}
}

@inproceedings{bfcl,
  title={The berkeley function calling leaderboard (bfcl): From tool use to agentic evaluation of large language models},
  author={Patil, Shishir G and Mao, Huanzhi and Yan, Fanjia and Ji, Charlie Cheng-Jie and Suresh, Vishnu and Stoica, Ion and Gonzalez, Joseph E},
  booktitle={Forty-second International Conference on Machine Learning},
  year={2025}
}

@inproceedings{spider,
  title={Spider 2.0: Evaluating language models on real-world enterprise text-to-sql workflows},
  author={Lei, Fangyu and Chen, Jixuan and Ye, Yuxiao and Cao, Ruisheng and Shin, Dongchan and Su, Hongjin and Suo, Zhaoqing and Gao, Hongcheng and Hu, Wenjing and Yin, Pengcheng and others},
  booktitle={International Conference on Learning Representations},
  volume={2025},
  pages={28691--28735},
  year={2025}
}

@article{dang2026mitigating,
  title={Mitigating Bias in Locally Constrained Decoding via Tractable Proposals},
  author={Dang, Meihua and Song, Linxin and Zhang, Honghua and Zhao, Jieyu and Broeck, Guy Van den and Ermon, Stefano},
  journal={arXiv preprint arXiv:2606.01926},
  year={2026}
}

@inproceedings{meel2026cfg,
  title={\# CFG and\# DNNF admit FPRAS},
  author={Meel, Kuldeep S and de Colnet, Alexis},
  booktitle={Proceedings of the 2026 Annual ACM-SIAM Symposium on Discrete Algorithms (SODA)},
  pages={5978--6010},
  year={2026},
  organization={SIAM}
}

@inproceedings{commongen,
  title={CommonGen: A constrained text generation challenge for generative commonsense reasoning},
  author={Lin, Bill Yuchen and Zhou, Wangchunshu and Shen, Ming and Zhou, Pei and Bhagavatula, Chandra and Choi, Yejin and Ren, Xiang},
  booktitle={Findings of the Association for Computational Linguistics: EMNLP 2020},
  pages={1823--1840},
  year={2020}
}

@inproceedings{xgrammar2,
  title={XGrammar-2: Dynamic and Efficient Structured Generation Engine for Agentic LLMs},
  author={Li, Linzhang and Dong, Yixin and Wang, Guanjie and Xu, Ziyi and Jiang, Alexander and Chen, Tianqi},
  booktitle={Proceedings of the ACM Conference on AI and Agentic Systems},
  pages={1009--1022},
  year={2026}
}

@inproceedings{picard,
  title={PICARD: Parsing incrementally for constrained auto-regressive decoding from language models},
  author={Scholak, Torsten and Schucher, Nathan and Bahdanau, Dzmitry},
  booktitle={Proceedings of the 2021 conference on empirical methods in natural language processing},
  pages={9895--9901},
  year={2021}
}

@article{gad,
  title={Grammar-aligned decoding},
  author={Park, Kanghee and Wang, Jiayu and Berg-Kirkpatrick, Taylor and Polikarpova, Nadia and D'Antoni, Loris},
  journal={Advances in Neural Information Processing Systems},
  volume={37},
  pages={24547--24568},
  year={2024}
}

@article{approxaligned,
  title={Approximately aligned decoding},
  author={Melcer, Daniel and Gonugondla, Sujan Kumar and Perera, Pramuditha and Qian, Haifeng and Chiang, Wen-Hao and Wang, Yanjun and Jain, Nihal and Garg, Pranav and Ma, Xiaofei and Deoras, Anoop},
  journal={Advances in Neural Information Processing Systems},
  volume={38},
  pages={11445--11479},
  year={2026}
}

@article{jsonschemabench,
  title={Jsonschemabench: A rigorous benchmark of structured outputs for language models},
  author={Geng, Saibo and Cooper, Hudson and Moskal, Micha{\l} and Jenkins, Samuel and Berman, Julian and Ranchin, Nathan and West, Robert and Horvitz, Eric and Nori, Harsha},
  journal={arXiv preprint arXiv:2501.10868},
  year={2025}
}

@article{gemma4,
  title={Gemma 4 technical report},
  author={Team, Gemma and Abd, Sherif El and Aggarwal, Vaibhav and Algayres, Robin and Andreev, Alek and Bachem, Olivier and Ballantyne, Ian and Brick, Cormac and C{\u{a}}rbune, Victor and Casbon, Michelle and others},
  journal={arXiv preprint arXiv:2607.02770},
  year={2026}
}

@article{qwen3.5,
  title={Qwen3. 5-omni technical report},
  author={Team, Qwen},
  journal={arXiv preprint arXiv:2604.15804},
  year={2026}
}

@article{synchromesh,
  title={Synchromesh: Reliable code generation from pre-trained language models},
  author={Poesia, Gabriel and Polozov, Oleksandr and Le, Vu and Tiwari, Ashish and Soares, Gustavo and Meek, Christopher and Gulwani, Sumit},
  journal={arXiv preprint arXiv:2201.11227},
  year={2022}
}

@inproceedings{mucola,
  title={Gradient-based constrained sampling from language models},
  author={Kumar, Sachin and Paria, Biswajit and Tsvetkov, Yulia},
  booktitle={Proceedings of the 2022 Conference on Empirical Methods in Natural Language Processing},
  pages={2251--2277},
  year={2022}
}

@article{qin2022cold,
  title={Cold decoding: Energy-based constrained text generation with langevin dynamics},
  author={Qin, Lianhui and Welleck, Sean and Khashabi, Daniel and Choi, Yejin},
  journal={Advances in Neural Information Processing Systems},
  volume={35},
  pages={9538--9551},
  year={2022}
}

@book{kleene1956representation,
  title={Representation of events in nerve nets and finite automata},
  author={Kleene, Stephen Cole},
  volume={34},
  year={1956},
  publisher={Princeton University Press Princeton}
}

@article{amini2023structured,
  title={Structured voronoi sampling},
  author={Amini, Afra and Du, Li and Cotterell, Ryan},
  journal={Advances in Neural Information Processing Systems},
  volume={36},
  pages={31689--31716},
  year={2023}
}

@inproceedings{lew2023sequential,
  title={Sequential Monte Carlo Steering of Large Language Models using Probabilistic Programs},
  author={Lew, Alexander K and Zhi-Xuan, Tan and Grand, Gabriel and Mansinghka, Vikash},
  booktitle={ICML 2023 Workshop: Sampling and Optimization in Discrete Space},
  year={2023}
}

@inproceedings{mcmc_cfg,
  title={Constrained Sampling for Language Models Should Be Easy: An MCMC Perspective},
  author={Gonzalez, Emmanuel Anaya and Vaidya, Sairam and Park, Kanghee and Ji, Ruyi and Berg-Kirkpatrick, Taylor and D'Antoni, Loris},
  booktitle={The Thirty-ninth Annual Conference on Neural Information Processing Systems},
  year={2025}
}

@inproceedings{tractablecontrol,
  title={Tractable control for autoregressive language generation},
  author={Zhang, Honghua and Dang, Meihua and Peng, Nanyun and Van den Broeck, Guy},
  booktitle={International Conference on Machine Learning},
  pages={40932--40945},
  year={2023},
  organization={PMLR}
}

@article{ctrlg,
  title={Adaptable logical control for large language models},
  author={Zhang, Honghua and Kung, Po-Nien and Yoshida, Masahiro and Van den Broeck, Guy and Peng, Nanyun},
  journal={Advances in Neural Information Processing Systems},
  volume={37},
  pages={115563--115587},
  year={2024}
}

@misc{sun2023evaluatinglargelanguagemodels,
      title={Evaluating Large Language Models on Controlled Generation Tasks}, 
      author={Jiao Sun and Yufei Tian and Wangchunshu Zhou and Nan Xu and Qian Hu and Rahul Gupta and John Frederick Wieting and Nanyun Peng and Xuezhe Ma},
      year={2023},
      eprint={2310.14542},
      archivePrefix={arXiv},
      primaryClass={cs.CL},
      url={https://arxiv.org/abs/2310.14542}, 
}

@inproceedings{promptconstraints,
  title={Bounding the capabilities of large language models in open text generation with prompt constraints},
  author={Lu, Albert and Zhang, Hongxin and Zhang, Yanzhe and Wang, Xuezhi and Yang, Diyi},
  booktitle={Findings of the Association for Computational Linguistics: EACL 2023},
  pages={1982--2008},
  year={2023}
}

@misc{trace,
      title={TRACE Back from the Future: A Probabilistic Reasoning Approach to Controllable Language Generation}, 
      author={Gwen Yidou Weng and Benjie Wang and Guy Van den Broeck},
      year={2025},
      eprint={2504.18535},
      archivePrefix={arXiv},
      primaryClass={cs.CL},
      url={https://arxiv.org/abs/2504.18535}, 
}

@article{alvarez1993very,
  title={A very hard log-space counting class},
  author={{\`A}lvarez, Carme and Jenner, Birgit},
  journal={Theoretical Computer Science},
  volume={107},
  number={1},
  pages={3--30},
  year={1993},
  publisher={Elsevier}
}

@article{arenas2021nfa,
  title={{\#}NFA Admits an FPRAS: Efficient Enumeration, Counting, and Uniform Generation for Logspace Classes},
  author={Arenas, Marcelo and Croquevielle, Luis Alberto and Jayaram, Rajesh and Riveros, Cristian},
  journal={Journal of the ACM (JACM)},
  volume={68},
  number={6},
  pages={1--40},
  year={2021},
  publisher={ACM New York, NY}
}

@article{meel2025towards,
author = {Meel, Kuldeep S. and de Colnet, Alexis},
title = {Towards Practical FPRAS for {\#}NFA: Exploiting the Power of Dependence},
year = {2025},
issue_date = {May 2025},
publisher = {Association for Computing Machinery},
address = {New York, NY, USA},
volume = {3},
number = {2},
url = {https://doi.org/10.1145/3725253},
doi = {10.1145/3725253},
journal = {Proc. ACM Manag. Data},
month = jun,
articleno = {116},
numpages = {23}
}

@article{thompson_nfa,
  title={Programming techniques: Regular expression search algorithm},
  author={Thompson, Ken},
  journal={Communications of the ACM},
  volume={11},
  number={6},
  pages={419--422},
  year={1968},
  publisher={ACM New York, NY, USA}
}

@article{glushkov_nfa,
  title={The abstract theory of automata},
  author={Glushkov, Victor Mikhaylovich},
  journal={Russian Mathematical Surveys},
  volume={16},
  number={5},
  pages={1--53},
  year={1961}
}

@article{moore1971bounds,
  title={On the bounds for state-set size in the proofs of equivalence between deterministic, nondeterministic, and two-way finite automata},
  author={Moore, Frank R},
  journal={IEEE Transactions on computers},
  volume={100},
  number={10},
  pages={1211--1214},
  year={1971},
  publisher={IEEE}
}

@misc{openai_function_calling_2024,
  author       = {{OpenAI}},
  title        = {Function Calling},
  year         = {2024},
  howpublished = {\url{https://platform.openai.com/docs/guides/function-calling}},
  note         = {OpenAI API documentation}
}

@inproceedings{realtoxicityprompts,
  title={Realtoxicityprompts: Evaluating neural toxic degeneration in language models},
  author={Gehman, Samuel and Gururangan, Suchin and Sap, Maarten and Choi, Yejin and Smith, Noah A},
  booktitle={Findings of the association for computational linguistics: EMNLP 2020},
  pages={3356--3369},
  year={2020}
}

@inproceedings{zhao2024probabilistic,
  title = {Probabilistic Inference in Language Models via
           Twisted Sequential {Monte Carlo}},
  author = {Zhao, Stephen and Brekelmans, Rob and Makhzani, Alireza
            and Grosse, Roger Baker},
  booktitle = {Proceedings of the 41st International Conference
               on Machine Learning},
  year = {2024},
  volume = {235},
  series = {Proceedings of Machine Learning Research},
  pages = {60704--60748},
  publisher = {PMLR},
  url = {https://proceedings.mlr.press/v235/zhao24c.html}
}

@inproceedings{loula2025syntactic,
  title = {Syntactic and Semantic Control of Large Language Models
           via Sequential {Monte Carlo}},
  author = {Loula, Jo{\~a}o and LeBrun, Benjamin and Du, Li
            and Lipkin, Ben and Pasti, Clemente and Grand, Gabriel
            and Liu, Tianyu and Emara, Yahya and Freedman, Marjorie
            and Eisner, Jason and Cotterell, Ryan
            and Mansinghka, Vikash and Lew, Alexander K.
            and Vieira, Tim and O'Donnell, Timothy J.},
  booktitle = {Proceedings of the International Conference
               on Learning Representations},
  year = {2025},
  url = {https://proceedings.iclr.cc/paper_files/paper/2025/hash/a2d537e69a6c6638a3630eef835f07de-Abstract-Conference.html}
}

@inproceedings{lipkin2025fast,
  title = {Fast Controlled Generation from Language Models
           with Adaptive Weighted Rejection Sampling},
  author = {Lipkin, Benjamin and LeBrun, Benjamin
            and Vigly, Jacob Hoover and Loula, Jo{\~a}o
            and MacIver, David R. and Du, Li and Eisner, Jason
            and Cotterell, Ryan and Mansinghka, Vikash
            and O'Donnell, Timothy J. and Lew, Alexander K.
            and Vieira, Tim},
  booktitle = {Proceedings of the Conference on Language Modeling},
  year = {2025},
  url = {https://openreview.net/forum?id=yGQqTuSJPK}
}

\onecolumn
\arxivtrue
\appendix

\section{Cached Membership}
\label{sec:cache}
In Algorithm~\ref{alg:hmm-weighted-count-nfa}, the canonical union requires membership tests of the form
\(x'\in L(q',q_F^n)\). Naive membership checking takes
\(O(|Q^u|)\) time. We can instead maintain a per-layer cache incrementally, making
each check constant-time with respect to \(|Q^u|\).
We implement these tests using two Boolean matrices
\(\mathsf{cache}_\ell^j\) and \(\mathsf{cache}_\ell^{\prime j}\)
for every layer \(\ell\) and every independent repetition \(j\). Let
\(\mathcal S_\ell^j=\bigcup_{\substack{r\in[n_sn_t],\,q\in Q^\ell,\,b\in H_\ell}}
\{x_{\ell+1:n}\mid
(x_{\ell+1:n},y_{\ell+1:n})\in S^{r,j}(q,b)\}\). The completed cache
\(\mathsf{cache}_\ell^j\) is a Boolean matrix over \(\{0,1\}\) with
dimensions \(|\mathcal S_\ell^j|\times|Q^\ell|\), and it is correct when
\(\mathsf{cache}_\ell^j(x,q)=\mathbf{1}_{x\in L(q,q_F^n)}\) for every \(x\in\mathcal S_\ell^j\) and \(q\in Q^\ell\). \(\mathsf{computeCache}^j(\ell)\) is defined as follows.  At the terminal layer $\ell=n$,
\(\mathcal S_n^j=\{\lambda\}\) and \(\mathsf{cache}_n^j(\lambda,q_F^n)=1\). For \(0\leq\ell<n\), denote the candidate row set
\(\mathcal S_\ell^{\prime j}
=
\Sigma\cdot\mathcal S_{\ell+1}^j\). For \(a\in\Sigma\), let
\(\mathsf{transition}_{\ell+1,\ell}^a\) be the
\(|Q^{\ell+1}|\times|Q^\ell|\) Boolean matrix
\(\mathsf{transition}_{\ell+1,\ell}^a(q_i,q)
=
\mathbf{1}_{(q,a,q_i)\in T^u}\). Writing \(\Sigma=(a_1,\ldots,a_{|\Sigma|})\),
\(\mathsf{computeCache}^j(\ell)\) computes the following matrix and
normalizes each entry to either \(0\) or \(1\), setting entries greater than
\(1\) to \(1\):
\[\mathsf{cache}_\ell^{\prime j}=
\left(
\begin{array}{c}
\mathsf{cache}_{\ell+1}^j
\times\mathsf{transition}_{\ell+1,\ell}^{a_1}\\
\vdots\\
\mathsf{cache}_{\ell+1}^j
\times\mathsf{transition}_{\ell+1,\ell}^{a_{|\Sigma|}}
\end{array}
\right)
\]
Thus,
\(\mathsf{cache}_\ell^{\prime j}\) has dimension
\((|\Sigma|\cdot|\mathcal S_{\ell+1}^j|)\times|Q^\ell|
\) and satisfies
\(\mathsf{cache}_\ell^{\prime j}(a\cdot x',q)=\mathbf{1}_{a\cdot x'\in L(q,q_F^n)}\). 
Since
\(\mathcal S_\ell^j\subseteq\mathcal S_\ell^{\prime j}\),
\(\mathsf{updateCache}^j(\ell)\) obtains
\(\mathsf{cache}_\ell^j\) by keeping exactly the rows indexed by \(\mathcal S_\ell^j\). For the $a$-successors $(q_1,\ldots,q_k)$ of $q$, the canonical union includes
\((a,b')\cdot(x',y')\) from successor \(q_i\) iff
\(\mathsf{cache}_{\ell+1}^j(x',q_{i'})=0\) for every $i'<i$.  Accordingly, \(\mathsf{computeCache}^j(\ell)\) is called before
processing \(Q^\ell\), and \(\mathsf{updateCache}^j(\ell)\) afterward. Cache-augmented Algorithm~\ref{alg:hmm-weighted-count-nfa} is displayed in Algorithm~\ref{alg:hmm-weighted-count-nfa-with-cache}.

\ifarxiv\begin{algorithm}[!htbp]\else\begin{algorithm}[t]\fi
\DontPrintSemicolon
\SetAlgoNlRelativeSize{-1}
Unroll $A$ to \(A^u=(Q^u,\Sigma,T^u,q_0^0,q_F^n)\)

\mbox{\(\kappa,n_t,n_u\leftarrow \frac{\varepsilon}{6+\varepsilon},\left\lceil 8\log(16h|Q^u|)\right\rceil,\left\lceil 8\log(\delta^{-1})\right\rceil\)}

\mbox{
\(n_s,\theta
\leftarrow\left\lceil\frac{16(n+1)}{\kappa^2(1-\kappa)}\right\rceil,
\left\lceil 16(1+\kappa)n_s n_t h|Q^u|\right\rceil\)}

    \For{\(j\in[n_u]\)}{
    $\mathsf{fail}^j\leftarrow0$
    
    \For{$\ell=0,\ldots,n, q\in Q^\ell, b\in H_\ell, r\in[n_sn_t]$}{
        \mbox{$\hat{W}^j(q,b),p^j(q,b),S^{r,j}(q,b)\leftarrow0,1,\emptyset$}
    }

    \For{$b\in H_n, r\in[n_sn_t]$}{\mbox{$\hat{W}^j(q_F^n,b),p^j(q_F^n,b),S^{r,j}(q_F^n,b)\leftarrow1,1,\{(\lambda,\lambda)\}$}}

    \textcolor{red}{\(\mathsf{computeCache}^j(n)\)}

    \For{$\ell=n-1,\ldots,0$}{
        \textcolor{red}{\(\mathsf{computeCache}^j(\ell)\)}
        
        \For{$q\in Q^\ell$}{
            \(\mathsf{estimateAndSample}(q,j,\ell)\)

            {\advance\rightskip by -2\algoskipindent\If{\mbox{
            \(
            \overset{n_sn_t}{\underset{r=1}{\sum}}
            \overset{n}{\underset{k=0}\sum}
            \underset{q\in Q^k}{\sum}
            {\underset{b\in H_k}{\sum}}
            |S^{r,j}(q,b)|\geq \theta
            \)
            }}{$\mathsf{fail}^j\leftarrow1$, go to repetition \(j+1\)}}
        }
        \textcolor{red}{\(\mathsf{updateCache}^j(\ell)\)%
        }
    }
}
\Return{\(\hat W^j(q,b), p^j(q,b),S^{r,j}(q,b)\)}
\caption{\(\mathsf{precompute}(A,n,\varepsilon,\delta)\) \textcolor{red}{with cached membership}}
\label{alg:hmm-weighted-count-nfa-with-cache}
\SetNlSty{textbf}{c}{}
\end{algorithm}

\section{Theoretical Analysis}
\label{sec:whole_proof}
We first review the key definitions and basic probability facts in
Sections~\ref{subsec:canonical-product-runs} and~\ref{subsec:probability-tools}
before presenting the full proof.
\subsection{Canonical Product Runs}
\label{subsec:canonical-product-runs}

Runs in \(A^u\) are seen as paths from a state \(q\) to the accepting state
\(q_F^n\). A token suffix can have several accepting runs. For an atom
\(z\in Z(q)\), we map \((z,(q,b))\) to a unique accepting product run, called
the canonical product run of \(z\) for \((q,b)\).

For every \(j\in[n_u]\), all terminal product states \((q_F^n,b')\) have the
same deterministic sample set \(\{(\lambda,\lambda)\}\) and satisfy
\(p^j(q_F^n,b')=W(q_F^n,b')=1\). We therefore identify them with the single
terminal state \(q_F^n\) in the proof. The final hidden state \(b'\) remains
in the label of the edge entering \(q_F^n\).

\begin{definition}[Canonical product run]
Let \(z\in Z(q)\). The canonical product run of \(z\) for \((q,b)\), denoted
by \(\mathsf{run}(z,(q,b))\), is defined inductively as follows:
\begin{itemize}
    \item If \(q=q_F^n\), then \(z=(\lambda,\lambda)\) and
    \(\mathsf{run}((\lambda,\lambda),(q_F^n,b))=q_F^n\).
    \item If \(q\neq q_F^n\), write \(z=(a,b')\cdot z'\), 
    where \(z'=(x',y')\), and let
    \(q'=\mathsf{first}(q,a,x')\). Then
    \(\mathsf{run}(z,(q,b))=(q,b)\xrightarrow{(a,b')}
    \mathsf{run}(z',(q',b'))\). We say that
    \((z',(q',b'))\) is the canonical child of \((z,(q,b))\).
\end{itemize}
When \(q'=q_F^n\), the product state \((q',b')\) is treated as the terminal
state \(q_F^n\).
\end{definition}

Fix \(v=(q_\ell,y_\ell)\), where \(q_\ell\in Q^\ell\), and let
\(z=(x_{\ell+1:n},y_{\ell+1:n})\in Z(q_\ell)\). We write its canonical product run as
\(\mathsf{run}(z,v)=(q_\ell,y_\ell)
\xrightarrow{(x_{\ell+1},y_{\ell+1})}\cdots
\xrightarrow{(x_n,y_n)}q_F^n\).

\begin{definition}[Convergence state]
Let \(z_1,z_2\in Z_+(q,b)\). The longest common suffix of
\(\mathsf{run}(z_1,(q,b))\) and \(\mathsf{run}(z_2,(q,b))\) is the longest sub-run having the same states and edge labels to the terminal state. The convergence state,
denoted by \(v^{z_1,z_2}\), is the first state of this common suffix. When
\(z_1=z_2\), we define \(v^{z_1,z_2}=(q,b)\).
\end{definition}
Let \(z=(x_{\ell+1:n},y_{\ell+1:n})\in Z_+(q_\ell,y_\ell)\). For \(\ell\leq k<n\), define
\(D(z,(q_\ell,y_\ell),k)
=\{z'\in Z_+(q_\ell,y_\ell):v^{z,z'}=(q_k,y_k)\}\).
Thus, \(D(z,(q_\ell,y_\ell),k)\) contains the atoms whose canonical product
runs have the same suffix as the run of \(z\) from \((q_k,y_k)\), but not
from the preceding state. 

\begin{restatable}{proposition}{weightedsplicing}
\label{lem:weighted-splicing}
Fix \(v=(q_\ell,y_\ell)\) and
\(z=(x_{\ell+1:n},y_{\ell+1:n})\in Z_+(q_\ell,y_\ell)\). Let
\(D(z,v,k)\) be the set of atoms whose convergence state with \(z\) is
\((q_k,y_k)\) for \(\ell\leq k<n\). Then,
\[
\frac{\sum_{z'\in D(z,v,k)}w(z',y_\ell)}
{w((x_{k+1:n},y_{k+1:n}),y_k)}
\leq \frac{W(q_\ell,y_\ell)}{W(q_k,y_k)}
\]
\end{restatable}

\begin{proof}
Let \(D(z,(q_\ell,y_\ell),k)=\{z_1,z_2,\ldots\}\) and let
\(z^{\text{suffix}}=(x_{k+1:n},y_{k+1:n})\). By definition, every \(z_i\) is of the
form \(z_i^{\text{prefix}}\cdot z^{\text{suffix}}\) for some product prefix \(z_i^{\text{prefix}}\), and the
prefixes \(z_i^{\text{prefix}}\) are pairwise unequal. If
\(z^{\text{suffix}'}\in Z_+(q_k,y_k)\), then
\(z_i^{\text{prefix}}\cdot z^{\text{suffix}'}\in Z_+(q_\ell,y_\ell)\). Therefore, all atoms in
\(\{z_i^{\text{prefix}}\cdot z^{\text{suffix}'}:z_i\in D(z,(q_\ell,y_\ell),k),
z^{\text{suffix}'}\in Z_+(q_k,y_k)\}\) are distinct atoms in
\(Z_+(q_\ell,y_\ell)\).

Write \(D_k=D(z,(q_\ell,y_\ell),k)\) and
\(Z_k=Z_+(q_k,y_k)\). By the HMM factorization at \(y_k\),
\ifarxiv
\begin{equation*}
\begin{aligned}
&W(q_\ell,y_\ell)
\geq
\sum_{z_i\in D_k}\sum_{z'\in Z_k}
w(z_i^{\text{prefix}}\cdot z',y_\ell)
=
\frac{\sum_{z_i\in D_k}w(z_i,y_\ell)}
     {w(z^{\text{suffix}},y_k)}
\sum_{z'\in Z_k}w(z',y_k)\\
&=
\frac{W(q_k,y_k)}
     {w(z^{\text{suffix}},y_k)}
\sum_{z_i\in D_k}
w(z_i,y_\ell)
\end{aligned}
\end{equation*}
\else
\begin{align*}
W(q_\ell,y_\ell)
&\geq
\sum_{z_i\in D_k}\sum_{z'\in Z_k}
w(z_i^{\text{prefix}}\cdot z',y_\ell)\\
&=
\frac{\sum_{z_i\in D_k}w(z_i,y_\ell)}
     {w(z^{\text{suffix}},y_k)}
\sum_{z'\in Z_k}w(z',y_k)\\
&=
\frac{W(q_k,y_k)}
     {w(z^{\text{suffix}},y_k)}
\sum_{z_i\in D_k}
w(z_i,y_\ell)
\end{align*}
\fi
The result follows.
\end{proof}

\subsection{Probability Basics}
\label{subsec:probability-tools}
Let \(E\subseteq\Omega\) be an event, and let \(X\) be a random variable on
\(\Omega\). The conditional expectation \(\mathbb E[X\mid E]\) is defined as
\[
 \mathbb E[X\mid E]
 =
 \begin{cases}
  \mathbb E[\mathbf{1}_{E} X]/\mathbb{P}[E],&\text{if } \mathbb{P}[E]>0\\
  0,&\text{otherwise}
 \end{cases}
\]
Consider a sequence of random variables
\(\mathcal F=(Y_1,\ldots,Y_k)\), and, for every execution
\(\omega\in\Omega\), let
\(\mathcal F(\omega)=(Y_1(\omega),\ldots,Y_k(\omega))\). Then
\(\mathbb{E}[X\mid Y_1,\ldots,Y_k]\) is a random variable from \(\Omega\) to
\(\mathbb{R}\), defined as follows. Let
\(\Omega_1\sqcup\cdots\sqcup\Omega_{k'}=\Omega\) be a partition such that
two executions \(\omega_1,\omega_2\) belong to the same \(\Omega_i\) if and
only if \(\mathcal F(\omega_1)=\mathcal F(\omega_2)\). Let \(\Omega_\omega\)
be the part of the partition containing \(\omega\). Then
\(\mathbb E[X\mid Y_1,\ldots,Y_k](\omega)=\mathbb{E}[X\mid\Omega_\omega]\).
We say that \(X\) is fixed by \(\mathcal F\) (i.e., \(X\) is
deterministically computed by \(\mathcal F\)) if, for every
\(\omega_1,\omega_2\in\Omega\) satisfying
\(\mathcal F(\omega_1)=\mathcal F(\omega_2)\), we have
\(X(\omega_1)=X(\omega_2)\).

We will use the following well-known facts.

\begin{fact}
Given random variables \(X\) and \(Y\) and a sequence of random variables
\(\mathcal F\), if \(X\) is fixed by \(\mathcal F\), then
\(\mathbb{E}[XY\mid\mathcal F]=X\mathbb E[Y\mid\mathcal F]\) and
\(\mathbb{E}[X\mid\mathcal F]=X\).
\end{fact}

\begin{fact}
    (Tower rule) Given random variable $X$ and sequences of random variables $\mathcal F_1,\mathcal{F}_2$, if $\mathcal F_1\subseteq \mathcal F_2$, there is $\mathbb E[\mathbb E[X\mid \mathcal F_2]\mid \mathcal F_1]=\mathbb E[X\mid \mathcal F_1]$ and $\mathbb E[\mathbb E[X\mid \mathcal F_2]]=\mathbb E[X]$. 
\end{fact}

\begin{fact}[Intersection tail bound]
\label{fact:intersection-tail}
Let $B\geq 1$ and $0\leq p\leq 1/2$.  Suppose that events
$E_1,\ldots,E_B$ satisfy
\(
   \mathbb{P}\!\left[\bigcap_{t\in F}E_t\right]
   \leq p^{|F|}
\)
for every $F\subseteq[B]$. Then
\(
   \mathbb{P}\!\left[\sum_{t=1}^{B}\mathbf{1}_{E_t}\geq B/2\right]
   \leq (4p)^{B/2}
\).
\end{fact}

\begin{proof}
Put $X=\sum_{t=1}^{B}\mathbf{1}_{E_t}$ and $k=\lceil B/2\rceil$.
On the event $X\geq B/2$ we have $\binom{X}{k}\geq1$. By Markov's inequality, we have
\ifarxiv
\begin{equation*}
\mathbb{P}[X\geq B/2] \leq \mathbb E\!\left[\binom{X}{k}\right]=\sum_{\substack{F\subseteq[B]\\|F|=k}}
       \mathbb{P}\!\left[\bigcap_{t\in F}E_t\right]
\leq \binom{B}{k}p^k\leq 2^Bp^{B/2}=(4p)^{B/2}
\end{equation*}
\else
\begin{align*}
&\mathbb{P}[X\geq B/2] \leq \mathbb E\!\left[\binom{X}{k}\right]=\sum_{\substack{F\subseteq[B]\\|F|=k}}
       \mathbb{P}\!\left[\bigcap_{t\in F}E_t\right]\\
&\leq \binom{B}{k}p^k\leq 2^Bp^{B/2}=(4p)^{B/2}
\end{align*}
\fi
\end{proof}

\begin{definition}[Total variation distance]
For probability distributions \(P\) and \(Q\) over a finite domain \(\mathcal X\),
their total variation (TV) distance is
\(D_{\mathrm{TV}}(P,Q)
=\frac12\sum_{x\in\mathcal X}|P(x)-Q(x)|\).
It satisfies \(0\leq D_{\mathrm{TV}}(P,Q)\leq1\) and the triangle
inequality.
\end{definition}

\begin{fact}[Markov's inequality]
For any random variable $X\geq 0$ with $a>0$, 
\(\mathbb{P}[X\geq a]\leq\mathbb E[X]/a\).
\end{fact}

\begin{fact}[Hoeffding's inequality]
If \(X_1,\ldots,X_N\) are independent random variables taking values in
\([0,1]\), then, for every \(t>0\),
\(\mathbb{P}\!\left[
\sum_{i=1}^N X_i-\mathbb E\!\left[\sum_{i=1}^N X_i\right]\geq tN
\right]
\leq e^{-2Nt^2}\).

\end{fact}

\begin{fact}[Jensen's inequality]
For \(a_1,\ldots,a_N\in\mathbb R\) and nonnegative
\(b_1,\ldots,b_N\) satisfying \(\sum_{i=1}^Nb_i=1\), \(\left(\sum_{i=1}^Nb_i a_i\right)^2
\leq\sum_{i=1}^Nb_i a_i^2\).
\end{fact}

\begin{fact}[Weighted Cauchy--Schwarz inequality]
For \(a_1,\ldots,a_N\in\mathbb R\) and nonnegative
\(w_1,\ldots,w_N\),
\(\left(\sum_{i=1}^Nw_i|a_i|\right)^2
\leq
\left(\sum_{i=1}^Nw_i\right)
\left(\sum_{i=1}^Nw_i a_i^2\right)\).
\end{fact}
We call \(m\) a median of \(z_1,\ldots,z_N\) if at least \(N/2\) of the
values are at most \(m\) and at least \(N/2\) are at least \(m\). This
convention applies for both odd and even \(N\).
\begin{fact}
\label{fact:median-square}
Let \(z_1,\ldots,z_N,z\in\mathbb R\), and let \(m\) be a median of
\(z_1,\ldots,z_N\). Then
 \((m-z)^2\leq\frac2N\sum_{i=1}^N(z_i-z)^2\).

\end{fact}

\begin{proof}
By the definition of a median,
\(\left|\{j:z_j\geq m\}\right|\geq \frac{N}{2}\) and
\(\left|\{j:z_j\leq m\}\right|\geq \frac{N}{2}\).
Suppose first that \(z\leq m\). For every \(j\) satisfying \(z_j\geq m\),
\(|z_j-z|\geq m-z=|m-z|\). Thus,
\ifarxiv
\begin{equation*}
\sum_{j=1}^N (z_j-z)^2
\geq
\sum_{\{j:z_j\geq m\}}(z_j-z)^2
\geq
\frac{N}{2}(m-z)^2
\end{equation*}
\else
\begin{align*}
\sum_{j=1}^N (z_j-z)^2
&\geq
\sum_{\{j:z_j\geq m\}}(z_j-z)^2
\geq
\frac{N}{2}(m-z)^2
\end{align*}
\fi
Rearranging gives
\((m-z)^2\leq\frac{2}{N}\sum_{j=1}^N(z_j-z)^2\).
If \(z\geq m\), the same argument applies to the indices satisfying
\(z_j\leq m\), completing the proof.
\end{proof}

\begin{fact}
\label{fact:median-under-contamination}
Let \(z_1,\ldots,z_N,z\in\mathbb R\), let \(m\) be a median of
\(z_1,\ldots,z_N\), and suppose that \(\mathcal J\subseteq[N]\) satisfies
\(|\mathcal J|>9N/16\). Then
\((m-z)^2
\leq\frac{16}{N}\sum_{j\in\mathcal J}(z_j-z)^2\).
\end{fact}
\begin{proof}
Assume first that \(z\leq m\). Since \(m\) is a median, at least \(N/2\)
indices satisfy \(z_j\geq m\). Because
\(|\mathcal J|>9N/16\), more than
\(\frac{N}{2}+\frac{9N}{16}-N=\frac{N}{16}\)
of these indices belong to \(\mathcal J\). For every such index,
\((z_j-z)^2\geq(m-z)^2\). Hence
\(\sum_{j\in\mathcal J}(z_j-z)^2
\geq\frac{N}{16}(m-z)^2\).
Rearranging proves the claim. If \(z\geq m\), the same argument applies to
the indices satisfying \(z_j\leq m\).
\end{proof}

\subsection{Correctness}
\label{subsec:core-correctness}

We first fix the random variables used by completion-probability queries. For
every \(j\in[n_u]\), \(\ell\in[n-1]\), and
\(x_{1:\ell}\in\Sigma^\ell\), let
\(\mathcal G^j(x_{1:\ell})\) be the collection of mutually independent
uniform random variables used by all calls to \(\mathsf{reduce}\) when
Algorithm~\ref{alg:hmm-weighted-prefix-estimate} evaluates the prefix
\(x_{1:\ell}\) in repetition \(j\).  Write \(\mathcal G^j
=
\bigcup_{\ell=1}^{n-1}
\bigcup_{x_{1:\ell}\in\Sigma^\ell}
\mathcal G^j(x_{1:\ell})\).
We use the following zero-likelihood convention. For
\(\ell\in[n-1]\), if \(P_{\text{hmm}}(x_{1:\ell})=0\), then
\(P_{\text{hmm}}(\alpha\mid x_{1:\ell})\),
\(\hat P_{\text{hmm}}^{\,j}(\alpha\mid x_{1:\ell})\), and
\(\hat P_{\text{hmm}}(\alpha\mid x_{1:\ell})\) are all defined to be zero. At \(\ell=n\),
for every \(j\in[n_u]\), we define
\(P_{\text{hmm}}(\alpha\mid x_{1:n})
=\hat P_{\text{hmm}}^{\,j}(\alpha\mid x_{1:n})
=\hat P_{\text{hmm}}(\alpha\mid x_{1:n})
=\mathbf{1}_{x_{1:n}\in L_n(A)}\).
Given a prefix \(x_{1:\ell-1}\), define the normalization constants
\(\Gamma(x_{1:\ell-1})=
\sum_{x_\ell\in\Sigma}
P_{\text{lm}}(x_\ell\mid x_{1:\ell-1})
P_{\text{hmm}}(\alpha\mid x_{1:\ell})\) and
\(\hat\Gamma(x_{1:\ell-1})=
\sum_{x_\ell\in\Sigma}
P_{\text{lm}}(x_\ell\mid x_{1:\ell-1})
\hat P_{\text{hmm}}(\alpha\mid x_{1:\ell})\).
When the corresponding normalizing constant is positive, define
\(P(x_\ell\mid x_{1:\ell-1},\alpha)=
\frac{
P_{\text{lm}}(x_\ell\mid x_{1:\ell-1})
P_{\text{hmm}}(\alpha\mid x_{1:\ell})
}{
\Gamma(x_{1:\ell-1})
}\) and
\(\hat P(x_\ell\mid x_{1:\ell-1},\alpha)=
\frac{
P_{\text{lm}}(x_\ell\mid x_{1:\ell-1})
\hat P_{\text{hmm}}(\alpha\mid x_{1:\ell})
}{
\hat\Gamma(x_{1:\ell-1})
}\).
If a normalizing constant is zero, define the corresponding conditional
distribution according to an arbitrary fixed rule. Finally, let
\(P(x_{1:n}\mid\alpha)=\prod_{\ell=1}^n
P(x_\ell\mid x_{1:\ell-1},\alpha)\) and
\(\hat P(x_{1:n}\mid\alpha)=\prod_{\ell=1}^n
\hat P(x_\ell\mid x_{1:\ell-1},\alpha)\).
For \(1\leq\ell\leq n\), we use
\(P(x_{1:\ell}\mid\alpha)\) for the corresponding exact prefix marginal. We assume that, for every
\(\ell\in[n]\) and every prefix satisfying
\(P(x_{1:\ell-1}\mid\alpha)>0\), \(\Gamma(x_{1:\ell-1})>0\).

Let \(\mathcal V
=\{(q,b):q\in Q^k,\ b\in H_k,\ 0\leq k<n,\ Z_+(q,b)\neq\emptyset\}\). We call a
product state \(v=(q,b)\) productive when \(Z_+(q,b)\neq\emptyset\), and use
the shorthand \(Z_+(v)=Z_+(q,b)\), \(W(v)=W(q,b)\), \(p^j(v)=p^j(q,b)\),
\(\bar p^j(v)=\bar p^j(q,b)\), \(\rho^j(v)=\rho^j(q,b)\),
\(S^{r,j}(v)=S^{r,j}(q,b)\), \(\hat S^{r,j}(v)=\hat S^{r,j}(q,b)\),
\(M^{t,j}(v)=M^{t,j}(q,b)\), and \(\hat W^j(v)=\hat W^j(q,b)\).

For \(j\in[n_u]\), let \(\mathcal N^{\theta,j}\) be the \(j\)-th of the
\(n_u\) outer repetitions of \(\mathsf{precompute}\), including the test
against \(\theta\), and let
\(\mathcal N^j\) be the same repetition without this test. Thus
\(\mathcal N^j\) always completes. In \(\mathcal N^j\), define
\ifarxiv
\begin{equation*}
\begin{aligned}
&\mathcal A^j
=\bigcup_{(q,b)\in\mathcal V}
\left\{p^j(q,b)\notin(1\pm\kappa)W(q,b)^{-1}\right\}\\
&\mathcal B^j
=\left\{
\sum_{r=1}^{n_sn_t}\sum_{k=0}^{n}
\sum_{q\in Q^k}\sum_{b\in H_k}|S^{r,j}(q,b)|\geq\theta
\right\}\\
&\mathcal C^j
=(\mathcal A^j)^c\cap(\mathcal B^j)^c
\end{aligned}
\end{equation*}
\else
\begin{align*}
\mathcal A^j
&=\bigcup_{(q,b)\in\mathcal V}
\left\{p^j(q,b)\notin(1\pm\kappa)W(q,b)^{-1}\right\}\\
\mathcal B^j
&=\left\{
\sum_{r=1}^{n_sn_t}\sum_{k=0}^{n}
\sum_{q\in Q^k}\sum_{b\in H_k}|S^{r,j}(q,b)|\geq\theta
\right\}\\
\mathcal C^j
&=(\mathcal A^j)^c\cap(\mathcal B^j)^c
\end{align*}
\fi
These represent, respectively, approximation failure, sample-overflow
failure, and core success (i.e., neither approximation failure nor sample
overflow occurs).

By autoregressive factorization and the standing assumption, every
\(x_{1:\ell}\) with \(P(x_{1:\ell}\mid\alpha)>0\), \(\ell\in[n]\),
satisfies \(P_{\text{hmm}}(\alpha\mid x_{1:\ell})>0\). Hence the following
relative errors are well defined.

For \(\ell\in[n]\) and every prefix satisfying
\(P(x_{1:\ell}\mid\alpha)>0\), define the repetition-level relative errors
for \(j\in[n_u]\) and the final relative error, respectively, by
\(\xi^j(x_{1:\ell})
=
\frac{
\hat P_{\text{hmm}}^{\,j}(\alpha\mid x_{1:\ell})
}{
P_{\text{hmm}}(\alpha\mid x_{1:\ell})
}
-1\) and
\(\xi(x_{1:\ell})
=
\frac{
\hat P_{\text{hmm}}(\alpha\mid x_{1:\ell})
}{
P_{\text{hmm}}(\alpha\mid x_{1:\ell})
}
-1\).
Because \(\hat P_{\text{hmm}}\) is the median of the repetition-level estimates and
the denominator is positive and independent of \(j\),
\(\xi(x_{1:\ell})
=
\median_{j\in[n_u]}\xi^j(x_{1:\ell})\).

Define the corresponding integrated squared errors by
\ifarxiv
\begin{equation*}
\Delta^j
=
\sum_{\ell=1}^{n}
\sum_{\substack{x_{1:\ell}\in\Sigma^\ell\\
P(x_{1:\ell}\mid\alpha)>0}}
P(x_{1:\ell}\mid\alpha)
\left(\xi^j(x_{1:\ell})\right)^2\qquad
\Delta
=
\sum_{\ell=1}^{n}
\sum_{\substack{x_{1:\ell}\in\Sigma^\ell\\
P(x_{1:\ell}\mid\alpha)>0}}
P(x_{1:\ell}\mid\alpha)
\left(\xi(x_{1:\ell})\right)^2
\end{equation*}
\else
\begin{align*}
\Delta^j
&=
\sum_{\ell=1}^{n}
\sum_{\substack{x_{1:\ell}\in\Sigma^\ell\\
P(x_{1:\ell}\mid\alpha)>0}}
P(x_{1:\ell}\mid\alpha)
\left(\xi^j(x_{1:\ell})\right)^2\\
\Delta
&=
\sum_{\ell=1}^{n}
\sum_{\substack{x_{1:\ell}\in\Sigma^\ell\\
P(x_{1:\ell}\mid\alpha)>0}}
P(x_{1:\ell}\mid\alpha)
\left(\xi(x_{1:\ell})\right)^2
\end{align*}
\fi
and let
\(\mathcal I^j
=
\mathcal C^j
\cap
\left\{\Delta^j\leq2n\kappa^2\right\}\).
For every \(x_{1:n}\) in these sums, the positivity statement above and the
terminal definition give
\(\xi^j(x_{1:n})=\xi(x_{1:n})=0\). Thus, the \(\ell=n\) terms in both
integrated errors are zero. In particular, for \(n=1\), both integrated
errors equal zero.

We first state the estimates used in the correctness proof. Their proofs are
given in the remainder of this section.

\stateestimates*

\sampleoverflow*

\integratedonecore*
For \(j\in[n_u]\) and \(\ell\in[n]\), write \(E_\ell^j\) for the event
\[
E_\ell^j=
\left\{
\hat P_{\mathrm{hmm}}^{\,j}(\alpha\mid x_{1:\ell})
\notin (1\pm\kappa)^{-1}
P_{\mathrm{hmm}}(\alpha\mid x_{1:\ell})
\right\}
\]
We then have the following lemma.
\begin{lemma}
\label{lem:fixed-sequence-one-core}
Fix \(x_{1:n}\in\Sigma^n\) satisfying
\(P_{\mathrm{hmm}}(x_{1:n})>0\). For every \(j\in[n_u]\), the following
bound holds.
\[
\mathbb P\!\left[
\mathcal C^j\cap\bigcup_{\ell=1}^{n-1}E_\ell^j
\right]\leq\frac1{16}
\]
\end{lemma}

\prefixaccuracy*
\begin{proof}
When \(n=1\), the result follows from the exact terminal clause of
Algorithm~\ref{alg:hmm-weighted-prefix-estimate}. Assume \(n\geq2\).
For every \(j\in[n_u]\), Lemmas~\ref{lem:sample-overflow} and
\ref{lem:fixed-sequence-one-core} give
\ifarxiv
\begin{equation*}
\mathbb P\!\left[\bigcup_{\ell=1}^{n-1}E_\ell^j\right]
\leq
\mathbb P[(\mathcal C^j)^c]+\frac1{16}
\leq\frac3{16}
\end{equation*}
\else
\begin{align*}
\mathbb P\!\left[\bigcup_{\ell=1}^{n-1}E_\ell^j\right]
&\leq
\mathbb P[(\mathcal C^j)^c]+\frac1{16}
\leq\frac3{16}
\end{align*}
\fi
Let \(X_j=\bigcup_{\ell=1}^{n}E_\ell^j\).
These events are independent across \(j\). Hence Hoeffding's inequality and
the definition of \(n_u\) give the following bound.
\[
\mathbb P\!\left[
\sum_{j=1}^{n_u}
\mathbf{1}_{X_j}
\geq\frac{n_u}{2}
\right]
\leq e^{-n_u/8}
\leq\delta
\]
Outside this event, for every \(\ell\in[n]\), more than half of the
repetition-level estimates lie in
\((1\pm\kappa)^{-1}
P_{\mathrm{hmm}}(\alpha\mid x_{1:\ell})\). Their median therefore lies in
the same interval. Since \(\kappa=\varepsilon/(6+\varepsilon)\),
\ifarxiv
\begin{align}
\frac{1}{1+\kappa}
=1-\frac{\varepsilon}{6+2\varepsilon}
\geq 1-\frac{\varepsilon}{6}\geq 1-\varepsilon\\
\frac{1}{1-\kappa}
=1+\frac{\varepsilon}{6}\leq 1+\varepsilon
\end{align}
\else
\begin{align*}
\frac{1}{1+\kappa}
&=1-\frac{\varepsilon}{6+2\varepsilon}
\geq 1-\frac{\varepsilon}{6}\geq 1-\varepsilon\\
\frac{1}{1-\kappa}
&=1+\frac{\varepsilon}{6}\leq 1+\varepsilon
\end{align*}
\fi
Consequently, the following inclusion holds.
\[
(1\pm\kappa)^{-1}
P_{\mathrm{hmm}}(\alpha\mid x_{1:\ell})
\subseteq
(1\pm\varepsilon)
P_{\mathrm{hmm}}(\alpha\mid x_{1:\ell})
\]
This proves the result.
\end{proof}

We next give two lemmas that convert completion-probability
error into total variation distance.

\begin{lemma}
\label{lem:normalized-reweighting}
Fix \(\ell\in[n]\) and a prefix \(x_{1:\ell-1}\) satisfying
\(P(x_{1:\ell-1}\mid\alpha)>0\). Then 

\begin{align*}
D_{\mathrm{TV}}\!\left(
\hat P(\cdot\mid x_{1:\ell-1},\alpha),
P(\cdot\mid x_{1:\ell-1},\alpha)\right)\leq \sum_{\substack{x_\ell\in\Sigma\\
P(x_\ell\mid x_{1:\ell-1},\alpha)>0}}
P(x_\ell\mid x_{1:\ell-1},\alpha)\cdot\left|
\frac{\hat P_{\text{hmm}}(\alpha\mid x_{1:\ell})}
{P_{\text{hmm}}(\alpha\mid x_{1:\ell})}-1
\right|\\
\end{align*}

\end{lemma}

\begin{proof}
We first show that \(P_{\text{hmm}}(\alpha\mid x_{1:\ell})=0\) implies \(\hat P_{\text{hmm}}(\alpha\mid x_{1:\ell})=0\). If \(\ell<n\) and \(P_{\text{hmm}}(x_{1:\ell})=0\), the implication follows from the zero-likelihood convention. If \(\ell=n\), it follows from the terminal definition. It remains to consider \(\ell<n\) and \(P_{\text{hmm}}(x_{1:\ell})>0\). Assume \(P_{\text{hmm}}(\alpha\mid x_{1:\ell})=0\) and write \(R=R(x_{1:\ell})\). If \(R=\emptyset\), Algorithm~\ref{alg:hmm-weighted-prefix-estimate} returns zero. Otherwise, since
\(
P_{\text{hmm}}(\alpha\mid x_{1:\ell})
=\sum_{b\in H_\ell}
P_{\text{hmm}}(y_\ell=b\mid x_{1:\ell})W_R(b)
\),
nonnegativity implies that
\(W_R(b)=0\) for every \(b\in H_\ell\) satisfying
\(P_{\text{hmm}}(y_\ell=b\mid x_{1:\ell})>0\). For every such \(b\),
\(R_b^+=\emptyset\), because any \(q\in R_b^+\) would have a positive-weight completion. Thus,
each nonfailed repetition of Algorithm~\ref{alg:hmm-weighted-prefix-estimate} sets
\(\hat W_R^j(b)=0\), while each failed repetition is already assigned a zero estimate. Every \(b\in H_\ell\) satisfying
\(P_{\text{hmm}}(y_\ell=b\mid x_{1:\ell})=0\) contributes zero to
\(\hat P_{\text{hmm}}^{\,j}(\alpha\mid x_{1:\ell})\). Hence every
repetition-level estimate is zero, as is the median. Consequently, for every
\(x_\ell\in\Sigma\) satisfying
\(P(x_\ell\mid x_{1:\ell-1},\alpha)=0\),
\(P_{\text{lm}}(x_\ell\mid x_{1:\ell-1})
\hat P_{\text{hmm}}(\alpha\mid x_{1:\ell})=0\). Suppose first that
\(\hat\Gamma(x_{1:\ell-1})>0\). For every
\(x_\ell\in\Sigma\) satisfying
\(P(x_\ell\mid x_{1:\ell-1},\alpha)>0\), the definitions give
\ifarxiv
\begin{equation*}
\frac{\hat P(x_\ell\mid x_{1:\ell-1},\alpha)}{P(x_\ell\mid x_{1:\ell-1},\alpha)}=\frac{\Gamma(x_{1:\ell-1})}{\hat\Gamma(x_{1:\ell-1})}\frac{\hat P_{\text{hmm}}(\alpha\mid x_{1:\ell})}
{P_{\text{hmm}}(\alpha\mid x_{1:\ell})}
\end{equation*}
\else
\begin{align*}
\frac{\hat P(x_\ell\mid x_{1:\ell-1},\alpha)}{P(x_\ell\mid x_{1:\ell-1},\alpha)}=\frac{\Gamma(x_{1:\ell-1})}{\hat\Gamma(x_{1:\ell-1})}\frac{\hat P_{\text{hmm}}(\alpha\mid x_{1:\ell})}
{P_{\text{hmm}}(\alpha\mid x_{1:\ell})}
\end{align*}
\fi

Moreover,
\ifarxiv
\begin{equation*}
\sum_{\substack{x_\ell\in\Sigma\\
P(x_\ell\mid x_{1:\ell-1},\alpha)>0}}
P(x_\ell\mid x_{1:\ell-1},\alpha)\frac{\hat P_{\text{hmm}}(\alpha\mid x_{1:\ell})}
{P_{\text{hmm}}(\alpha\mid x_{1:\ell})}
=\frac{\hat\Gamma(x_{1:\ell-1})}{\Gamma(x_{1:\ell-1})}
\end{equation*}
\else
\begin{align*}
&\sum_{\substack{x_\ell\in\Sigma\\
P(x_\ell\mid x_{1:\ell-1},\alpha)>0}}
P(x_\ell\mid x_{1:\ell-1},\alpha)\frac{\hat P_{\text{hmm}}(\alpha\mid x_{1:\ell})}
{P_{\text{hmm}}(\alpha\mid x_{1:\ell})}\\
&=\frac{\hat\Gamma(x_{1:\ell-1})}{\Gamma(x_{1:\ell-1})}
\end{align*}
\fi
Applying the triangle inequality through these unnormalized weights gives
\ifarxiv
\begin{equation*}
\begin{aligned}
&2D_{\mathrm{TV}}\!\left(
\hat P(\cdot\mid x_{1:\ell-1},\alpha),
P(\cdot\mid x_{1:\ell-1},\alpha)
\right)\\
 \leq &\left|1-\frac{\hat\Gamma(x_{1:\ell-1})}
{\Gamma(x_{1:\ell-1})}\right|+\sum_{\substack{x_\ell\in\Sigma\\
P(x_\ell\mid x_{1:\ell-1},\alpha)>0}}\bigl(
P(x_\ell\mid x_{1:\ell-1},\alpha)\left|
\frac{\hat P_{\text{hmm}}(\alpha\mid x_{1:\ell})}
{P_{\text{hmm}}(\alpha\mid x_{1:\ell})}-1
\right|\bigr)\\
\leq &2\sum_{\substack{x_\ell\in\Sigma\\
P(x_\ell\mid x_{1:\ell-1},\alpha)>0}}\bigl(P(x_\ell\mid x_{1:\ell-1},\alpha)\left|
\frac{\hat P_{\text{hmm}}(\alpha\mid x_{1:\ell})}
{P_{\text{hmm}}(\alpha\mid x_{1:\ell})}-1
\right|\bigr)
\end{aligned}
\end{equation*}
\else
\begin{align*}
&2D_{\mathrm{TV}}\!\left(
\hat P(\cdot\mid x_{1:\ell-1},\alpha),
P(\cdot\mid x_{1:\ell-1},\alpha)
\right) \\
&\leq \left|1-\frac{\hat\Gamma(x_{1:\ell-1})}
{\Gamma(x_{1:\ell-1})}\right|+\sum_{\substack{x_\ell\in\Sigma\\
P(x_\ell\mid x_{1:\ell-1},\alpha)>0}}\bigl(\\
&\quad
P(x_\ell\mid x_{1:\ell-1},\alpha)\left|
\frac{\hat P_{\text{hmm}}(\alpha\mid x_{1:\ell})}
{P_{\text{hmm}}(\alpha\mid x_{1:\ell})}-1
\right|\bigr)\\
&\leq 2\sum_{\substack{x_\ell\in\Sigma\\
P(x_\ell\mid x_{1:\ell-1},\alpha)>0}}\bigl(\\
&P(x_\ell\mid x_{1:\ell-1},\alpha)\left|
\frac{\hat P_{\text{hmm}}(\alpha\mid x_{1:\ell})}
{P_{\text{hmm}}(\alpha\mid x_{1:\ell})}-1
\right|\bigr)
\end{align*}
\fi

If \(\hat\Gamma(x_{1:\ell-1})=0\), nonnegativity gives
\(\hat P_{\text{hmm}}(\alpha\mid x_{1:\ell})=0\) for every
\(x_\ell\in\Sigma\) satisfying
\(P(x_\ell\mid x_{1:\ell-1},\alpha)>0\), and the claim trivially holds.
\end{proof}

\begin{lemma}[Sequential total variation]
\label{lem:sequential-tv}
For any two autoregressive distributions \(P(\cdot\mid\alpha)\) and
\(\hat P(\cdot\mid\alpha)\) on \(\Sigma^n\),
\ifarxiv
\begin{equation*}
\begin{aligned}
&D_{\mathrm{TV}}\!\left(\hat P(\cdot\mid\alpha),P(\cdot\mid\alpha)\right)
\leq \sum_{\ell=1}^n
\sum_{x_{1:\ell-1}\in\Sigma^{\ell-1}}
P(x_{1:\ell-1}\mid\alpha)D_{\mathrm{TV}}\!\left(\hat P(\cdot\mid x_{1:\ell-1},\alpha),
P(\cdot\mid x_{1:\ell-1},\alpha)
\right)
\end{aligned}
\end{equation*}
\else
\begin{align*}
&D_{\mathrm{TV}}\!\left(\hat P(\cdot\mid\alpha),P(\cdot\mid\alpha)\right)\\
&\quad\leq \sum_{\ell=1}^n
\sum_{x_{1:\ell-1}\in\Sigma^{\ell-1}}
P(x_{1:\ell-1}\mid\alpha)\\
&\qquad\qquad\cdot D_{\mathrm{TV}}\!\left(\hat P(\cdot\mid x_{1:\ell-1},\alpha),
P(\cdot\mid x_{1:\ell-1},\alpha)
\right)
\end{align*}
\fi

\end{lemma}

\begin{proof}
For \(0\leq k\leq n\), define the interpolating distribution between
\(\hat P(\cdot\mid\alpha)\) and \(P(\cdot\mid\alpha)\) by
\ifarxiv
\begin{equation*}
P^{[k]}(x_{1:n}\mid\alpha)
=\prod_{\ell=1}^{k}
P(x_\ell\mid x_{1:\ell-1},\alpha) \cdot\prod_{\ell=k+1}^{n}
\hat P(x_\ell\mid x_{1:\ell-1},\alpha)
\end{equation*}
\else
\begin{align*}
& P^{[k]}(x_{1:n}\mid\alpha)\\
&=\prod_{\ell=1}^{k}
P(x_\ell\mid x_{1:\ell-1},\alpha) \cdot\prod_{\ell=k+1}^{n}
\hat P(x_\ell\mid x_{1:\ell-1},\alpha)
\end{align*}
\fi

Thus
\(P^{[0]}=\hat P(\cdot\mid\alpha)\) and
\(P^{[n]}=P(\cdot\mid\alpha)\). When comparing
\(P^{[\ell-1]}\) and \(P^{[\ell]}\), the prefix \(x_{1:\ell-1}\) has exact
marginal \(P(x_{1:\ell-1}\mid\alpha)\), and the common approximate
continuation after layer \(\ell\) sums to one. Hence
\ifarxiv
\begin{equation*}
D_{\mathrm{TV}}(P^{[\ell-1]},P^{[\ell]})
= \sum_{x_{1:\ell-1}\in\Sigma^{\ell-1}}
P(x_{1:\ell-1}\mid\alpha)\cdot
D_{\mathrm{TV}}\!\left(\hat P(\cdot\mid x_{1:\ell-1},\alpha),P(\cdot\mid x_{1:\ell-1},\alpha) \right)
\end{equation*}
\else
\begin{align*}
    & D_{\mathrm{TV}}(P^{[\ell-1]},P^{[\ell]}) \\
    &= \sum_{x_{1:\ell-1}\in\Sigma^{\ell-1}}
P(x_{1:\ell-1}\mid\alpha)\cdot \\ & D_{\mathrm{TV}}\!\left(\hat P(\cdot\mid x_{1:\ell-1},\alpha),P(\cdot\mid x_{1:\ell-1},\alpha) \right) \\
\end{align*}
\fi

The triangle inequality for total variation distance yields the claimed
inequality.
\end{proof}

\distributionpreservation*

\begin{proof}
When \(n=1\), the completion probability used for every candidate token is
the exact terminal membership value, so \(\hat P=P\) and their total
variation distance is zero. Assume
\(n\geq2\). For \(j\in[n_u]\), let \(Y_j=\mathbf{1}_{(\mathcal I^j)^c}\).
Lemma~\ref{lem:integrated-one-core} gives
\(\mathbb E[Y_j]\leq3/16\). These variables are independent because each
\(\mathcal I^j\) depends only on outer repetition \(j\) and
\(\mathcal G^j\). Hoeffding's inequality gives
\ifarxiv
\begin{equation*}
\mathbb{P}\!\left[
\sum_{j=1}^{n_u}Y_j\geq\frac{7n_u}{16}
\right]
\leq
\mathbb{P}\!\left[
\sum_{j=1}^{n_u}Y_j
-\mathbb E\!\left[\sum_{i=1}^{n_u}Y_i\right]
\geq\frac{n_u}{4}
\right]
\leq e^{-n_u/8}\leq\delta
\end{equation*}
\else
\begin{align*}
\mathbb{P}\!\left[
\sum_{j=1}^{n_u}Y_j\geq\frac{7n_u}{16}
\right]
&\leq
\mathbb{P}\!\left[
\sum_{j=1}^{n_u}Y_j
-\mathbb E\!\left[\sum_{i=1}^{n_u}Y_i\right]
\geq\frac{n_u}{4}
\right]\\
&\leq e^{-n_u/8}\leq\delta
\end{align*}
\fi
Outside this event, let
\(\mathcal J_{\mathrm{good}}
=\{j\in[n_u]:\mathcal I^j\text{ occurs}\}\).
Then \(|\mathcal J_{\mathrm{good}}|>9n_u/16\). For every
\(\ell\in[n]\) and every prefix satisfying
\(P(x_{1:\ell}\mid\alpha)>0\), the median identity above and
Fact~\ref{fact:median-under-contamination}, applied with
\(z_j=\xi^j(x_{1:\ell})\) and \(z=0\), give
\(\left(\xi(x_{1:\ell})\right)^2
\leq
\frac{16}{n_u}
\sum_{j\in\mathcal J_{\mathrm{good}}}
\left(\xi^j(x_{1:\ell})\right)^2\).
For every \(j\in\mathcal J_{\mathrm{good}}\), the definition of
\(\mathcal I^j\) gives
\(\Delta^j\leq2n\kappa^2\). Summing against the exact prefix marginals
therefore gives
\ifarxiv
\begin{equation}
\Delta
\leq
\frac{16}{n_u}
\sum_{j\in\mathcal J_{\mathrm{good}}}\Delta^j
\leq32n\kappa^2
\label{eq:amplified-integrated-error}
\end{equation}
\else
\begin{align}
\Delta
&\leq
\frac{16}{n_u}
\sum_{j\in\mathcal J_{\mathrm{good}}}\Delta^j
\leq32n\kappa^2
\label{eq:amplified-integrated-error}
\end{align}
\fi

The conditional distributions at layer \(n\) agree because the terminal
completion probability is exact. Lemmas~\ref{lem:normalized-reweighting}
and~\ref{lem:sequential-tv} give
\ifarxiv
\begin{equation*}
D_{\mathrm{TV}}\!\left(
\hat P(\cdot\mid\alpha),P(\cdot\mid\alpha)
\right)
\leq
\sum_{\ell=1}^{n}
\sum_{\substack{x_{1:\ell}\in\Sigma^\ell\\
P(x_{1:\ell}\mid\alpha)>0}}
P(x_{1:\ell}\mid\alpha)
\cdot
|\xi(x_{1:\ell})|
\end{equation*}
\else
\begin{align*}
&D_{\mathrm{TV}}\!\left(
\hat P(\cdot\mid\alpha),P(\cdot\mid\alpha)
\right)
\leq\\
&
\sum_{\ell=1}^{n}
\sum_{\substack{x_{1:\ell}\in\Sigma^\ell\\
P(x_{1:\ell}\mid\alpha)>0}}
P(x_{1:\ell}\mid\alpha)
\cdot
|\xi(x_{1:\ell})|
\end{align*}
\fi
The prefix marginals satisfy
\ifarxiv
\begin{equation*}
\sum_{\ell=1}^{n}
\sum_{x_{1:\ell}\in\Sigma^\ell}
P(x_{1:\ell}\mid\alpha)
=n
\end{equation*}
\else
\begin{align*}
\sum_{\ell=1}^{n}
&
\sum_{x_{1:\ell}\in\Sigma^\ell}
P(x_{1:\ell}\mid\alpha)
=n
\end{align*}
\fi
By the weighted Cauchy--Schwarz inequality,
\ifarxiv
\begin{equation*}
\left(D_{\mathrm{TV}}\!\left(
\hat P(\cdot\mid\alpha),P(\cdot\mid\alpha)
\right)\right)^2
\leq n\Delta
\leq32n^2\kappa^2
\end{equation*}
\else
\begin{align*}
\left(D_{\mathrm{TV}}\!\left(
\hat P(\cdot\mid\alpha),P(\cdot\mid\alpha)
\right)\right)^2
&\leq n\Delta\\
&\leq32n^2\kappa^2
\end{align*}
\fi
Taking square roots on both sides and substituting
\(\kappa=\varepsilon/(6+\varepsilon)\), we have
\ifarxiv
\begin{equation*}
D_{\mathrm{TV}}\!\left(\hat P(\cdot\mid\alpha),P(\cdot\mid\alpha)\right)
\leq 4\sqrt{2}n\kappa
<6n\frac{\varepsilon}{6+\varepsilon}
\leq n\varepsilon
\end{equation*}
\else
\begin{align*}
D_{\mathrm{TV}}\!\left(\hat P(\cdot\mid\alpha),P(\cdot\mid\alpha)\right)
&\leq 4\sqrt{2}n\kappa\\
&<6n\frac{\varepsilon}{6+\varepsilon}\\
&\leq n\varepsilon
\end{align*}
\fi
\end{proof}
It remains to prove Lemmas~\ref{lem:state-estimates},
\ref{lem:sample-overflow}, \ref{lem:integrated-one-core},
and~\ref{lem:fixed-sequence-one-core}.

\subsection{Oracle-Corrected Auxiliary Process}
\label{subsec:oracle-corrected-analysis}

We return to the fixed \(j\in[n_u]\) and simplify the analysis by considering
a modified algorithm. The algorithm
\(\mathcal N^{*,j}\) follows \(\mathcal N^j\), but applies an analysis-only
oracle correction. When a state \(v=(q,b)\in\mathcal V\) is
processed, let \(\bar p^j(q,b)\) be the value computed before any
correction. If
\(\bar p^j(q,b)\notin(1\pm\kappa)W(q,b)^{-1}\), the oracle correction sets
\(p^j(q,b)=(1-\kappa)W(q,b)^{-1}\). Otherwise it sets
\(p^j(q,b)=\bar p^j(q,b)\).

\begin{algorithm}[t]
\DontPrintSemicolon
\SetAlgoNlRelativeSize{-1}
\(\rho^j(q,b)\gets
\min_{(a,q_i,b')\in\mathcal E_\ell(q,b)}
\frac{p^j(q_i,b')}{\psi(a,b,b')}\)\;
\For{\(r\in[n_sn_t]\)}{
  \For{\((a,q_i,b')\in\mathcal E_\ell(q,b)\)}{
    \(\bar S^{r,j}(a,q,q_i,b')\gets
    \mathsf{reduce}\left(
    (a,b')\cdot S^{r,j}(q_i,b'),
    \frac{\rho^j(q,b)\psi(a,b,b')}{p^j(q_i,b')}
    \right)\)\;
  }
  \(\hat S^{r,j}(q,b)\gets
  \bigcup_{(a,q_i,b')\in\mathcal E_\ell(q,b)}
  \{(a,b')\cdot(x',y')\in\bar S^{r,j}(a,q,q_i,b'):
  \mathsf{first}(q,a,x')=q_i\}\)\;
}
\For{\(t\in[n_t]\)}{
  \(M^{t,j}(q,b)\gets
  (n_s\rho^j(q,b))^{-1}
  \sum_{r=n_s(t-1)+1}^{n_st}|\hat S^{r,j}(q,b)|\)\;
}
\(\hat W^j(q,b)\gets\median_{t\in[n_t]}M^{t,j}(q,b)\)\;
\(\bar p^j(q,b)\gets\min\{\rho^j(q,b),\hat W^j(q,b)^{-1}\}\)\;
\AlgoHighlightStart{oracle-correction}%
\eIf{\(\bar p^j(q,b)\notin(1\pm\kappa)W(q,b)^{-1}\)}{
  \(p^j(q,b)\gets(1-\kappa)W(q,b)^{-1}\)%
  \AlgoHighlightEnd{oracle-correction}\tcp*{Oracle correction}
}{
  \(p^j(q,b)\gets\bar p^j(q,b)\)\;
}
\For{\(r\in[n_sn_t]\)}{
  \(S^{r,j}(q,b)\gets
  \mathsf{reduce}\left(
  \hat S^{r,j}(q,b),
  \frac{p^j(q,b)}{\rho^j(q,b)}
  \right)\)\;
}
\caption{\(\mathsf{estimateAndSample}(q,j,\ell)\) (oracle-corrected \(\mathcal N^{*,j}\))}
\label{alg:oracle-corrected-weighted}
\end{algorithm}

The oracle correction guarantees
\(p^j(q,b)\in(1\pm\kappa)W(q,b)^{-1}\) at every productive product state. We
first check that the correction does not make a reduction probability
larger than one. If \((a,q',b')\in\mathcal E_\ell(q,b)\), prefixing every
positive atom at \((q',b')\) by \((a,b')\) gives distinct positive atoms at
\((q,b)\). Therefore,
\(W(q,b)\geq\psi(a,b,b')W(q',b')\).
All productive children have already been processed. Hence
\ifarxiv
\begin{equation*}
\rho^j(q,b)
=\min_{(a,q',b')\in\mathcal E_\ell(q,b)}
\frac{p^j(q',b')}{\psi(a,b,b')}
\geq\min_{(a,q',b')\in\mathcal E_\ell(q,b)}
\frac{1-\kappa}{\psi(a,b,b')W(q',b')}
\geq\frac{1-\kappa}{W(q,b)}
\end{equation*}
\else
\begin{align*}
\rho^j(q,b)
&=\min_{(a,q',b')\in\mathcal E_\ell(q,b)}
\frac{p^j(q',b')}{\psi(a,b,b')}\\
&\geq\min_{(a,q',b')\in\mathcal E_\ell(q,b)}
\frac{1-\kappa}{\psi(a,b,b')W(q',b')}\\
&\geq\frac{1-\kappa}{W(q,b)}
\end{align*}
\fi
Thus the value set by the oracle correction is at most \(\rho^j(q,b)\). The edge
reductions are legal by the definition of \(\rho^j(q,b)\), and the final
reduction is legal because \(p^j(q,b)\leq\rho^j(q,b)\).

We now relate \(\mathcal N^j\) and \(\mathcal N^{*,j}\). Implement each call
to \(\mathsf{reduce}\) using an independent uniform random variable. Use the
same variables in the two executions until the oracle correction is first
applied. Consider the
states in \(\mathcal V\) in their processing order. If \(\mathcal A^j\)
occurs in \(\mathcal N^j\), there is an earliest state
\(v\) for which \(p^j(v)\) is outside \((1\pm\kappa)W(v)^{-1}\). A
\(v\)-prefix execution assigns all variables exposed through the computation
of the tentative value at \(v\). Before the first such state,
\(\mathcal N^j\) and
\(\mathcal N^{*,j}\) perform the same reductions with the same
probabilities. The identity on the corresponding uniform-variable outcomes is therefore a probability-
preserving bijection between the \(v\)-prefix executions where \(v\) is the
first state outside the target interval in \(\mathcal N^j\) and those where
the oracle correction is first applied at \(v\). Summing over the possible first such
states gives
\begin{equation}
\mathbb{P}_{\mathcal N^j}[\mathcal A^j]
\leq
\sum_{v\in\mathcal V}
\mathbb{P}_{\mathcal N^{*,j}}\!\left[
\bar p^j(v)\notin(1\pm\kappa)W(v)^{-1}
\right]
\label{eq:first-bad-union-bound}
\end{equation}
For a prefix \(x_{1:\ell}\), let
\(\hat P_{\text{hmm},*}^{\,j}(\alpha\mid x_{1:\ell})\) denote the estimate obtained
from the samples of \(\mathcal N^{*,j}\) using the variables in
\(\mathcal G^j(x_{1:\ell})\).

\begin{lemma}
\label{lem:core-success-coupling}
The executions \(\mathcal N^{\theta,j}\), \(\mathcal N^j\), and
\(\mathcal N^{*,j}\) can be coupled so that, on \(\mathcal C^j\), the
budgeted execution does not set \(\mathsf{fail}^j=1\), all three executions
have the same stored samples and rates, and, simultaneously for every
\(\ell\in[n-1]\) and \(x_{1:\ell}\in\Sigma^\ell\),
\(\hat P_{\text{hmm}}^{\,j}(\alpha\mid x_{1:\ell})
=\hat P_{\text{hmm},*}^{\,j}(\alpha\mid x_{1:\ell})\).
\end{lemma}

\begin{proof}
Use the same uniform random variable for every corresponding call to
\(\mathsf{reduce}\) in the three preprocessing executions. On
\((\mathcal A^j)^c\), every tentative value in \(\mathcal N^j\) belongs to
its target interval, so the oracle correction is never applied and
\(\mathcal N^j=\mathcal N^{*,j}\). The total number of stored samples is
nondecreasing. Hence, on \((\mathcal B^j)^c\), it never reaches \(\theta\),
so \(\mathcal N^{\theta,j}=\mathcal N^j\) and
\(\mathsf{fail}^j=0\). Using the same query variables in
\(\mathcal G^j\) then gives the final assertion.
\end{proof}

\subsubsection{Probabilistic Environment}
\label{sec:weighted-probabilistic-environment}

We now work with \(\mathcal N^{*,j}\). The transition lemmas below also
hold in \(\mathcal N^j\), since their proofs do not use the correction
rule. For \(q\in Q^k\), define the suffix depth of \(q\) and of every
product state \((q,b)\), to be \(n-k\). For \(0<i\leq n+1\), let
\(\mathcal F_i^j\) contain all variables belonging to product states of suffix
depth smaller than \(i\), including their final sample sets. For
\(0<i\leq n\), let \(\hat{\mathcal F}_i^j\) extend
\(\mathcal F_i^j\) with all
variables constructed at suffix depth \(i\) through the corrected sampling
probability \(p^j(q,b)\), but not the uniform random variables used by the final
calls to \(\mathsf{reduce}\) producing
\(S^{r,j}(q,b)\). We include the deterministic terminal initialization in every
\(\mathcal F_i^j\). The variables can be exposed so that
\(\mathcal F_i^j\subseteq\hat{\mathcal F}_i^j\subseteq\mathcal F_{i+1}^j\).

Indeed, every edge from a product state of suffix depth \(i\) goes to a state
of suffix depth \(i-1\). Operations at the same depth use only completed
child sets and mutually independent variables used by \(\mathsf{reduce}\), so
they may be
exposed together without changing their joint distribution.

In Section~\ref{subsec:precomputation}, we used the intuitive invariant
\(\mathbb{P}[z\in S^{r,j}(q,b)]=p^j(q,b)w(z,b)\) to explain the algorithm.
This statement is not rigorous enough because its left-hand side is a fixed
value, whereas its right-hand side is a random variable. The invariant can
be formalized as follows. For a productive product state \(v=(q,b)\), an
atom \(z\in Z_+(q,b)\), and
\(r\in[n_sn_t]\), define
\ifarxiv
\begin{equation*}
A^{r,j}(z,v)=\frac{\mathbf{1}_{z\in S^{r,j}(v)}}{p^j(v)w(z,b)}\qquad
\hat A^{r,j}(z,v)=\frac{\mathbf{1}_{z\in\hat S^{r,j}(v)}}{\rho^j(v)w(z,b)}
\end{equation*}
\else
\begin{align*}
A^{r,j}(z,v)&=\frac{\mathbf{1}_{z\in S^{r,j}(v)}}{p^j(v)w(z,b)}\\
\hat A^{r,j}(z,v)&=\frac{\mathbf{1}_{z\in\hat S^{r,j}(v)}}{\rho^j(v)w(z,b)}
\end{align*}
\fi
At the identified terminal state \(q_F^n\), both normalized variables are
defined to be one.
If \(v\) has suffix depth
\(i\), the exact final-reduction identity is
\(\mathbb E[\mathbf{1}_{z\in S^{r,j}(v)}\mid\hat{\mathcal F}_i^j]
=\mathbf{1}_{z\in\hat S^{r,j}(v)}p^j(v)/\rho^j(v)\).
After normalization, this gives the first identity in
Lemma~\ref{lem:weighted-one-atom-transfer}, while the second identity
transports the normalized variable to its canonical child. Iterating these
identities and applying the tower rule gives
\(\mathbb E[A^{r,j}(z,v)]=\mathbb E[\hat A^{r,j}(z,v)]=1\), as proved in
Lemma~\ref{lem:weighted-unit-expectation}.

\begin{lemma}
\label{lem:weighted-one-atom-transfer}
Let \(v=(q,b)\) have positive suffix depth and let
\((z',(q',b'))\) be the canonical child of \((z,v)\). If the suffix depth of
\(v\) is \(i\), then
\ifarxiv
\begin{equation*}
\mathbb E[A^{r,j}(z,v)\mid\hat{\mathcal F}_i^j]
=\hat A^{r,j}(z,v)\qquad
\mathbb E[\hat A^{r,j}(z,v)\mid\mathcal F_i^j]
=A^{r,j}(z',(q',b'))
\end{equation*}
\else
\begin{align*}
\mathbb E[A^{r,j}(z,v)\mid\hat{\mathcal F}_i^j]
&=\hat A^{r,j}(z,v)\\
\mathbb E[\hat A^{r,j}(z,v)\mid\mathcal F_i^j]
&=A^{r,j}(z',(q',b'))
\end{align*}
\fi
\end{lemma}

\begin{proof}
Write \(z=(a,b')\cdot z'\). Conditioned on \(\mathcal F_i^j\), the atom
\(z\) belongs to \(\hat S^{r,j}(v)\) if and only if \(z'\in S^{r,j}(q',b')\) and the
corresponding call to \(\mathsf{reduce}\) includes its canonical copy. Its
conditional inclusion probability is
\(\frac{\rho^j(v)\psi(a,b,b')}{p^j(q',b')}
\).

Since
\(w(z,b)=\psi(a,b,b')w(z',b')\), the fixed-variable rule gives
\ifarxiv
\begin{equation*}
\mathbb E[\hat A^{r,j}(z,v)\mid\mathcal F_i^j]
=\frac{\mathbf{1}_{z'\in S^{r,j}(q',b')}}
{\rho^j(v)w(z,b)}
\frac{\rho^j(v)\psi(a,b,b')}{p^j(q',b')}
=A^{r,j}(z',(q',b'))
\end{equation*}
\else
\begin{align*}
\mathbb E[\hat A^{r,j}(z,v)\mid\mathcal F_i^j]
&=\frac{\mathbf{1}_{z'\in S^{r,j}(q',b')}}
{\rho^j(v)w(z,b)}
\frac{\rho^j(v)\psi(a,b,b')}{p^j(q',b')}\\
&=A^{r,j}(z',(q',b'))
\end{align*}
\fi
Conditioned on \(\hat{\mathcal F}_i^j\), the only remaining randomness in
\(S^{r,j}(v)\) is its final reduction with probability \(p^j(v)/\rho^j(v)\). The
fixed-variable rule therefore gives
\ifarxiv
\begin{equation*}
\mathbb E[A^{r,j}(z,v)\mid\hat{\mathcal F}_i^j]
=\frac{\mathbf{1}_{z\in\hat S^{r,j}(v)}}{p^j(v)w(z,b)}
\frac{p^j(v)}{\rho^j(v)}
=\hat A^{r,j}(z,v)
\end{equation*}
\else
\begin{align*}
\mathbb E[A^{r,j}(z,v)\mid\hat{\mathcal F}_i^j]
&=\frac{\mathbf{1}_{z\in\hat S^{r,j}(v)}}{p^j(v)w(z,b)}
\frac{p^j(v)}{\rho^j(v)}\\
&=\hat A^{r,j}(z,v)
\end{align*}
\fi
\end{proof}

\begin{lemma}
\label{lem:weighted-two-atom-transfer}
Let \((z_1,v_1)\neq(z_2,v_2)\) be two atom--state pairs at the same positive
suffix depth \(i\), and let \((z_1',v_1')\) and \((z_2',v_2')\) be their
canonical children. Then
\ifarxiv
\begin{equation*}
\mathbb E[A^{r,j}(z_1,v_1)A^{r,j}(z_2,v_2)\mid\hat{\mathcal F}_i^j]=
\hat A^{r,j}(z_1,v_1)\hat A^{r,j}(z_2,v_2)
\end{equation*}
\else
\begin{align*}
&\mathbb E[A^{r,j}(z_1,v_1)A^{r,j}(z_2,v_2)\mid\hat{\mathcal F}_i^j]=\\
&
\hat A^{r,j}(z_1,v_1)\hat A^{r,j}(z_2,v_2)
\end{align*}
\fi
and 
\ifarxiv
\begin{equation*}
\mathbb E[\hat A^{r,j}(z_1,v_1)\hat A^{r,j}(z_2,v_2)\mid\mathcal F_i^j]=
A^{r,j}(z_1',v_1')A^{r,j}(z_2',v_2')
\end{equation*}
\else
\begin{align*}
& \mathbb E[\hat A^{r,j}(z_1,v_1)\hat A^{r,j}(z_2,v_2)\mid\mathcal F_i^j]=\\
& A^{r,j}(z_1',v_1')A^{r,j}(z_2',v_2')
\end{align*}
\fi

The second right-hand side is a square when the canonical children coincide.
The identities also hold for different indices \(r_1,r_2\) when the triples
\((r_1,z_1,v_1)\) and \((r_2,z_2,v_2)\) are distinct.
\end{lemma}

\begin{proof}
Write \(v_\nu=(q_\nu,b_\nu)\),
\(z_\nu=(a_\nu,b_\nu')\cdot z_\nu'\), and
\(\psi_\nu=\psi(a_\nu,b_\nu,b_\nu')\) for \(\nu\in\{1,2\}\). Distinct atom--state
pairs use distinct random variables in the final calls to
\(\mathsf{reduce}\).
Conditional independence gives
\ifarxiv
\begin{equation*}
\begin{aligned}
&\mathbb E[A^{r,j}(z_1,v_1)A^{r,j}(z_2,v_2)\mid\hat{\mathcal F}_i^j]
=
\frac{\mathbf{1}_{z_1\in\hat S^{r,j}(v_1)}}{p^j(v_1)w(z_1,b_1)}
\frac{p^j(v_1)}{\rho^j(v_1)}\cdot
\frac{\mathbf{1}_{z_2\in\hat S^{r,j}(v_2)}}{p^j(v_2)w(z_2,b_2)}
\frac{p^j(v_2)}{\rho^j(v_2)}\\
&=\hat A^{r,j}(z_1,v_1)\hat A^{r,j}(z_2,v_2)
\end{aligned}
\end{equation*}
\else
\begin{align*}
&\mathbb E[A^{r,j}(z_1,v_1)A^{r,j}(z_2,v_2)\mid\hat{\mathcal F}_i^j]\\
&=
\frac{\mathbf{1}_{z_1\in\hat S^{r,j}(v_1)}}{p^j(v_1)w(z_1,b_1)}
\frac{p^j(v_1)}{\rho^j(v_1)}\cdot
\frac{\mathbf{1}_{z_2\in\hat S^{r,j}(v_2)}}{p^j(v_2)w(z_2,b_2)}
\frac{p^j(v_2)}{\rho^j(v_2)}\\
&=\hat A^{r,j}(z_1,v_1)\hat A^{r,j}(z_2,v_2)
\end{align*}
\fi

For the other transition, canonical deduplication assigns the two parent
atoms to distinct calls to \(\mathsf{reduce}\). Conditioned on
\(\mathcal F_i^j\), their inclusion probabilities are
\(\rho^j(v_\nu)\psi_\nu/p^j(v_\nu')\). Thus

\ifarxiv
\begin{equation*}
\begin{aligned}
&\mathbb E[\hat A^{r,j}(z_1,v_1)\hat A^{r,j}(z_2,v_2)\mid\mathcal F_i^j]
= \frac{\mathbf{1}_{z_1'\in S^{r,j}(v_1')}}{\rho^j(v_1)w(z_1,b_1)}
\frac{\rho^j(v_1)\psi_1}{p^j(v_1')}\times
\frac{\mathbf{1}_{z_2'\in S^{r,j}(v_2')}}{\rho^j(v_2)w(z_2,b_2)}
\frac{\rho^j(v_2)\psi_2}{p^j(v_2')}\\
&= A^{r,j}(z_1',v_1')A^{r,j}(z_2',v_2')
\end{aligned}
\end{equation*}
\else
\begin{align*}
& \mathbb E[\hat A^{r,j}(z_1,v_1)\hat A^{r,j}(z_2,v_2)\mid\mathcal F_i^j]\\
&= \frac{\mathbf{1}_{z_1'\in S^{r,j}(v_1')}}{\rho^j(v_1)w(z_1,b_1)}
\frac{\rho^j(v_1)\psi_1}{p^j(v_1')}\times\\
&
\frac{\mathbf{1}_{z_2'\in S^{r,j}(v_2')}}{\rho^j(v_2)w(z_2,b_2)}
\frac{\rho^j(v_2)\psi_2}{p^j(v_2')}\\
&= A^{r,j}(z_1',v_1')A^{r,j}(z_2',v_2')
\end{align*}
\fi

Here, the last equality uses
\(w(z_\nu,b_\nu)=\psi_\nu w(z_\nu',b_\nu')\). The random variables used for the two
parent atoms remain distinct when the canonical children coincide, in which
case the last expression is a square. Variables associated with distinct
indices \(r\) are independent, which gives the final assertion.
\end{proof}

Fix a productive product state \(u=(q_\ell,b)\) and set \(y_\ell=b\). For
\(z=(x_{\ell+1:n},y_{\ell+1:n})\in Z_+(u)\), write its canonical run as in
Section~\ref{subsec:canonical-product-runs}, with \((q_k,y_k)\) at layer
\(k\). For \(\ell\leq k<n\) define
\ifarxiv
\begin{equation*}
A^{r,j}(z,u,k)
=A^{r,j}((x_{k+1:n},y_{k+1:n}),(q_k,y_k))\qquad
\hat A^{r,j}(z,u,k)
=\hat A^{r,j}((x_{k+1:n},y_{k+1:n}),(q_k,y_k))
\end{equation*}
\else
\begin{align*}
A^{r,j}(z,u,k)
&=A^{r,j}((x_{k+1:n},y_{k+1:n}),(q_k,y_k))\\
\hat A^{r,j}(z,u,k)
&=\hat A^{r,j}((x_{k+1:n},y_{k+1:n}),(q_k,y_k))
\end{align*}
\fi
At \(k=n\), both variables are one at \(q_F^n\).

Let \(z_1,z_2\in Z_+(u)\), and write
\(z_\nu=(x_{\ell+1:n}^{(\nu)},y_{\ell+1:n}^{(\nu)})\) for
\(\nu\in\{1,2\}\). Write \(v^{z_1,z_2}=(q_a,y_a)\) when their convergence state
lies at layer \(a\). Then
\((x_{a+1:n}^{(1)},y_{a+1:n}^{(1)})
=(x_{a+1:n}^{(2)},y_{a+1:n}^{(2)})\). For
\(\ell\leq k\leq n\), define the correction variables as follows.
\begin{equation}
\ifarxiv
\hat C^{r,j}(z_1,z_2,u,k)
=
\begin{cases}\dfrac{p^j(q_a,y_a)W(q_a,y_a)}{1-\kappa}
\hat A^{r,j}(z_1,u,k)-1 & k<a\\
\frac{W(q_a,y_a)}
{(1-\kappa)w((x_{a+1:n}^{(1)},y_{a+1:n}^{(1)}),y_a)}-1 & k\geq a
\end{cases}
\else
\begin{aligned}
&\hat C^{r,j}(z_1,z_2,u,k)\\
&=\begin{cases}
\dfrac{p^j(q_a,y_a)W(q_a,y_a)}{1-\kappa}
\hat A^{r,j}(z_1,u,k)-1 & k<a\\
\frac{W(q_a,y_a)}
{(1-\kappa)w((x_{a+1:n}^{(1)},y_{a+1:n}^{(1)}),y_a)}-1 & k\geq a
\end{cases}
\end{aligned}
\fi
\label{eq:hat-correction-variable}
\end{equation}
\begin{equation}
\ifarxiv
C^{r,j}(z_1,z_2,u,k)
=
\begin{cases}\dfrac{p^j(q_a,y_a)W(q_a,y_a)}{1-\kappa}
A^{r,j}(z_1,u,k)-1 & k<a\\
\frac{W(q_a,y_a)}
{(1-\kappa)w((x_{a+1:n}^{(1)},y_{a+1:n}^{(1)}),y_a)}-1 & k\geq a
\end{cases}
\else
\begin{aligned}
&C^{r,j}(z_1,z_2,u,k)\\
&=\begin{cases}
\dfrac{p^j(q_a,y_a)W(q_a,y_a)}{1-\kappa}
A^{r,j}(z_1,u,k)-1 & k<a\\
\frac{W(q_a,y_a)}
{(1-\kappa)w((x_{a+1:n}^{(1)},y_{a+1:n}^{(1)}),y_a)}-1 & k\geq a
\end{cases}
\end{aligned}
\fi
\label{eq:correction-variable}
\end{equation}
If \(v^{z_1,z_2}=q_F^n\), set \(a=n\) and interpret
\(p^j(q_a,y_a)=W(q_a,y_a)=1\) and assign weight one to the empty suffix. When
\(z_1=z_2\), the convergence state is \(u\), so the second branch is
\(W(u)/((1-\kappa)w(z_1,b))-1\).

\begin{lemma}
\label{lem:corrected-pair-transport}
For \(\ell\leq k<n\),
\ifarxiv
\begin{equation}
\mathbb E[\hat C^{r,j}(z_1,z_2,u,k)\hat A^{r,j}(z_2,u,k)
\mid\mathcal F_{n-k}^j]=
C^{r,j}(z_1,z_2,u,k+1)A^{r,j}(z_2,u,k+1)
\label{eq:corrected-hat-to-child}
\end{equation}
\else
\begin{align}
   & \mathbb E[\hat C^{r,j}(z_1,z_2,u,k)\hat A^{r,j}(z_2,u,k)
\mid\mathcal F_{n-k}^j]=\notag\\
&C^{r,j}(z_1,z_2,u,k+1)A^{r,j}(z_2,u,k+1)
\label{eq:corrected-hat-to-child}
\end{align}
\fi
\ifarxiv
\begin{equation}
\mathbb E[C^{r,j}(z_1,z_2,u,k)A^{r,j}(z_2,u,k)
\mid\hat{\mathcal F}_{n-k}^j]=
\hat C^{r,j}(z_1,z_2,u,k)\hat A^{r,j}(z_2,u,k)
\label{eq:corrected-final-reduce}
\end{equation}
\else
\begin{align}
&\mathbb E[C^{r,j}(z_1,z_2,u,k)A^{r,j}(z_2,u,k)
\mid\hat{\mathcal F}_{n-k}^j]=\notag\\
&\hat C^{r,j}(z_1,z_2,u,k)\hat A^{r,j}(z_2,u,k)
\label{eq:corrected-final-reduce}
\end{align}
\fi

\end{lemma}

\begin{proof}
Suppose first that \(k+1<a\). The current atom--state pairs and their
children are distinct. Moreover, \(p^j(q_a,y_a)\) is fixed by
\(\mathcal F_{n-k}^j\).
Lemmas~\ref{lem:weighted-one-atom-transfer} and
\ref{lem:weighted-two-atom-transfer} give
\ifarxiv
\begin{equation*}
\begin{aligned}
&\mathbb E[\hat C^{r,j}(z_1,z_2,u,k)\hat A^{r,j}(z_2,u,k)
\mid\mathcal F_{n-k}^j]\\
=&\frac{p^j(q_a,y_a)W(q_a,y_a)}{1-\kappa}
A^{r,j}(z_1,u,k+1)A^{r,j}(z_2,u,k+1)-A^{r,j}(z_2,u,k+1)
\end{aligned}
\end{equation*}
\else
\begin{align*}
& \mathbb E[\hat C^{r,j}(z_1,z_2,u,k)\hat A^{r,j}(z_2,u,k)
\mid\mathcal F_{n-k}^j]=
\\
&\frac{p^j(q_a,y_a)W(q_a,y_a)}{1-\kappa}
A^{r,j}(z_1,u,k+1)A^{r,j}(z_2,u,k+1)\\
&-A^{r,j}(z_2,u,k+1)
\end{align*}
\fi

This is the right-hand side of~\eqref{eq:corrected-hat-to-child}.

Suppose next that \(k+1=a\). The two children are the same atom--state pair.
Write \(A=A^{r,j}(z_1,u,a)=A^{r,j}(z_2,u,a)\).
The indicator in \(A\) is idempotent, and therefore
\[A^2=
\frac{A}{p^j(q_a,y_a)w((x_{a+1:n}^{(1)},y_{a+1:n}^{(1)}),y_a)}
\]
It follows that
\ifarxiv
\begin{equation*}
\frac{p^j(q_a,y_a)W(q_a,y_a)}{1-\kappa}A^2-A
=\left(
\frac{W(q_a,y_a)}
{(1-\kappa)w((x_{a+1:n}^{(1)},y_{a+1:n}^{(1)}),y_a)}-1
\right)A
\end{equation*}
\else
\begin{align*}
& \frac{p^j(q_a,y_a)W(q_a,y_a)}{1-\kappa}A^2-A
\\
&=\left(
\frac{W(q_a,y_a)}
{(1-\kappa)w((x_{a+1:n}^{(1)},y_{a+1:n}^{(1)}),y_a)}-1
\right)A
\end{align*}
\fi
This is again the desired right-hand side. If \(k\geq a\), the correction
factor is deterministic and the claim follows from the one-atom transition.
This proves~\eqref{eq:corrected-hat-to-child}. No convergence layer is crossed
in~\eqref{eq:corrected-final-reduce}. Before the convergence state we use the
two-atom final transition, and after it we use the one-atom final transition.
\end{proof}

We also record the transition from the samples produced by
Algorithm~\ref{alg:hmm-weighted-count-nfa} to those used by
Algorithm~\ref{alg:hmm-weighted-prefix-estimate}. Fix a reachable state set
\(R\subseteq Q^\ell\) and \(b\in H_\ell\), and write
\(R_b^+=\{q\in R:\mathcal E_\ell(q,b)\neq\emptyset\}\),
\(Z_R^+(b)=\bigcup_{q\in R_b^+}Z_+(q,b)\), and
\(W_R(b)=\sum_{z\in Z_R^+(b)}w(z,b)\).
For \(R=R(x_{1:\ell})\), this agrees with the definition in
Section~\ref{sec:canonical-runs}.
Assume \(W_R(b)>0\). For \(z=(x,y)\in Z_R^+(b)\), let
\(q_R(z)=\mathsf{first}(R,x)\). This is the unique state whose copy of \(z\)
is included by the canonical union in the latter algorithm. Since \(z\) has
positive weight,
\(q_R(z)\in R_b^+\). Put
\(\rho_R^j(b)=\min_{q\in R_b^+}p^j(q,b)\) and define
\(\tilde A_R^{r,j}(z,b)=
\frac{\mathbf{1}_{z\in\widetilde S_R^{r,j}(b)}}{\rho_R^j(b)w(z,b)}\).
Let \(\mathcal F_{\mathrm{pc}}^j\) contain all variables produced by the
completed outer repetition of Algorithm~\ref{alg:hmm-weighted-count-nfa}.
Conditioned on these variables, \(z\) appears in
\(\widetilde S_R^{r,j}(b)\) if and only
if it belongs to \(S^{r,j}(q_R(z),b)\) and the corresponding call to
\(\mathsf{reduce}\) includes the atom. Its inclusion probability is
\(\rho_R^j(b)/p^j(q_R(z),b)\). Hence
\ifarxiv
\begin{equation}
\mathbb E[\tilde A_R^{r,j}(z,b)\mid\mathcal F_{\mathrm{pc}}^j]
=\frac{\mathbf{1}_{z\in S^{r,j}(q_R(z),b)}}
{\rho_R^j(b)w(z,b)}
\frac{\rho_R^j(b)}{p^j(q_R(z),b)}
=A^{r,j}(z,(q_R(z),b))
\label{eq:online-one-atom-transfer}
\end{equation}
\else
\begin{align}
\mathbb E[\tilde A_R^{r,j}(z,b)\mid\mathcal F_{\mathrm{pc}}^j]
&=\frac{\mathbf{1}_{z\in S^{r,j}(q_R(z),b)}}
{\rho_R^j(b)w(z,b)}
\frac{\rho_R^j(b)}{p^j(q_R(z),b)}\notag\\
&=A^{r,j}(z,(q_R(z),b))
\label{eq:online-one-atom-transfer}
\end{align}
\fi
If \((r_1,z_1)\neq(r_2,z_2)\), the corresponding calls to
\(\mathsf{reduce}\) use distinct uniform random variables, even when
\(q_R(z_1)=q_R(z_2)\) or the two atoms have the same token suffix. These
variables are independent, and the same calculation gives
\ifarxiv
\begin{equation}
\mathbb E[\tilde A_R^{r_1,j}(z_1,b)\tilde A_R^{r_2,j}(z_2,b)
\mid\mathcal F_{\mathrm{pc}}^j]
=A^{r_1,j}(z_1,(q_R(z_1),b))
A^{r_2,j}(z_2,(q_R(z_2),b))
\label{eq:online-two-atom-transfer}
\end{equation}
\else
\begin{multline}
\mathbb E[\tilde A_R^{r_1,j}(z_1,b)\tilde A_R^{r_2,j}(z_2,b)
\mid\mathcal F_{\mathrm{pc}}^j]
\\
=A^{r_1,j}(z_1,(q_R(z_1),b))
A^{r_2,j}(z_2,(q_R(z_2),b))
\label{eq:online-two-atom-transfer}
\end{multline}
\fi

\begin{lemma}
\label{lem:weighted-block-intersection}
In \(\mathcal N^{*,j}\), for every \(u\in\mathcal V\) and
\(\mathcal T\subseteq[n_t]\),
\ifarxiv
\begin{equation}
\mathbb E\!\left[
\prod_{t\in\mathcal T}(M^{t,j}(u)-W(u))^2
\right]
\leq
\left(
\frac{(n+1)W(u)^2}{(1-\kappa)n_s}
\right)^{|\mathcal T|}
\label{eq:weighted-block-moment}
\end{equation}
\else
\begin{align}
 & \mathbb E\!\left[
\prod_{t\in\mathcal T}(M^{t,j}(u)-W(u))^2
\right]
&\notag\\
&\leq
\left(
\frac{(n+1)W(u)^2}{(1-\kappa)n_s}
\right)^{|\mathcal T|}
\label{eq:weighted-block-moment}
\end{align}
\fi
Fix \(\ell\in[n-1]\), \(x_{1:\ell}\in\Sigma^\ell\), and
\(b\in H_\ell\) with \(W_R(b)>0\), where \(R=R(x_{1:\ell})\).
The same bound holds with \(M^{t,j}(u)\) and \(W(u)\) replaced by
\(M_R^{t,j}(b)\) and \(W_R(b)\), respectively, where the expectation is
over \(\mathcal N^{*,j}\) and \(\mathcal G^j(x_{1:\ell})\).
\end{lemma}
Lemma~\ref{lem:weighted-block-intersection} supplies the dependent
concentration estimate for Lemma~\ref{lem:state-estimates}. Its proof is
deferred to Section~\ref{sec:proof-block-intersection}.

\begin{lemma}
\label{lem:weighted-unit-expectation}
For every productive product state \(v\), every \(z\in Z_+(v)\), and every
\(r\in[n_sn_t]\), in both \(\mathcal N^j\) and \(\mathcal N^{*,j}\),
\(\mathbb E[A^{r,j}(z,v)]=\mathbb E[\hat A^{r,j}(z,v)]=1\).
\end{lemma}

\begin{proof}
We induct on the suffix depth. At the terminal state both variables equal one
deterministically. Let \((z',v')\) be the canonical child of \((z,v)\).
Lemma~\ref{lem:weighted-one-atom-transfer} and the tower rule give
\ifarxiv
\begin{equation*}
\mathbb E[A^{r,j}(z,v)]
=\mathbb E[\hat A^{r,j}(z,v)]
=\mathbb E[A^{r,j}(z',v')]
\end{equation*}
\else
\begin{align*}
\mathbb E[A^{r,j}(z,v)]
&=\mathbb E[\hat A^{r,j}(z,v)]
=\mathbb E[A^{r,j}(z',v')]
\end{align*}
\fi
The last expectation is one by the induction hypothesis.
\end{proof}

\subsubsection{Proof of Lemma~\ref{lem:state-estimates}}

We prove Lemma~\ref{lem:state-estimates} by bounding the tentative value
\(\bar p^j(v)\) at each state in \(\mathcal V\) in the oracle-corrected algorithm
and then
substituting the bound into~\eqref{eq:first-bad-union-bound}.

\begin{proof}[Proof of Lemma~\ref{lem:state-estimates}]
Fix \(v\in\mathcal V\) and abbreviate
\(W=W(v)\). If a block mean is outside
\([W/(1+\kappa),W/(1-\kappa)]\), then
\(|M^{t,j}(v)-W|\geq\kappa W/(1+\kappa)\). With
\(\eta=\kappa/(1+\kappa)\), the parameter choice for \(n_s\) gives
\(\frac{n+1}{(1-\kappa)n_s\eta^2}
\leq\frac{(1+\kappa)^2}{16}
\leq\frac9{64}<\frac16\),
where we used \(0<\varepsilon\leq1\) and hence \(\kappa\leq1/2\).

If the median \(\hat W^j(v)\) is outside the inverse interval, at least half
of the blocks satisfy the preceding deviation event. Applying
Lemma~\ref{lem:weighted-block-intersection}, Markov's inequality, and
Fact~\ref{fact:intersection-tail} gives
\ifarxiv
\begin{equation*}
\mathbb{P}_{\mathcal N^{*,j}}\!\left[
\hat W^j(v)\notin(1\pm\kappa)^{-1}W(v)
\right]
\leq\left(\frac{4}{6}\right)^{n_t/2}\leq e^{-n_t/8}\leq\frac1{16h|Q^u|}
\end{equation*}
\else
\begin{align*}
& \mathbb{P}_{\mathcal N^{*,j}}\!\left[
\hat W^j(v)\notin(1\pm\kappa)^{-1}W(v)
\right]\\
&\leq\left(\frac{4}{6}\right)^{n_t/2}\leq e^{-n_t/8}\leq\frac1{16h|Q^u|}
\end{align*}
\fi

It remains to pass from the median to the tentative value. If
\(\hat W^j(v)\in(1\pm\kappa)^{-1}W(v)\), then
\(\hat W^j(v)^{-1}\in(1\pm\kappa)W(v)^{-1}\). We have already proved that
\(\rho^j(v)\geq(1-\kappa)/W(v)\). Since
\(\bar p^j(v)=\min\{\rho^j(v),\hat W^j(v)^{-1}\}\), it follows that
\(\bar p^j(v)\in(1\pm\kappa)W(v)^{-1}\). Therefore
\ifarxiv
\begin{equation*}
\mathbb{P}_{\mathcal N^j}[\mathcal A^j]
\leq
\sum_{v\in\mathcal V}
\mathbb{P}_{\mathcal N^{*,j}}\!\left[
\bar p^j(v)\notin(1\pm\kappa)W(v)^{-1}
\right]
\leq\frac{|\mathcal V|}{16h|Q^u|}
\leq\frac1{16}
\end{equation*}
\else
\begin{align*}
& \mathbb{P}_{\mathcal N^j}[\mathcal A^j] \\
&\leq
\sum_{v\in\mathcal V}
\mathbb{P}_{\mathcal N^{*,j}}\!\left[
\bar p^j(v)\notin(1\pm\kappa)W(v)^{-1}
\right]\\
&\leq\frac{|\mathcal V|}{16h|Q^u|}
\leq\frac1{16}
\end{align*}
\fi
Here, the first inequality is~\eqref{eq:first-bad-union-bound}, and
\(|\mathcal V|\leq h|Q^u|\). This proves
Lemma~\ref{lem:state-estimates}.
\end{proof}

\subsubsection{Proof of Lemma~\ref{lem:sample-overflow}}

We first establish the unit-expectation proposition used to bound the number
of samples.

When \(\mathcal A^j\) does not occur, for every productive \(v=(q,b)\),
\ifarxiv
\begin{equation*}
\begin{aligned}
&\mathbb E\!\left[
|S^{r,j}(v)|\mathbf{1}_{(\mathcal A^j)^c}
\right]
=\mathbb E\!\left[
p^j(v)\sum_{z\in Z_+(v)}w(z,b)A^{r,j}(z,v)
\mathbf{1}_{(\mathcal A^j)^c}
\right]\\
&\leq\frac{1+\kappa}{W(v)}
\sum_{z\in Z_+(v)}w(z,b)\mathbb E[A^{r,j}(z,v)]
=1+\kappa
\end{aligned}
\end{equation*}
\else
\begin{align*}
&\mathbb E\!\left[
|S^{r,j}(v)|\mathbf{1}_{(\mathcal A^j)^c}
\right]\\
&=\mathbb E\!\left[
p^j(v)\sum_{z\in Z_+(v)}w(z,b)A^{r,j}(z,v)
\mathbf{1}_{(\mathcal A^j)^c}
\right]\\
&\leq\frac{1+\kappa}{W(v)}
\sum_{z\in Z_+(v)}w(z,b)\mathbb E[A^{r,j}(z,v)]\\
&=1+\kappa
\end{align*}
\fi
The sample sets at zero-mass states are empty. Summing over all layers, hidden
states, and sample indices and then applying Markov's inequality gives
\ifarxiv
\begin{equation*}
\mathbb{P}_{\mathcal N^j}\!\left[
\mathcal B^j\cap(\mathcal A^j)^c
\right]
\leq
\frac{(1+\kappa)n_sn_th|Q^u|}{\theta}
\leq\frac1{16}
\end{equation*}
\else
\begin{align*}
\mathbb{P}_{\mathcal N^j}\!\left[
\mathcal B^j\cap(\mathcal A^j)^c
\right]
&\leq
\frac{(1+\kappa)n_sn_th|Q^u|}{\theta}
\leq\frac1{16}
\end{align*}
\fi
Consequently,
\ifarxiv
\begin{equation*}
\mathbb{P}_{\mathcal N^j}[(\mathcal C^j)^c]
=
\mathbb{P}_{\mathcal N^j}[\mathcal A^j\cup\mathcal B^j]
=
\mathbb{P}_{\mathcal N^j}[\mathcal A^j]
+
\mathbb{P}_{\mathcal N^j}\!\left[
\mathcal B^j\cap(\mathcal A^j)^c
\right]
\leq\frac18
\end{equation*}
\else
\begin{align*}
\mathbb{P}_{\mathcal N^j}[(\mathcal C^j)^c]
&=
\mathbb{P}_{\mathcal N^j}[\mathcal A^j\cup\mathcal B^j]\\
&=
\mathbb{P}_{\mathcal N^j}[\mathcal A^j]
+
\mathbb{P}_{\mathcal N^j}\!\left[
\mathcal B^j\cap(\mathcal A^j)^c
\right]\\
&\leq\frac18
\end{align*}
\fi
This proves Lemma~\ref{lem:sample-overflow}.

\subsubsection{Proof of Lemma~\ref{lem:integrated-one-core}}
\label{sec:proof-integrated-one-core}

Fix \(j\in[n_u]\).
If \(n=1\), then \(\Delta^j=0\), so
\(\mathcal I^j=\mathcal C^j\) and the result follows from
Lemma~\ref{lem:sample-overflow}. Assume \(n\geq2\).

For \(\ell\in[n-1]\) and every prefix satisfying
\(P(x_{1:\ell}\mid\alpha)>0\), define the oracle-corrected relative error by
\(\xi_{*}^j(x_{1:\ell})
=
\frac{
\hat P_{\text{hmm},*}^{\,j}(\alpha\mid x_{1:\ell})
}{
P_{\text{hmm}}(\alpha\mid x_{1:\ell})
}
-1\).
Define the corresponding oracle-corrected integrated squared error by
\begin{equation*}
\Delta_*^j
=
\sum_{\ell=1}^{n-1}
\sum_{\substack{x_{1:\ell}\in\Sigma^\ell\\
P(x_{1:\ell}\mid\alpha)>0}}
P(x_{1:\ell}\mid\alpha)
\left(\xi_{*}^j(x_{1:\ell})\right)^2
\end{equation*}

Fix \(\ell\in[n-1]\), \(x_{1:\ell}\in\Sigma^\ell\) satisfying
\(P(x_{1:\ell}\mid\alpha)>0\), and put
\(R=R(x_{1:\ell})\). Fix \(b\in H_\ell\) satisfying \(W_R(b)>0\).
Every \(q\in R_b^+\) is productive. In the oracle-corrected execution,
\(p^j(q,b)\in(1\pm\kappa)W(q,b)^{-1}\), and
\(W_R(b)\geq W(q,b)\). Therefore
\ifarxiv
\begin{equation}
\rho_R^j(b)
=\min_{q\in R_b^+}p^j(q,b)
\geq
\min_{q\in R_b^+}\frac{1-\kappa}{W(q,b)}
\geq\frac{1-\kappa}{W_R(b)}
\label{eq:online-rate-lower-bound}
\end{equation}
\else
\begin{align}
\rho_R^j(b)
&=\min_{q\in R_b^+}p^j(q,b)\notag\\
&\geq
\min_{q\in R_b^+}\frac{1-\kappa}{W(q,b)}\notag\\
&\geq\frac{1-\kappa}{W_R(b)}
\label{eq:online-rate-lower-bound}
\end{align}
\fi

Let \(\hat W_{R,*}^j(b)\) be the median of the block means
\(M_R^{1,j}(b),\ldots,M_R^{n_t,j}(b)\) obtained from
\(\mathcal N^{*,j}\). The singleton case of
Lemma~\ref{lem:weighted-block-intersection} and the parameter choice for
\(n_s\) give, for every \(t\in[n_t]\),
\ifarxiv
\begin{equation*}
\mathbb E\!\left[
\left(
\frac{M_R^{t,j}(b)}{W_R(b)}-1
\right)^2
\right]
\leq
\frac{n+1}{(1-\kappa)n_s}
\leq\frac{\kappa^2}{16}
\end{equation*}
\else
\begin{align*}
\mathbb E\!\left[
\left(
\frac{M_R^{t,j}(b)}{W_R(b)}-1
\right)^2
\right]
&\leq
\frac{n+1}{(1-\kappa)n_s}
\leq\frac{\kappa^2}{16}
\end{align*}
\fi
Fact~\ref{fact:median-square} gives the pointwise inequality
\(\left(
\frac{\hat W_{R,*}^j(b)}{W_R(b)}-1
\right)^2
\leq
\frac2{n_t}
\sum_{t=1}^{n_t}
\left(
\frac{M_R^{t,j}(b)}{W_R(b)}-1
\right)^2\).
Taking expectations yields
\begin{equation}
\mathbb E\!\left[
\left(
\frac{\hat W_{R,*}^j(b)}{W_R(b)}-1
\right)^2
\right]
\leq\frac{\kappa^2}{8}
\label{eq:oracle-corrected-prefix-state-mse}
\end{equation}

On the indices satisfying
\(P_{\text{hmm}}(y_\ell=b\mid x_{1:\ell})W_R(b)>0\), define
\(\beta_{\ell,b}
=\frac{P_{\text{hmm}}(y_\ell=b\mid x_{1:\ell})W_R(b)}
{P_{\text{hmm}}(\alpha\mid x_{1:\ell})}\).
The \(\beta_{\ell,b}\) are nonnegative weights that sum to one.
If \(P_{\text{hmm}}(y_\ell=b\mid x_{1:\ell})=0\), the contribution of \(b\) is zero. If
\(W_R(b)=0\), then \(R_b^+=\emptyset\), so the query returns
\(\hat W_{R,*}^j(b)=0\). Thus the terms satisfying
\(P_{\text{hmm}}(y_\ell=b\mid x_{1:\ell})W_R(b)=0\) contribute zero to both the exact and oracle-corrected
estimates, and
\ifarxiv
\begin{equation*}
\xi_{*}^j(x_{1:\ell})
=
\sum_{\substack{b\in H_\ell\\
P_{\text{hmm}}(y_\ell=b\mid x_{1:\ell})W_R(b)>0}}
\beta_{\ell,b}
\left(
\frac{\hat W_{R,*}^j(b)}{W_R(b)}-1
\right)
\end{equation*}
\else
\begin{align*}
& \xi_{*}^j(x_{1:\ell})
=
\sum_{\substack{b\in H_\ell\\
P_{\text{hmm}}(y_\ell=b\mid x_{1:\ell})W_R(b)>0}}
\beta_{\ell,b}
\left(
\frac{\hat W_{R,*}^j(b)}{W_R(b)}-1
\right)
\end{align*}
\fi
Jensen's inequality and~\eqref{eq:oracle-corrected-prefix-state-mse} give
\begin{equation}
\mathbb E\!\left[
\left(\xi_{*}^j(x_{1:\ell})\right)^2
\right]
\leq\frac{\kappa^2}{8}
\label{eq:oracle-corrected-prefix-mse}
\end{equation}

All sums are finite. Thus~\eqref{eq:oracle-corrected-prefix-mse} and linearity of
expectation give
\ifarxiv
\begin{equation}
\mathbb E[\Delta_*^j]
\leq
\sum_{\ell=1}^{n-1}
\sum_{x_{1:\ell}\in\Sigma^\ell}
P(x_{1:\ell}\mid\alpha)\frac{\kappa^2}{8}
=
\frac{(n-1)\kappa^2}{8}
\label{eq:oracle-corrected-integrated-mse}
\end{equation}
\else
\begin{align}
\mathbb E[\Delta_*^j]
&\leq
\sum_{\ell=1}^{n-1}
\sum_{x_{1:\ell}\in\Sigma^\ell}
P(x_{1:\ell}\mid\alpha)\frac{\kappa^2}{8}
=
\frac{(n-1)\kappa^2}{8}
\label{eq:oracle-corrected-integrated-mse}
\end{align}
\fi

By Lemma~\ref{lem:core-success-coupling}, the actual budgeted and oracle-corrected
estimates agree on \(\mathcal C^j\) for \(\ell\in[n-1]\). Together with the
zero terminal contribution to \(\Delta^j\), this gives
\(\mathbf{1}_{\mathcal C^j}\Delta^j
=
\mathbf{1}_{\mathcal C^j}\Delta_*^j
\leq\Delta_*^j\).
This is an indicator domination; in particular, we do not condition the
oracle-corrected execution on \(\mathcal C^j\). Taking expectations and
using~\eqref{eq:oracle-corrected-integrated-mse} gives
\(\mathbb E\!\left[
\mathbf{1}_{\mathcal C^j}\Delta^j
\right]
\leq\frac{(n-1)\kappa^2}{8}\).

Markov's inequality gives
\ifarxiv
\begin{equation*}
\mathbb{P}\!\left[
\mathcal C^j
\cap
\left\{\Delta^j>2n\kappa^2\right\}
\right]
\leq\frac{n-1}{16n}
\leq\frac1{16}
\end{equation*}
\else
\begin{align*}
\mathbb{P}\!\left[
\mathcal C^j
\cap
\left\{\Delta^j>2n\kappa^2\right\}
\right]
&\leq\frac{n-1}{16n}
\leq\frac1{16}
\end{align*}
\fi
Consequently,
\ifarxiv
\begin{equation*}
\mathbb{P}[(\mathcal I^j)^c]
\leq \mathbb{P}[(\mathcal C^j)^c]+
\mathbb{P}\!\left[\mathcal C^j
\cap
\left\{\Delta^j>2n\kappa^2\right\}
\right]
\leq\frac18+\frac1{16}=\frac3{16}
\end{equation*}
\else
\begin{align*}
&\mathbb{P}[(\mathcal I^j)^c]\\
&\leq \mathbb{P}[(\mathcal C^j)^c]+
\mathbb{P}\!\left[\mathcal C^j
\cap
\left\{\Delta^j>2n\kappa^2\right\}
\right]\\
&\leq\frac18+\frac1{16}=\frac3{16}
\end{align*}
\fi
This proves Lemma~\ref{lem:integrated-one-core}.

\subsubsection{Proof of Lemma~\ref{lem:fixed-sequence-one-core}}

\begin{proof}
For \(n=1\), the union is empty. Assume \(n\geq2\) and fix
\(j\in[n_u]\). For each \(\ell\in[n-1]\), put \(R=R(x_{1:\ell})\).
For every \(b\in H_\ell\) satisfying \(W_R(b)>0\), the oracle-corrected execution
satisfies
\ifarxiv
\begin{equation*}
\rho_R^j(b)
=\min_{q\in R_b^+}p^j(q,b)
\geq\min_{q\in R_b^+}\frac{1-\kappa}{W(q,b)}
\geq\frac{1-\kappa}{W_R(b)}
\end{equation*}
\else
\begin{align*}
\rho_R^j(b)
&=\min_{q\in R_b^+}p^j(q,b)\\
&\geq\min_{q\in R_b^+}\frac{1-\kappa}{W(q,b)}\\
&\geq\frac{1-\kappa}{W_R(b)}
\end{align*}
\fi
If \(\hat W_{R,*}^j(b)\notin(1\pm\kappa)^{-1}W_R(b)\), at least
\(n_t/2\) block means differ from \(W_R(b)\) by at least
\(\kappa W_R(b)/(1+\kappa)\). Lemma~\ref{lem:weighted-block-intersection},
Markov's inequality, Fact~\ref{fact:intersection-tail}, and the parameter
choices give
\ifarxiv
\begin{equation*}
\mathbb P\!\left[
\hat W_{R,*}^j(b)\notin(1\pm\kappa)^{-1}W_R(b)
\right]
\leq e^{-n_t/8}
\leq\frac1{16h|Q^u|}
\end{equation*}
\else
\begin{align*}
\mathbb P\!\left[
\hat W_{R,*}^j(b)\notin(1\pm\kappa)^{-1}W_R(b)
\right]
&\leq e^{-n_t/8}\\
&\leq\frac1{16h|Q^u|}
\end{align*}
\fi
If \(W_R(b)=0\), then \(R_b^+=\emptyset\) and
\(\hat W_{R,*}^j(b)=0\).
By nonnegativity and the identity
\ifarxiv
\begin{equation*}
P_{\mathrm{hmm}}(\alpha\mid x_{1:\ell})
=\sum_{b\in H_\ell}
P_{\mathrm{hmm}}(y_\ell=b\mid x_{1:\ell})W_R(b)
\end{equation*}
\else
\begin{align*}
P_{\mathrm{hmm}}(\alpha\mid x_{1:\ell})
&=\sum_{b\in H_\ell}
P_{\mathrm{hmm}}(y_\ell=b\mid x_{1:\ell})W_R(b)
\end{align*}
\fi
we obtain
\ifarxiv
\begin{equation*}
\begin{aligned}
&\mathbb P\!\left[
\bigcup_{\ell=1}^{n-1}
\left\{
\hat P_{\mathrm{hmm},*}^{\,j}(\alpha\mid x_{1:\ell})
\notin
(1\pm\kappa)^{-1}
P_{\mathrm{hmm}}(\alpha\mid x_{1:\ell})
\right\}
\right]\\
&\leq
\sum_{\ell=1}^{n-1}
\sum_{\substack{b\in H_\ell\\W_R(b)>0}}
\mathbb P\!\left[
\hat W_{R,*}^j(b)\notin(1\pm\kappa)^{-1}W_R(b)
\right]
\leq\sum_{\ell=1}^{n-1}\frac{|H_\ell|}{16h|Q^{u}|}
\leq\frac1{16}
\end{aligned}
\end{equation*}
\else
\begin{align*}
&\mathbb P\!\left[
\bigcup_{\ell=1}^{n-1}
\left\{
\hat P_{\mathrm{hmm},*}^{\,j}(\alpha\mid x_{1:\ell})
\notin
(1\pm\kappa)^{-1}
P_{\mathrm{hmm}}(\alpha\mid x_{1:\ell})
\right\}
\right]\\
&\leq
\sum_{\ell=1}^{n-1}
\sum_{\substack{b\in H_\ell\\W_R(b)>0}}
\mathbb P\!\left[
\hat W_{R,*}^j(b)\notin(1\pm\kappa)^{-1}W_R(b)
\right]
\\
&\leq\sum_{\ell=1}^{n-1}\frac{|H_\ell|}{16h|Q^{u}|}\\
&\leq\frac1{16}
\end{align*}
\fi
On \(\mathcal C^j\), Lemma~\ref{lem:core-success-coupling} identifies the
ordinary and oracle-corrected estimates simultaneously for every
\(\ell\in[n-1]\), which proves the result.
\end{proof}

\subsubsection{Proof of Lemma~\ref{lem:weighted-block-intersection}}
\label{sec:proof-block-intersection}
Fix \(u=(q_\ell,b)\), let
\(\mathcal D_t=\{n_s(t-1)+1,\ldots,n_st\}\), and put
\(B^{r,j}(z,u,k)=A^{r,j}(z,u,k)-1\) and
\(\hat B^{r,j}(z,u,k)=\hat A^{r,j}(z,u,k)-1\). The block mean at \(u\) is
\begin{equation}
M^{t,j}(u)=\frac1{n_s}\sum_{r\in\mathcal D_t}
\sum_{z\in Z_+(u)}w(z,b)\hat A^{r,j}(z,u,\ell)
\label{eq:product-state-block-mean}
\end{equation}

For \(\ell\leq k\leq n\), define
\ifarxiv
\begin{equation*}
\begin{aligned}
&\hat N_{t,k}^j
=\sum_{r\in\mathcal D_t}
\sum_{z_1,z_2\in Z_+(u)}w(z_1,b)w(z_2,b)
\cdot\hat C^{r,j}(z_1,z_2,u,k)\hat A^{r,j}(z_2,u,k)\\
&N_{t,k}^j
=\sum_{r\in\mathcal D_t}
\sum_{z_1,z_2\in Z_+(u)}w(z_1,b)w(z_2,b)
\cdot C^{r,j}(z_1,z_2,u,k)A^{r,j}(z_2,u,k)
\end{aligned}
\end{equation*}
\else
\begin{align*}
\hat N_{t,k}^j
&=\sum_{r\in\mathcal D_t}
\sum_{z_1,z_2\in Z_+(u)}w(z_1,b)w(z_2,b)\\
&\cdot\hat C^{r,j}(z_1,z_2,u,k)\hat A^{r,j}(z_2,u,k)\\
N_{t,k}^j
&=\sum_{r\in\mathcal D_t}
\sum_{z_1,z_2\in Z_+(u)}w(z_1,b)w(z_2,b)\\
&\cdot C^{r,j}(z_1,z_2,u,k)A^{r,j}(z_2,u,k)
\end{align*}
\fi
and
\ifarxiv
\begin{equation*}
\begin{aligned}
&\hat V_{t,k}^j
=\sum_{\substack{r_1,r_2\in\mathcal D_t\\r_1\neq r_2}}
\sum_{z_1,z_2\in Z_+(u)}w(z_1,b)w(z_2,b) \cdot\hat B^{r_1,j}(z_1,u,k)\hat B^{r_2,j}(z_2,u,k)\\
&-
\sum_{r\in\mathcal D_t}
\sum_{z_1,z_2\in Z_+(u)}w(z_1,b)w(z_2,b) \cdot\hat B^{r,j}(z_1,u,k)\\
&V_{t,k}^j
=\sum_{\substack{r_1,r_2\in\mathcal D_t\\r_1\neq r_2}}
\sum_{z_1,z_2\in Z_+(u)}w(z_1,b)w(z_2,b) \cdot B^{r_1,j}(z_1,u,k)B^{r_2,j}(z_2,u,k)\\
&-
\sum_{r\in\mathcal D_t}
\sum_{z_1,z_2\in Z_+(u)}w(z_1,b)w(z_2,b) \cdot B^{r,j}(z_1,u,k)
\end{aligned}
\end{equation*}
\else
\begin{align*}
\hat V_{t,k}^j
&=\sum_{\substack{r_1,r_2\in\mathcal D_t\\r_1\neq r_2}}
\sum_{z_1,z_2\in Z_+(u)}w(z_1,b)w(z_2,b)\\
&\cdot\hat B^{r_1,j}(z_1,u,k)\hat B^{r_2,j}(z_2,u,k)\\
&-
\sum_{r\in\mathcal D_t}
\sum_{z_1,z_2\in Z_+(u)}w(z_1,b)w(z_2,b)\\
&\cdot\hat B^{r,j}(z_1,u,k)\\
V_{t,k}^j
&=\sum_{\substack{r_1,r_2\in\mathcal D_t\\r_1\neq r_2}}
\sum_{z_1,z_2\in Z_+(u)}w(z_1,b)w(z_2,b)\\
&\cdot B^{r_1,j}(z_1,u,k)B^{r_2,j}(z_2,u,k)\\
&-
\sum_{r\in\mathcal D_t}
\sum_{z_1,z_2\in Z_+(u)}w(z_1,b)w(z_2,b)\\
&\cdot B^{r,j}(z_1,u,k)
\end{align*}
\fi
The following proposition gives the important recurrence used in the proof:
\begin{proposition}
\label{lem:weighted-potential-transport}
For every \(\mathcal T\subseteq[n_t]\) and \(\ell\leq k<n\),
\begin{equation}
\mathbb E\!\left[
\prod_{t\in\mathcal T}(\hat N_{t,k}^j+\hat V_{t,k}^j)
\right]
=
\mathbb E\!\left[
\prod_{t\in\mathcal T}(N_{t,k+1}^j+V_{t,k+1}^j)
\right]
\label{eq:hat-to-next-potential}
\end{equation}
and
\begin{equation}
\mathbb E\!\left[
\prod_{t\in\mathcal T}(N_{t,k}^j+V_{t,k}^j)
\right]
=
\mathbb E\!\left[
\prod_{t\in\mathcal T}(\hat N_{t,k}^j+\hat V_{t,k}^j)
\right]
\label{eq:final-to-hat-potential}
\end{equation}
\end{proposition}
\begin{proof}[Proof of Proposition~\ref{lem:weighted-potential-transport}]

We prove the two identities in turn. By the tower rule,
\ifarxiv
\begin{equation*}
\mathbb E\!\left[
\prod_{t\in\mathcal T}(\hat N_{t,k}^j+\hat V_{t,k}^j)
\right]={}
\mathbb E\!\left[
\mathbb E\!\left[
\prod_{t\in\mathcal T}(\hat N_{t,k}^j+\hat V_{t,k}^j)
\mathrel{\Big|}\mathcal F_{n-k}^j
\right]
\right]
\end{equation*}
\else
\begin{align*}
&\mathbb E\!\left[
\prod_{t\in\mathcal T}(\hat N_{t,k}^j+\hat V_{t,k}^j)
\right]={}\\
&\mathbb E\!\left[
\mathbb E\!\left[
\prod_{t\in\mathcal T}(\hat N_{t,k}^j+\hat V_{t,k}^j)
\mathrel{\Big|}\mathcal F_{n-k}^j
\right]
\right]
\end{align*}
\fi
Conditioned on \(\mathcal F_{n-k}^j\), all child sets and sampling
probabilities at layer \(k\) are fixed. Distinct blocks use disjoint sample
indices, and the current calls to \(\mathsf{reduce}\) use independent random
variables across these indices. The inner
conditional expectation therefore factors as
\ifarxiv
\begin{equation*}
\mathbb E\!\left[
\prod_{t\in\mathcal T}(\hat N_{t,k}^j+\hat V_{t,k}^j)
\mathrel{\Big|}\mathcal F_{n-k}^j
\right]={}
\prod_{t\in\mathcal T}
\left(
\mathbb E[\hat N_{t,k}^j\mid\mathcal F_{n-k}^j]
+\mathbb E[\hat V_{t,k}^j\mid\mathcal F_{n-k}^j]
\right)
\end{equation*}
\else
\begin{align*}
&\mathbb E\!\left[
\prod_{t\in\mathcal T}(\hat N_{t,k}^j+\hat V_{t,k}^j)
\mathrel{\Big|}\mathcal F_{n-k}^j
\right]={}\\
&\prod_{t\in\mathcal T}
\left(
\mathbb E[\hat N_{t,k}^j\mid\mathcal F_{n-k}^j]
+\mathbb E[\hat V_{t,k}^j\mid\mathcal F_{n-k}^j]
\right)
\end{align*}
\fi

Applying Lemma~\ref{lem:corrected-pair-transport} term by term gives
\(\mathbb E[\hat N_{t,k}^j\mid\mathcal F_{n-k}^j]=N_{t,k+1}^j\). The one-atom
transition gives
\(\mathbb E[\hat B^{r,j}(z,u,k)\mid\mathcal F_{n-k}^j]=B^{r,j}(z,u,k+1)\). If
\(r_1\neq r_2\), independence of the corresponding random variables gives
\ifarxiv
\begin{equation*}
\mathbb E[\hat B^{r_1,j}(z_1,u,k)\hat B^{r_2,j}(z_2,u,k)
\mid\mathcal F_{n-k}^j]={}
B^{r_1,j}(z_1,u,k+1)B^{r_2,j}(z_2,u,k+1)
\end{equation*}
\else
\begin{align*}
&\mathbb E[\hat B^{r_1,j}(z_1,u,k)\hat B^{r_2,j}(z_2,u,k)
\mid\mathcal F_{n-k}^j]={}\\
&B^{r_1,j}(z_1,u,k+1)B^{r_2,j}(z_2,u,k+1)
\end{align*}
\fi
Thus the two required identities are
\ifarxiv
\begin{equation*}
\mathbb E[\hat N_{t,k}^j\mid\mathcal F_{n-k}^j]
=N_{t,k+1}^j\qquad
\mathbb E[\hat V_{t,k}^j\mid\mathcal F_{n-k}^j]
=V_{t,k+1}^j
\end{equation*}
\else
\begin{align*}
\mathbb E[\hat N_{t,k}^j\mid\mathcal F_{n-k}^j]
&=N_{t,k+1}^j\\
\mathbb E[\hat V_{t,k}^j\mid\mathcal F_{n-k}^j]
&=V_{t,k+1}^j
\end{align*}
\fi
Substitution in the conditional product proves
\eqref{eq:hat-to-next-potential}.

For~\eqref{eq:final-to-hat-potential}, condition on
\(\hat{\mathcal F}_{n-k}^j\). The final reductions are independent across
blocks. Lemma~\ref{lem:corrected-pair-transport} gives
\(\mathbb E[N_{t,k}^j\mid\hat{\mathcal F}_{n-k}^j]=\hat N_{t,k}^j\), and the
one-atom transition together with independence across distinct sample
indices gives the remaining identity. In summary,
\ifarxiv
\begin{equation*}
\mathbb E[N_{t,k}^j\mid\hat{\mathcal F}_{n-k}^j]
=\hat N_{t,k}^j\qquad
\mathbb E[V_{t,k}^j\mid\hat{\mathcal F}_{n-k}^j]
=\hat V_{t,k}^j
\end{equation*}
\else
\begin{align*}
\mathbb E[N_{t,k}^j\mid\hat{\mathcal F}_{n-k}^j]
&=\hat N_{t,k}^j\\
\mathbb E[V_{t,k}^j\mid\hat{\mathcal F}_{n-k}^j]
&=\hat V_{t,k}^j
\end{align*}
\fi
\end{proof}
We now prove Lemma~\ref{lem:weighted-block-intersection}.
We now bound the joint second moment of the block errors. From
\eqref{eq:product-state-block-mean},
\ifarxiv
\begin{equation}
\begin{aligned}
&n_s^2(M^{t,j}(u)-W(u))^2 =\sum_{r\in\mathcal D_t}
\sum_{z_1,z_2\in Z_+(u)}w(z_1,b)w(z_2,b) \cdot\hat A^{r,j}(z_1,u,\ell)\hat A^{r,j}(z_2,u,\ell)\\
&-\sum_{r\in\mathcal D_t}
\sum_{z_1,z_2\in Z_+(u)}w(z_1,b)w(z_2,b) \cdot\hat A^{r,j}(z_2,u,\ell)\\
&-\sum_{r\in\mathcal D_t}
\sum_{z_1,z_2\in Z_+(u)}w(z_1,b)w(z_2,b) \cdot\hat B^{r,j}(z_1,u,\ell)\\
&+\sum_{\substack{r_1,r_2\in\mathcal D_t\\r_1\neq r_2}}
\sum_{z_1,z_2\in Z_+(u)}w(z_1,b)w(z_2,b) \cdot\hat B^{r_1,j}(z_1,u,\ell)\hat B^{r_2,j}(z_2,u,\ell)
\end{aligned}
\label{eq:weighted-second-moment-expansion}
\end{equation}
\else
\begin{align}
&n_s^2(M^{t,j}(u)-W(u))^2
\notag\\
&=\sum_{r\in\mathcal D_t}
\sum_{z_1,z_2\in Z_+(u)}w(z_1,b)w(z_2,b)
\notag\\
&\cdot\hat A^{r,j}(z_1,u,\ell)\hat A^{r,j}(z_2,u,\ell)
\notag\\
&-\sum_{r\in\mathcal D_t}
\sum_{z_1,z_2\in Z_+(u)}w(z_1,b)w(z_2,b)
\notag\\
&\cdot\hat A^{r,j}(z_2,u,\ell)
\notag\\
&-\sum_{r\in\mathcal D_t}
\sum_{z_1,z_2\in Z_+(u)}w(z_1,b)w(z_2,b)
\notag\\
&\cdot\hat B^{r,j}(z_1,u,\ell)
\notag\\
&+\sum_{\substack{r_1,r_2\in\mathcal D_t\\r_1\neq r_2}}
\sum_{z_1,z_2\in Z_+(u)}w(z_1,b)w(z_2,b)
\notag\\
&\cdot\hat B^{r_1,j}(z_1,u,\ell)\hat B^{r_2,j}(z_2,u,\ell)
\label{eq:weighted-second-moment-expansion}
\end{align}
\fi

If \(z_1\neq z_2\), their convergence state occurs strictly after \(u\).
The oracle correction guarantees \(p^j(v^{z_1,z_2})W(v^{z_1,z_2})/(1-\kappa)\geq1\).
Nonnegativity of the normalized atoms gives
\(\hat A^{r,j}(z_1,u,\ell)\hat A^{r,j}(z_2,u,\ell)
-\hat A^{r,j}(z_2,u,\ell)
\leq\hat C^{r,j}(z_1,z_2,u,\ell)\hat A^{r,j}(z_2,u,\ell)\).
For \(z_1=z_2=z\), idempotence of the indicator and
\(\rho^j(u)\geq(1-\kappa)/W(u)\) give
\ifarxiv
\begin{equation*}
\begin{aligned}
&(\hat A^{r,j}(z,u,\ell))^2-\hat A^{r,j}(z,u,\ell)\\
=&\left(\frac1{\rho^j(u)w(z,b)}-1\right)\hat A^{r,j}(z,u,\ell)\\
\leq&\left(\frac{W(u)}{(1-\kappa)w(z,b)}-1\right)\hat A^{r,j}(z,u,\ell)\\
=& \hat C^{r,j}(z,z,u,\ell)\hat A^{r,j}(z,u,\ell)
\end{aligned}
\end{equation*}
\else
\begin{align*}
&(\hat A^{r,j}(z,u,\ell))^2-\hat A^{r,j}(z,u,\ell)\\
&=\left(\frac1{\rho^j(u)w(z,b)}-1\right)\hat A^{r,j}(z,u,\ell)\\
&\leq\left(\frac{W(u)}{(1-\kappa)w(z,b)}-1\right)\hat A^{r,j}(z,u,\ell)\\
&=\hat C^{r,j}(z,z,u,\ell)\hat A^{r,j}(z,u,\ell)
\end{align*}
\fi
Substituting these bounds into
\eqref{eq:weighted-second-moment-expansion} gives
\(0\leq n_s^2(M^{t,j}(u)-W(u))^2
\leq\hat N_{t,\ell}^j+\hat V_{t,\ell}^j\).
Multiplying over \(t\in\mathcal T\), taking expectations, and successively
using Proposition~\ref{lem:weighted-potential-transport} gives
\ifarxiv
\begin{equation}
n_s^{2|\mathcal T|}
\mathbb E\!\left[
\prod_{t\in\mathcal T}(M^{t,j}(u)-W(u))^2
\right]
\leq\mathbb E\!\left[
\prod_{t\in\mathcal T}(N_{t,n}^j+V_{t,n}^j)
\right]
\label{eq:transport-to-terminal}
\end{equation}
\else
\begin{align}
&n_s^{2|\mathcal T|}
\mathbb E\!\left[
\prod_{t\in\mathcal T}(M^{t,j}(u)-W(u))^2
\right]
\notag\\
&
\leq\mathbb E\!\left[
\prod_{t\in\mathcal T}(N_{t,n}^j+V_{t,n}^j)
\right]
\label{eq:transport-to-terminal}
\end{align}
\fi

At \(q_F^n\), all normalized atoms equal one. Hence \(V_{t,n}^j=0\) and
\ifarxiv
\begin{equation*}
N_{t,n}^j
=n_s\left(\frac{W(u)^2}{1-\kappa}-W(u)^2\right)
+\frac{n_s}{1-\kappa}
\sum_{\substack{z_1,z_2\in Z_+(u)\\z_1\neq z_2}}\frac{w(z_1,b)w(z_2,b)W(q_a,y_a)}
{w((x_{a+1:n}^{(1)},y_{a+1:n}^{(1)}),y_a)}
\end{equation*}
\else
\begin{align*}
&N_{t,n}^j
=n_s\left(\frac{W(u)^2}{1-\kappa}-W(u)^2\right)\\
&+\frac{n_s}{1-\kappa}
\sum_{\substack{z_1,z_2\in Z_+(u)\\z_1\neq z_2}}\frac{w(z_1,b)w(z_2,b)W(q_a,y_a)}
{w((x_{a+1:n}^{(1)},y_{a+1:n}^{(1)}),y_a)}
\end{align*}
\fi
where \(v^{z_1,z_2}=(q_a,y_a)\), using the terminal convention stated after
\eqref{eq:correction-variable}.

For fixed \(z_1\), partition \(z_2\) according to its convergence state with
\(z_1\). Proposition~\ref{lem:weighted-splicing} bounds the inner sum by \(W(u)\)
whenever the convergence state is not \(q_F^n\). When the convergence state is
\(q_F^n\), the weight factor is one and the corresponding atoms form a subset
of \(Z_+(u)\), so their total weight is at most \(W(u)\). There are at most
\(n-\ell+1\) convergence states. The pairs satisfying \(z_1=z_2\) contribute
exactly \(W(u)^2\). It follows that
\(0\leq N_{t,n}^j
\leq\frac{n_s(n+1)W(u)^2}{1-\kappa}\).
Combining this with~\eqref{eq:transport-to-terminal},
\ifarxiv
\begin{equation*}
\mathbb E\!\left[
\prod_{t\in\mathcal T}(M^{t,j}(u)-W(u))^2
\right]
\leq
\left(\frac{(n+1)W(u)^2}{(1-\kappa)n_s}\right)^{|\mathcal T|}
\end{equation*}
\else
\begin{align*}
&\mathbb E\!\left[
\prod_{t\in\mathcal T}(M^{t,j}(u)-W(u))^2
\right]
\leq
\left(\frac{(n+1)W(u)^2}{(1-\kappa)n_s}\right)^{|\mathcal T|}
\end{align*}
\fi
This proves~\eqref{eq:weighted-block-moment}.

We next prove the remaining claim for \(\mathsf{compute}\). Fix
\(R\) and \(b\in H_\ell\) with \(W_R(b)>0\), and set \(y_\ell=b\). For
\(z=(x_{\ell+1:n},y_{\ell+1:n})\in Z_R^+(b)\), follow the canonical run
beginning at \((q_R(z),b)\). At layer \(k\), denote its NFA state by
\(q_k(z)\), and define
\begin{align*}
A_R^{r,j}(z,b,k)
&=A^{r,j}((x_{k+1:n},y_{k+1:n}),(q_k(z),y_k))\\
\hat A_R^{r,j}(z,b,k)
&=\hat A^{r,j}((x_{k+1:n},y_{k+1:n}),(q_k(z),y_k))
\end{align*}
for \(\ell\leq k<n\), with both values one at \(k=n\). For two distinct
atoms, write \(z_\nu=(x_{\ell+1:n}^{(\nu)},y_{\ell+1:n}^{(\nu)})\) for
\(\nu\in\{1,2\}\), and let \(v^{z_1,z_2}=(q_a,y_a)\) be the convergence state of
the canonical runs beginning at \((q_R(z_1),b)\) and \((q_R(z_2),b)\).
Define \(C_R^{r,j}\) and \(\hat C_R^{r,j}\) for such a pair by the two branches
in~\eqref{eq:correction-variable} and
\eqref{eq:hat-correction-variable}, replacing \(A^{r,j},\hat A^{r,j}\) by
\(A_R^{r,j},\hat A_R^{r,j}\), respectively. For \(z_1=z_2=z\), define at every layer
\(C_R^{r,j}(z,z,b,k)=\hat C_R^{r,j}(z,z,b,k)
=\frac{W_R(b)}{(1-\kappa)w(z,b)}-1\).
Thus the correction for \(z_1=z_2\) preserves the mass \(W_R(b)\) throughout
the run.

At the reduction stage of \(\mathsf{compute}\), define
\begin{equation*}
\ifarxiv
\tilde C_R^{r,j}(z_1,z_2,b){}
=
\begin{cases}\dfrac{p^j(q_a,y_a)W(q_a,y_a)}{1-\kappa}
\tilde A_R^{r,j}(z_1,b)-1 & z_1\neq z_2\\
\dfrac{W_R(b)}{(1-\kappa)w(z_1,b)}-1 & z_1=z_2
\end{cases}
\else
\begin{aligned}
&\tilde C_R^{r,j}(z_1,z_2,b){}\\
&=\begin{cases}
\dfrac{p^j(q_a,y_a)W(q_a,y_a)}{1-\kappa}
\tilde A_R^{r,j}(z_1,b)-1 & z_1\neq z_2\\
\dfrac{W_R(b)}{(1-\kappa)w(z_1,b)}-1 & z_1=z_2
\end{cases}
\end{aligned}
\fi
\end{equation*}
The terminal convention is used when the two canonical runs agree only at
\(q_F^n\). Equations~\eqref{eq:online-one-atom-transfer} and
\eqref{eq:online-two-atom-transfer} give
\ifarxiv
\begin{equation}
\mathbb E[\tilde C_R^{r,j}(z_1,z_2,b)\tilde A_R^{r,j}(z_2,b)
\mid\mathcal F_{\mathrm{pc}}^j]={}
C_R^{r,j}(z_1,z_2,b,\ell)A_R^{r,j}(z_2,b,\ell)
\label{eq:online-corrected-transfer}
\end{equation}
\else
\begin{align}
&\mathbb E[\tilde C_R^{r,j}(z_1,z_2,b)\tilde A_R^{r,j}(z_2,b)
\mid\mathcal F_{\mathrm{pc}}^j]={}\notag\\
&C_R^{r,j}(z_1,z_2,b,\ell)A_R^{r,j}(z_2,b,\ell)
\label{eq:online-corrected-transfer}
\end{align}
\fi
For \(\ell\leq k<n\), the two corrected-pair transitions also hold for
\(C_R^{r,j},A_R^{r,j}\): the proof is identical to that of
Lemma~\ref{lem:corrected-pair-transport}, and
Lemma~\ref{lem:weighted-two-atom-transfer} already permits the two atoms to
belong to different product states at the same layer.

Put \(\tilde B_R^{r,j}(z,b)=\tilde A_R^{r,j}(z,b)-1\). For every block,
define
\ifarxiv
\begin{equation*}
\begin{aligned}
&\tilde N_{R,t}^j
=\sum_{r\in\mathcal D_t}
\sum_{z_1,z_2\in Z_R^+(b)}w(z_1,b)w(z_2,b)
\cdot\tilde C_R^{r,j}(z_1,z_2,b)\tilde A_R^{r,j}(z_2,b)\\
&\tilde V_{R,t}^j
=\sum_{\substack{r_1,r_2\in\mathcal D_t\\r_1\neq r_2}}
\sum_{z_1,z_2\in Z_R^+(b)}w(z_1,b)w(z_2,b)
\cdot\tilde B_R^{r_1,j}(z_1,b)\tilde B_R^{r_2,j}(z_2,b)\\
&-
\sum_{r\in\mathcal D_t}
\sum_{z_1,z_2\in Z_R^+(b)}w(z_1,b)w(z_2,b)
\cdot\tilde B_R^{r,j}(z_1,b)
\end{aligned}
\end{equation*}
\else
\begin{align*}
\tilde N_{R,t}^j
&=\sum_{r\in\mathcal D_t}
\sum_{z_1,z_2\in Z_R^+(b)}w(z_1,b)w(z_2,b)\\
&\cdot\tilde C_R^{r,j}(z_1,z_2,b)\tilde A_R^{r,j}(z_2,b)\\
\tilde V_{R,t}^j
&=\sum_{\substack{r_1,r_2\in\mathcal D_t\\r_1\neq r_2}}
\sum_{z_1,z_2\in Z_R^+(b)}w(z_1,b)w(z_2,b)\\
&\cdot\tilde B_R^{r_1,j}(z_1,b)\tilde B_R^{r_2,j}(z_2,b)\\
&-
\sum_{r\in\mathcal D_t}
\sum_{z_1,z_2\in Z_R^+(b)}w(z_1,b)w(z_2,b)\\
&\cdot\tilde B_R^{r,j}(z_1,b)
\end{align*}
\fi
For \(\ell\leq k\leq n\), define
\(N_{R,t,k}^j,\hat N_{R,t,k}^j,V_{R,t,k}^j,\hat V_{R,t,k}^j\) by the four
product-state potential formulas above, replacing \(Z_+(u)\), \(A^{r,j}\),
\(\hat A^{r,j}\), \(C^{r,j}\), and \(\hat C^{r,j}\) by their \(R\)-subscripted
counterparts.

The block mean produced by \(\mathsf{compute}\) satisfies
\(M_R^{t,j}(b)=\frac1{n_s}\sum_{r\in\mathcal D_t}
\sum_{z\in Z_R^+(b)}w(z,b)\tilde A_R^{r,j}(z,b)\).
This is precisely the block mean computed by
Algorithm~\ref{alg:hmm-weighted-prefix-estimate}. The same expansion used
above applies. For \(z_1\neq z_2\), the oracle correction again gives
\(p^j(v^{z_1,z_2})W(v^{z_1,z_2})/(1-\kappa)\geq1\). For \(z_1=z_2=z\),
idempotence and \(\rho_R^j(b)\geq(1-\kappa)/W_R(b)\) give
\ifarxiv
\begin{equation*}
(\tilde A_R^{r,j}(z,b))^2-\tilde A_R^{r,j}(z,b)
=\left(\frac1{\rho_R^j(b)w(z,b)}-1\right)\tilde A_R^{r,j}(z,b)
\leq\left(\frac{W_R(b)}{(1-\kappa)w(z,b)}-1\right)\tilde A_R^{r,j}(z,b)
\end{equation*}
\else
\begin{align*}
& (\tilde A_R^{r,j}(z,b))^2-\tilde A_R^{r,j}(z,b)\\
&=\left(\frac1{\rho_R^j(b)w(z,b)}-1\right)\tilde A_R^{r,j}(z,b)\\
&\leq\left(\frac{W_R(b)}{(1-\kappa)w(z,b)}-1\right)\tilde A_R^{r,j}(z,b)
\end{align*}
\fi
Hence
\begin{equation}
0\leq n_s^2(M_R^{t,j}(b)-W_R(b))^2
\leq\tilde N_{R,t}^j+\tilde V_{R,t}^j
\label{eq:online-moment-domination}
\end{equation}
Conditional independence of the variables in \(\mathcal G^j\) across blocks,
\eqref{eq:online-one-atom-transfer}, \eqref{eq:online-two-atom-transfer},
and \eqref{eq:online-corrected-transfer} transport both the corrected and
centered terms and give
\ifarxiv
\begin{equation}
\mathbb E\!
\prod_{t\in\mathcal T}(\tilde N_{R,t}^j+\tilde V_{R,t}^j)
={}\mathbb E\!
\prod_{t\in\mathcal T}(N_{R,t,\ell}^j+V_{R,t,\ell}^j)
\label{eq:online-initial-potential-transfer}
\end{equation}
\else
\begin{align}
\mathbb E\!
\prod_{t\in\mathcal T}(\tilde N_{R,t}^j+\tilde V_{R,t}^j)
={}\mathbb E\!
\prod_{t\in\mathcal T}(N_{R,t,\ell}^j+V_{R,t,\ell}^j)
\label{eq:online-initial-potential-transfer}
\end{align}
\fi
For the \(R\)-subscripted potentials, the same two transitions give
for \(\ell\leq k<n\),
\begin{align*}
&\mathbb E\prod_{t\in\mathcal T}
(\hat N_{R,t,k}^j+\hat V_{R,t,k}^j)={}\mathbb E\prod_{t\in\mathcal T}
(N_{R,t,k+1}^j+V_{R,t,k+1}^j)\\
&\mathbb E\prod_{t\in\mathcal T}
(N_{R,t,k}^j+V_{R,t,k}^j)={}\mathbb E\prod_{t\in\mathcal T}
(\hat N_{R,t,k}^j+\hat V_{R,t,k}^j)
\end{align*}
Their proof is exactly the proof of
Lemma~\ref{lem:weighted-potential-transport}, applied to the canonical runs
beginning at \((q_R(z),b)\).
Thus the last expectation in~\eqref{eq:online-initial-potential-transfer}
is transported to the terminal layer.

It remains to bound the terminal expression. For fixed \(z\in Z_R^+(b)\)
and \(\ell<k<n\), let \(D_R(z,k)\) contain the atoms \(z'\in Z_R^+(b)\)
whose canonical run beginning at \((q_R(z'),b)\) first agrees with that of
\(z\) at \((q_k(z),y_k)\). The prefix--suffix injection from
Proposition~\ref{lem:weighted-splicing} now has image in
\(Z_R^+(b)\), and HMM factorization gives
\begin{equation}
\frac{W(q_k(z),y_k)}
{w((x_{k+1:n},y_{k+1:n}),y_k)}
\sum_{z'\in D_R(z,k)}w(z',b)
\leq W_R(b)
\label{eq:online-splicing}
\end{equation}
Replacing the common tail may change \(\mathsf{first}(R,x)\), but this is
irrelevant: the new atom still has an accepting run from some state in
\(R\), so it remains in the set union \(Z_R^+(b)\), and HMM factorization
preserves positive weight. Within one convergence class the atoms share the
same tail and distinct atoms have distinct prefixes. Concatenation at the
fixed cut is therefore injective.
At the terminal layer, the definition of the correction for \(z_1=z_2\)
gives the exact identity
\ifarxiv
\begin{equation*}
N_{R,t,n}^j
=n_s\left(\frac{W_R(b)^2}{1-\kappa}-W_R(b)^2\right)
+\frac{n_s}{1-\kappa}
\sum_{\substack{z_1,z_2\in Z_R^+(b)\\z_1\neq z_2}}\frac{w(z_1,b)w(z_2,b)W(q_a,y_a)}
{w((x_{a+1:n}^{(1)},y_{a+1:n}^{(1)}),y_a)}
\end{equation*}
\else
\begin{align*}
&N_{R,t,n}^j
=n_s\left(\frac{W_R(b)^2}{1-\kappa}-W_R(b)^2\right)\\
&+\frac{n_s}{1-\kappa}
\sum_{\substack{z_1,z_2\in Z_R^+(b)\\z_1\neq z_2}}\frac{w(z_1,b)w(z_2,b)W(q_a,y_a)}
{w((x_{a+1:n}^{(1)},y_{a+1:n}^{(1)}),y_a)}
\end{align*}
\fi
where \(v^{z_1,z_2}=(q_a,y_a)\), with the terminal convention used when
the two runs agree only at \(q_F^n\). Before multiplication by
\((1-\kappa)^{-1}\), the positive contribution from pairs satisfying
\(z_1=z_2\) is \(W_R(b)^2\). Grouping every pair satisfying
\(z_1\neq z_2\) by its convergence state and
applying~\eqref{eq:online-splicing} shows
\begin{equation}
0\leq N_{R,t,n}^j
\leq\frac{n_s(n+1)W_R(b)^2}{1-\kappa}
\label{eq:online-terminal-potential}
\end{equation}
Indeed, for fixed \(z_1\), the equality \(z_2=z_1\) supplies one class, and
unequal atoms can agree only at layers \(\ell+1,\ldots,n\). Thus the number
of classes is at most
\(n-\ell+1\leq n+1\). Combining
\eqref{eq:online-moment-domination}--
\eqref{eq:online-terminal-potential} gives
\ifarxiv
\begin{equation*}
\mathbb E\!\left[
\prod_{t\in\mathcal T}(M_R^{t,j}(b)-W_R(b))^2
\right]
\leq
\left(\frac{(n+1)W_R(b)^2}{(1-\kappa)n_s}\right)^{|\mathcal T|}
\end{equation*}
\else
\begin{align*}
&\mathbb E\!\left[
\prod_{t\in\mathcal T}(M_R^{t,j}(b)-W_R(b))^2
\right]
\leq
\left(\frac{(n+1)W_R(b)^2}{(1-\kappa)n_s}\right)^{|\mathcal T|}
\end{align*}
\fi
This proves the remaining claim for \(\mathsf{compute}\) in
Lemma~\ref{lem:weighted-block-intersection}.

\label{sec:proof-potential-transport}

\section{Time Complexity}
\label{sec:time_complexity}
\timecomplexity*
Let \(m=|Q|\), and retain \(|\Sigma|\) explicitly as the alphabet size. The
number of states in the unrolled NFA, \(|Q^u|\), is at most \(mn\). Recall
that Algorithm~\ref{alg:hmm-weighted-count-nfa} sets the following parameters:
\(n_s=O(n\varepsilon^{-2}),
n_t=O(\log(nmh)),
n_u=O(\log(\delta^{-1})),
\theta
=O\!\left(n^2mh\varepsilon^{-2}\log(nmh)\right)
\).
Theorem~\ref{thm:distribution-preservation} uses these parameters without
modification; in particular,
\(n_u=\lceil8\log(\delta^{-1})\rceil\).

We first analyze one of the \(n_u\) repetitions of
\(\mathsf{precompute}\). The sample counter is maintained incrementally,
so checking whether it reaches \(\theta\) takes constant time. Before the
repetition sets \(\mathsf{fail}^j=1\) and terminates, the number of stored
samples, and hence the total number of cache rows, is smaller than \(\theta\).

\begin{itemize}
\item For \(\mathsf{computeCache}\), at every layer \(\ell\) and for every
token in \(\Sigma\), we multiply an
\(|\mathcal S_{\ell+1}|\times|Q^{\ell+1}|\) matrix by an
\(|Q^{\ell+1}|\times|Q^\ell|\) matrix. If \(\omega\leq3\) is the
matrix-multiplication exponent, the total cost of all cache computations and
updates is
\(O\!\left(
|\Sigma|\left(nm^\omega+\theta m^{\omega-1}\right)
\right)
=O\!\left(|\Sigma|\theta m^2\right)\).

\item For \(\mathsf{reduce}\), each sampled atom is examined only for
productive labeled incoming edges from the preceding layer. It is therefore
examined at most \(|\Sigma|hm\) times. The cumulative cost of all reductions is
\(O(|\Sigma|\theta hm)
=O\!\left(
|\Sigma|n^2m^2h^2\varepsilon^{-2}\log(nmh)
\right)\).

\item For the median-of-means computations, there are at most
\(h|Q^{u}|\) productive product states. Computing their block means
and medians takes
\(O(n_sn_th|Q^{u}|)
=O\!\left(n^2mh\varepsilon^{-2}\log(nmh)\right)\).

\item For the canonical union, deciding whether an atom belongs to an
earlier NFA successor takes at most \(m\) cache lookups. The reductions
examine \(O(|\Sigma|\theta hm)\) candidate atoms, so all canonical unions take
\(O(|\Sigma|\theta hm^2)
=O\!\left(
|\Sigma|n^2m^3h^2\varepsilon^{-2}\log(nmh)
\right)\).
\end{itemize}

The canonical unions dominate the other operations. Thus one outer
repetition of \(\mathsf{precompute}\) takes
\(O\!\left(
|\Sigma|n^2m^3h^2\varepsilon^{-2}\log(nmh)
\right)\) time. Multiplying by \(n_u\), the complete weighted-\#NFA
precomputation takes
\[
O\!\left(
|\Sigma|n^2m^3h^2\varepsilon^{-2}
\log(nmh)\log(\delta^{-1})
\right)
\]
time. Constructing the unrolled NFA and its productive product states has
cost \(O(|\Sigma|nm^2h^2)\) and is absorbed in this bound.

We next analyze constrained generation. At each generation step \(\ell\), the
decoder evaluates
\(\hat P_{\text{hmm}}(\alpha\mid x_{1:\ell})\) for every \(x_\ell\in\Sigma\).
This requires \(|\Sigma|\) calls to \(\mathsf{compute}\) per step and
\(n|\Sigma|\) calls across all \(n\) steps.
For a fixed index \(j\), each relevant atom is tested once by
\(\mathsf{reduce}\) for each candidate prefix in which it appears, and the
canonical union performs at most \(m\) cache lookups for that atom. The block
means and medians take \(O(n_sn_th)\) additional time per generated position
along one candidate path. Thus, all completion-probability computations for a
fixed index \(j\) take
\(O\!\left(
|\Sigma|\left(m\theta+nn_sn_th\right)
\right)\)
time. Summing over all \(n_u\) repetitions and substituting the parameter
values, the completion-probability computations take
\ifarxiv
\begin{equation*}
O\!\left(
|\Sigma|n_u\left(m\theta+nn_sn_th\right)
\right)
=O\!\left(
|\Sigma|n^2m^2h\varepsilon^{-2}
\log(nmh)\log(\delta^{-1})
\right)
\end{equation*}
\else
\begin{align*}
&O\!\left(
|\Sigma|n_u\left(m\theta+nn_sn_th\right)
\right)\\
&=O\!\left(
|\Sigma|n^2m^2h\varepsilon^{-2}
\log(nmh)\log(\delta^{-1})
\right)
\end{align*}
\fi
time. At each generated position, computing
\(P_{\text{hmm}}(y_\ell=b\mid x_{1:\ell})\) for every \(b\in H_\ell\),
together with \(P_{\text{hmm}}(x_{1:\ell})\), for every candidate
\(x_\ell\in\Sigma\) takes
\(O(|\Sigma|h^2)\) time. Therefore, 
constrained generation of a length-\(n\) sequence takes
\[
O\!\left(
|\Sigma|n^2m^2h\varepsilon^{-2}
\log(nmh)\log(\delta^{-1})
+
|\Sigma|nh^2
\right)
\]
time.

Combining precomputation and constrained generation, the total running time is
\[
O\!\left(
|\Sigma|n^2m^3h^2\varepsilon^{-2}
\log(nmh)\log(\delta^{-1})
\right)
\]

\section{Constraint Satisfaction}
\label{sec:constraint-satisfaction}

We conclude by showing that any sequence returned by constrained generation
satisfies the NFA constraint. This guarantee does not depend on the accuracy
of the completion-probability estimates.

\begin{lemma}
\label{lem:constraint-satisfaction}
Suppose constrained generation uses the value
\(\hat{P}_{\text{hmm}}(\alpha\mid x_{1:\ell})\) returned by
Algorithm~\ref{alg:hmm-weighted-prefix-estimate} in its next-token
distribution and that the distribution's normalizing constant is positive at
every step. Every length-\(n\) sequence \(x_{1:n}\) returned by constrained
generation belongs to \(L_n(A)\).
\end{lemma}

\begin{proof}
Consider the final generation step. For every \(a\in\Sigma\), the
unnormalized probability assigned to \(a\) is
\(P_{\text{lm}}(a\mid x_{1:n-1})
\hat{P}_{\text{hmm}}(\alpha\mid x_{1:n-1}\cdot a)\).
By the terminal clause of
Algorithm~\ref{alg:hmm-weighted-prefix-estimate}, the definition of \(R\), and
\(L(A^u)=L_n(A)\),
\(\hat{P}_{\text{hmm}}(\alpha\mid x_{1:n-1}\cdot a)
=
\mathbf{1}_{q_F^n\in R(x_{1:n-1}\cdot a)}
=
\mathbf{1}_{x_{1:n-1}\cdot a\in L_n(A)}\).
Thus every \(a\) for which \(x_{1:n-1}\cdot a\notin L_n(A)\) has probability
zero and cannot be returned. It follows that the returned token \(x_n\)
satisfies \(q_F^n\in R(x_{1:n})\), and hence
\(x_{1:n}\in L(A^u)=L_n(A)\).
\end{proof}

\clearpage
\section{Prompt Templates}
\label{sec:prompt-templates}

The prompt for evaluating the LM baseline is shown below. The constraint is
included in the instruction prompt.

\begin{tcolorbox}[
  colframe=black!40,
  colback=white,
  coltitle=black,
  fonttitle=\bfseries,
  title=Base LM,
  boxrule=0.5pt
]
\begin{lstlisting}[style=prompt]
Generate a natural and coherent output that uses the following concepts: {concept_list}. The output must follow the following Regex constraint: {nfa_constraint}. The length of the output must be within {minimum_content_tokens} and {maximum_content_tokens}. Return only the output without explanation:

\end{lstlisting}
\end{tcolorbox}

The prompt for evaluating constrained generation methods
(XGrammar~\citep{xgrammar}, Ctrl-G~\citep{ctrlg},
and \framework{}) is shown below. The constraint is passed to the constrained
generation engine through the interface and is not shown in the prompt.

\begin{tcolorbox}[
  colframe=black!40,
  colback=white,
  coltitle=black,
  fonttitle=\bfseries,
  title=Constrained Generation Engine,
  boxrule=0.5pt
]
\begin{lstlisting}[style=prompt]
Generate a natural and coherent output that uses the following concepts: {concept_list}. Return only the output without explanation:
\end{lstlisting}
\end{tcolorbox}

The prompt for the LM judge is shown below.
\ifarxiv\else\newpage\fi
\begin{tcolorbox}[
  colframe=black!40,
  colback=white,
  coltitle=black,
  fonttitle=\bfseries,
  title=LM quality-judge prompt,
  boxrule=0.5pt
]
\begin{lstlisting}[style=prompt]
You are a strict quality evaluator of English paragraphs. Evaluate the quality of the following paragraph according to the following grading scheme. For each criterion, assign a grade of 1, 2, 3, or 4. Return the average grade.

1. Fluency: grammar, readability, and natural wording. A grade of 4 is well-formed and natural; a 3 has a tiny error; a 2 is understandable but has noticeable awkwardness; a 1 is unreadable.

2. Coherence: semantic reasonableness and internal consistency. A grade of 4 means logical and well connected; a 3 means broadly consistent but has unclear transitions; a 2 has a serious consistency issue but is readable; a 1 makes no semantic sense or is self-contradictory.

Put your final score within the delimiters <<< >>> for extraction.

Your Task: {output_to_be_graded}
Your Score:
\end{lstlisting}
\end{tcolorbox}


\end{document}